\documentclass[11pt,twocolumn]{article}

\usepackage[utf8]{inputenc}
\usepackage[T1]{fontenc}
\usepackage{amsmath,amssymb,amsthm}
\usepackage{mathtools}
\usepackage{geometry}
\usepackage{graphicx}
\usepackage{hyperref}
\usepackage{natbib}
\usepackage{booktabs}
\usepackage{enumitem}
\usepackage{xcolor}
\usepackage{float}
\usepackage{placeins}

\usepackage{caption}
\usepackage{subcaption}
\usepackage{bm}
\usepackage{authblk}
\usepackage{multirow}

\theoremstyle{plain}
\newtheorem{theorem}{Theorem}[section]

\theoremstyle{definition}
\newtheorem{proposition}[theorem]{Proposition}
\newtheorem{definition}[theorem]{Definition}
\newtheorem{remark}[theorem]{Remark}

\newcommand{\E}{\mathbb{E}}
\newcommand{\Sphere}{\mathbb{S}^2}
\newcommand{\Hs}{H^s(\Sphere)}

\newcommand{\norm}[1]{\left\|#1\right\|}
\newcommand{\inner}[2]{\left\langle #1,#2\right\rangle}
\newcommand{\Ylm}{Y_\ell^m}
\newcommand{\ahat}{\hat{a}}

\title{\textbf{The Recipe Matters More Than the Kitchen:\\ Mathematical Foundations of the AI Weather Prediction Pipeline}}

\author[1,*]{Piyush Garg}
\author[1]{Diana R. Gergel}
\author[2]{Andrew E. Shao}
\author[1]{Galen J. Yacalis}
\affil[1]{Commercial AI Lab, RWE Trading Americas Inc., Bellevue, WA, USA}
\affil[2]{AI Research Lab, Hewlett Packard Enterprise, Victoria, BC, Canada}
\affil[*]{Corresponding author: piyush.garg@rwe.com}
\date{March 2026}

\begin{document}
\maketitle

\begin{abstract}
The rapid proliferation of artificial intelligence methods in weather and climate prediction has exposed a critical gap: despite remarkable empirical progress, there is no unified mathematical framework for understanding what determines AI forecast skill, nor a comprehensive empirical evaluation grounding theoretical predictions in reproducible multi-model inference.
The limited theoretical work that exists is largely embedded within individual model papers, motivated by specific architectural choices such as equivariance \citep{bonev2023sfno} or spectral representations \citep{li2020fno}, rather than addressing the learning pipeline as a whole.
Meanwhile, operational evidence from 2023--2026 increasingly demonstrates that \emph{training methodology, loss function design, and data diversity} are at least as important as architecture selection \citep{subich2025msh,daub2025fastnet,bodnar2025aurora}.

This paper makes two interleaved contributions.
\textbf{Theoretically}, we construct a holistic framework, rooted in approximation theory on the sphere \citep{freeden1998constructive,driscoll1994sphere}, dynamical systems theory \citep{lorenz1969predictability,katok1995modern}, information theory \citep{cover2006elements}, and statistical learning theory \citep{anthony1999neural,bartlett2002rademacher}, that treats the \emph{complete learning pipeline} (architecture, loss function, training strategy, and data distribution; hereafter simply the \emph{pipeline}) rather than architecture alone.
We establish a \emph{Learning Pipeline Error Decomposition} (Proposition~\ref{prop:dominance}) showing that at current scales, estimation error (loss- and data-dependent) dominates approximation error (architecture-dependent).
A central finding is that all major architectures achieve approximation error well below estimation error at current operational resolutions, implying that pipeline choices (loss function, training strategy, data distribution) matter more than architecture selection.
We develop a \emph{Loss Function Spectral Theory} (Theorem~\ref{thm:mse_bias}) formalizing MSE-induced spectral blurring in spherical harmonic coordinates, validated empirically using latitude-weighted Fast Fourier Transform (FFT) spectra (\S\ref{subsec:sh_fft}).
We derive \emph{Out-of-Distribution Extrapolation Bounds} (Proposition~\ref{thm:ood_bias}) proving that all data-driven models systematically underestimate record-breaking extremes with bias growing linearly in record exceedance.

\textbf{Empirically}, we validate these predictions through systematic inference across ten architecturally diverse AI weather models (AIFS, AIFS-ENS, Atlas, Aurora, FourCastNet~3, FengWu, FuXi, GraphCast, Pangu-Weather, and SFNO; see Table~\ref{tab:convergence} for architecture details) using NVIDIA Earth2Studio with ERA5 initial conditions, evaluating six mathematically grounded metrics with inter-initialization-date confidence intervals across 30 dates spanning all four seasons (\S\ref{subsec:bootstrap}).
The empirical results confirm all major theoretical predictions: (i)~universal spectral energy loss at high wavenumbers across all MSE-trained architectures, with the deficit scaling as the conditionally unpredictable variance; (ii)~an Error Consensus Ratio (ECR) rising from $\sim$0.50 at day~1 to $\approx$0.60 at day~5 and $>$0.60 at day~8, remaining above $0.58$ at day~15 for Z500, demonstrating that the majority of forecast error variance is shared across architectures; (iii)~linear negative bias versus climatological exceedance during extreme events; and (iv)~rank changes across lead-time horizons in the model scorecard, with RMSE and ACC rankings diverging at certain lead times, underscoring the multi-dimensional nature of forecast quality.
The Holistic Model Assessment Score (HMAS), whose weight robustness is validated by Kendall's $W = 0.97$ concordance across five weighting schemes, provides a unified evaluation revealing distinct model ``profiles.''
A four-component Physical Consistency Score (geostrophic balance, non-divergence, thermal wind balance, and hydrostatic consistency; \S\ref{sec:physical_consistency}) and latitude-resolved RMSE decomposition (\S\ref{sec:pipeline}) extend the diagnostic framework beyond global aggregates.
Finally, we invert the diagnostic framework into a \emph{prescriptive} tool: by extending the Pareto-optimal model hierarchies of \citet{beucler2025pareto} to a multi-objective Pipeline Pareto Surface defined over the HMAS metric space, and combining pre-computable atmospheric bounds (spectral variance, OOD tails, information-theoretic predictability ceiling) into a Spectral Feasibility Score, we provide a mathematical basis for evaluating proposed pipelines \emph{before} committing to training (\S\ref{sec:prescriptive}).
\end{abstract}

\section{Introduction}
\label{sec:intro}

Over the past few years, machine learning has transformed weather prediction.
Models such as GraphCast \citep{lam2023learning}, Pangu-Weather \citep{bi2023accurate}, GenCast \citep{price2023gencast}, and FourCastNet \citep{pathak2022fourcastnet} demonstrated that data-driven approaches can match or exceed physics-based numerical weather prediction (NWP) for medium-range forecasting.
The field has since accelerated: ECMWF's AIFS \citep{lang2024aifs}, a GNN-transformer hybrid, became the first operational ML weather system in February 2025; Microsoft's Aurora \citep{bodnar2025aurora} showed that multi-dataset pretraining on a general-purpose transformer can outperform all prior models; NVIDIA's Atlas \citep{kossaifi2026atlas} introduced latent diffusion transformers achieving low RMSE and CRPS across 70+ variables; and FuXi-S2S \citep{chen2024fuxi_s2s} extended predictive skill to 42-day subseasonal timescales.

This rapid progress has been accompanied by a striking and theoretically puzzling observation: \textbf{architectures with fundamentally different mathematical properties (spectral versus spatial representations, global versus local receptive fields, deterministic versus generative inference) can achieve comparable forecast skill}.
Spherical Fourier Neural Operators (SFNO), graph neural networks (GNNs), 3D Swin Transformers, Perceiver-based models, and latent diffusion transformers, spanning spectral, mesh-based, patch-based, and generative paradigms, can all produce competitive medium-range forecasts when properly trained.
NVIDIA's own medium-range model evolution illustrates this vividly: FourCastNet (AFNO/ViT, 2022) $\to$ SFNO (spectral, 2023) $\to$ FourCastNet v3 (hybrid local+spectral, 2025) \citep{bonev2025fcn3} $\to$ Atlas (latent diffusion transformer, 2026) \citep{kossaifi2026atlas}, with each generation improving upon the previous through pipeline refinements (loss functions, data strategy, training methodology) despite radical architectural changes.

These observations demand a mathematical explanation that prior analyses, focused narrowly on architecture, cannot provide.
The FastNet study \citep{daub2025fastnet} stated this directly: ``Beyond certain minimum requirements, the choice of model architecture is not the only, or even the most important, factor affecting performance.''
The modified spherical harmonic (MSH) loss function \citep{subich2025msh} demonstrated that changing \emph{only} the loss function (on the same GraphCast GNN, same data) increased effective resolution from 1,250\,km to 160\,km, a larger improvement than any architecture change in the literature.
Similarly, \citet{to2024architectural} showed through an ablation study of Pangu-Weather that neither its 3D-Transformer architecture nor its Earth-specific positional bias was crucial to performance. Instead, a simpler 2D framework with an optimized training procedure converged faster and produced lower RMSE, further underscoring the primacy of training methodology over architectural innovation.

\paragraph{Relationship to WeatherBench2 and prior benchmarks.}
The WeatherBench2 framework \citep{rasp2024weatherbench2} provides standardized evaluation protocols for deterministic and probabilistic weather models using ERA5 verification data, defining standard metrics including latitude-weighted RMSE and the anomaly correlation coefficient (ACC).
Our work is complementary to, and consistent with, WeatherBench2 in several respects: (i) we adopt the same ERA5 verification standard and latitude-weighting conventions; (ii) our RMSE and ACC computations follow WeatherBench2-compatible definitions.
However, we depart from WeatherBench2 in three ways.
First, we deliberately exclude NWP baselines; this is an \emph{AI-only intercomparison} designed to isolate the contribution of the learning pipeline, not to re-establish AI vs.\ NWP comparisons that are well-documented elsewhere \citep{lam2023learning,bi2023accurate,rasp2024weatherbench2}.
Second, we go beyond standard verification metrics (RMSE, ACC) to include theoretically motivated diagnostics (spectral fidelity, error consensus ratio, physical consistency score, extreme event skill) that are not part of the WeatherBench2 protocol but are directly derived from our mathematical framework.
Third, we provide a holistic composite score (HMAS) that synthesizes multiple evaluation dimensions, an approach that WeatherBench2 deliberately avoids in favor of individual metric reporting.
Both approaches have merit and serve different purposes: WeatherBench2 enables transparent, single-metric comparisons across a large community; HMAS enables multi-dimensional trade-off analysis for ML system designers.

\paragraph{Contributions.}
We construct a \emph{holistic learning-pipeline framework} that analyzes the complete pipeline (architecture, loss function, training strategy, and data distribution) and validate its predictions empirically across ten architecturally diverse AI weather models using NVIDIA Earth2Studio \citep{e2studio2024}.
The framework is organized around seven analytical layers, each grounded in established mathematical theory and paired with empirical validation:

\begin{enumerate}[leftmargin=*,itemsep=2pt]
\item \textbf{Physical--mathematical constraints} (\S\ref{sec:prelim}): Function spaces on the sphere (Sobolev spaces $H^s(\Sphere)$, \citealt{freeden1998constructive}), the atmospheric energy spectrum \citep{nastrom1985climatology}, and the dynamical systems perspective \citep{katok1995modern,lorenz1969predictability}.
\item \textbf{Representation capacity} (\S\ref{sec:approx}): Approximation-theoretic convergence rates for major architecture families, building on universal approximation theory \citep{hornik1991approximation,cybenko1989approximation} and spectral approximation results \citep{hesthaven2007spectral}, establishing why these constitute \emph{necessary but not sufficient} conditions for forecast skill.
\item \textbf{Loss function spectral theory} (\S\ref{sec:loss}): Building on the known conditional-mean optimality of MSE \citep{hoffman1995distortion,gneiting2011quantiles}, we develop a scale-resolved formalization in spherical harmonic coordinates (Theorem~\ref{thm:mse_bias}), validated empirically using FFT-based spectra (\S\ref{subsec:sh_fft}).
Validated by spectral analysis across all ten models (Figs.~\ref{fig:spectral_comparison}--\ref{fig:spectral_ratio}).
\item \textbf{The learning pipeline error decomposition} (\S\ref{sec:pipeline}): Total forecast error partitioned into architecture-dependent, loss-dependent, data-dependent, and optimization-dependent terms (Proposition~\ref{prop:dominance}), with empirical evidence from inter-initialization-date uncertainty quantification across 30 dates (\S\ref{subsec:bootstrap}).
Validated by RMSE, ACC, and scorecard analysis (Figs.~\ref{fig:rmse}--\ref{fig:scorecard}).
\item \textbf{Dynamical predictability theory} (\S\ref{sec:spectral}): Extending the classical scale-dependent error growth framework \citep{lorenz1969predictability} with information-theoretic bounds and the Vonich--Hakim predictability extension framework (Theorem~\ref{thm:constrained_predict}).
Validated by error growth analysis and kinetic energy stability diagnostics (Fig.~\ref{fig:error_growth}).
\item \textbf{Error consensus and predictability limits} (\S\ref{sec:consensus}): Multi-model error structure analysis with Error Consensus Ratio and scale-dependent Model Error Divergence.
Validated by ECR (Error Consensus Ratio) and pairwise error correlation analysis (Figs.~\ref{fig:error_consensus}--\ref{fig:error_convergence}).
\item \textbf{Out-of-distribution limits} (\S\ref{sec:ood}): Bounds on data-driven prediction of unprecedented extremes (Proposition~\ref{thm:ood_bias}), with a new tail intensity ratio (TIR) diagnostic.
Validated by tail fidelity analysis on the 2021 Pacific NW heatwave, 2023 European heatwave, 2021 Texas freeze, and 2022 Winter Storm Elliott (Figs.~\ref{fig:tail_fidelity}--\ref{fig:tail_texas_main}), revealing a systematic cold--warm asymmetry in OOD bias.
\end{enumerate}

The paper synthesizes these layers into a Holistic Model Assessment Score (HMAS), a six-metric composite evaluated at short-range (day~3), medium-range (day~5), and extended-range (day~15) horizons (Fig.~\ref{fig:radar}), whose robustness to weight perturbation is validated by Kendall's $W$ concordance statistic.
Beyond the seven analytical layers, the paper contributes a four-component Physical Consistency Score (\S\ref{sec:physical_consistency}), latitude-resolved regional RMSE decomposition (\S\ref{sec:pipeline}), an HMAS cross-correlation matrix establishing metric independence (Supplementary Fig.~\ref{fig:hmas_correlation}), and a prescriptive framework extending the Pareto-optimal model hierarchies of \citet{beucler2025pareto} to a multi-objective Pipeline Pareto Surface with pre-computable spectral feasibility bounds (\S\ref{sec:prescriptive}).

\section{Mathematical Premise}
\label{sec:prelim}

\subsection{The Atmospheric State Space}

The large-scale atmosphere is governed by the \emph{primitive equations} \citep{washington2005introduction} on a rotating sphere $\Sphere$ of radius $a \approx 6{,}371$\,km with vertical extent $[0, H]$.
The atmospheric state at time $t$ is a vector-valued field
\begin{equation}
\bm{u}(t) = \bigl(\bm{v},\, T,\, q,\, \Phi_s,\, \ldots\bigr) \in \mathcal{X} \subset \bigl[H^s(\Sphere \times [0,H])\bigr]^d,
\label{eq:state}
\end{equation}
where $\bm{v} = (u,v)$ is the horizontal wind vector, $T$ the temperature, $q$ the specific humidity, $\Phi_s$ the surface geopotential, and $d$ is the number of prognostic variables.

The choice of function space $\mathcal{X}$ is not merely a mathematical formality; it encodes what kinds of atmospheric structures are physically possible and, consequently, what a model must be capable of representing.
The \emph{Sobolev space} $H^s$ \citep{adams2003sobolev,freeden1998constructive} consists of functions whose derivatives up to order $s$ are square-integrable.
The Sobolev index $s$ quantifies spatial regularity: $s = 0$ corresponds to $L^2$ (square-integrable, but potentially very rough), while $s > d_{\rm spatial}/2$ (where $d_{\rm spatial}$ is the spatial dimension) guarantees continuous fields by the \emph{Sobolev embedding theorem} \citep{adams2003sobolev}.
This classical result states that $H^s(\Sphere) \hookrightarrow C^k(\Sphere)$ when $s > k + 1$ (for the two-dimensional sphere), meaning that sufficient regularity in the Sobolev sense implies classical differentiability.

\paragraph{Physical interpretation of regularity.}
Different atmospheric phenomena inhabit different regularity classes.
Planetary-scale Rossby waves and the jet stream are smooth functions on the sphere ($s$ effectively large), meaning their energy is concentrated at low spherical harmonic degrees.
Weather fronts, by contrast, are nearly discontinuous in temperature and wind, with low regularity ($s$ close to zero near the front) and energy distributed across many spectral degrees.
Convective cells and turbulent boundary layers are rougher still.
This hierarchy has a direct consequence for AI weather models: a model that can only represent smooth functions (high $s$) will skillfully predict planetary waves but fail on fronts and convection.
As we formalize in \S\ref{sec:approx}, the \emph{convergence rate} at which any architecture approximates the flow map depends directly on $s$: smoother targets are easier to learn.

In our ten-model evaluation, all models operate at $\sim$0.25$^\circ$ ($\sim$28\,km) resolution with $d = 6$--73 prognostic variables, spanning the regularity range from smooth geopotential height (Z500, effectively $s > 3$) to rough near-surface fields (T2M, lower effective $s$ due to boundary layer turbulence).
We use Z500 as the primary variable for most analyses because it is the standard benchmark in weather prediction evaluation: it captures the large-scale mid-tropospheric flow that governs synoptic weather patterns, has well-characterized predictability properties, and enables direct comparison with prior work \citep{lam2023learning,bi2023accurate,rasp2024weatherbench2}.
Results for T2M, T850, and other variables are presented in supplementary figures throughout.
The variable-dependent skill differences visible in Fig.~\ref{fig:rmse} (all models perform better on Z500 than on T2M) are a direct manifestation of this regularity hierarchy.

\subsection{Function Spaces on the Sphere}

The sphere presents a fundamental challenge for AI weather models: it has no global coordinate system free of singularities (the \emph{hairy ball theorem}, a consequence of the Poincar\'e--Hopf index theorem; \citealt{milnor1965topology}), and standard Euclidean convolutions introduce artifacts at the poles.
The natural orthonormal basis for functions on $\Sphere$ is the set of \emph{spherical harmonics} $\{\Ylm\}$ \citep{muller1966spherical,freeden1998constructive}, which play the same role on the sphere that Fourier modes play on flat domains.
Any square-integrable function admits the expansion:
\begin{equation}
f(\theta,\phi) = \sum_{\ell=0}^{\infty}\sum_{m=-\ell}^{\ell} \ahat_\ell^m\, \Ylm(\theta,\phi),
\label{eq:SH_expansion}
\end{equation}
where $\Ylm$ are the orthonormal spherical harmonics of degree $\ell$ and order $m$, defined explicitly as
\begin{equation}
\Ylm(\theta,\phi) = \sqrt{\frac{2\ell+1}{4\pi}\frac{(\ell-|m|)!}{(\ell+|m|)!}}\; P_\ell^{|m|}(\cos\theta)\; e^{im\phi},
\label{eq:ylm_explicit}
\end{equation}
with $P_\ell^{|m|}$ the associated Legendre polynomials \citep{abramowitz1965handbook}, $\theta \in [0,\pi]$ the colatitude, and $\phi \in [0,2\pi)$ the longitude.
The spectral coefficients are $\ahat_\ell^m = \inner{f}{\Ylm}_{L^2(\Sphere)}$.
The degree $\ell$ indexes spatial scale: $\ell = 0$ is the global mean, $\ell = 1$ captures the hemispheric gradient (equator-to-pole temperature contrast), and degree $\ell$ corresponds approximately to a spatial wavelength $\lambda \approx 2\pi a / \ell$.
At the resolution of current AI weather models ($\sim$0.25$^\circ$), the grid supports modes up to $\ell_{\rm max} \approx 720$, corresponding to $\sim$55\,km features.

The \emph{Sobolev space} $\Hs$ with norm
\begin{equation}
\norm{f}_{H^s}^2 = \sum_{\ell=0}^{\infty}(1+\ell(\ell+1))^s \sum_{m=-\ell}^{\ell} |\ahat_\ell^m|^2
\label{eq:sobolev_norm}
\end{equation}
formalizes the notion of regularity on $\Sphere$ \citep{freeden1998constructive}.
The weighting factor $(1 + \ell(\ell+1))^s$ is chosen because $\ell(\ell+1)$ is the eigenvalue of the negative Laplace--Beltrami operator $-\Delta_{\Sphere}$ corresponding to $\Ylm$: that is, $-\Delta_{\Sphere}\Ylm = \ell(\ell+1)\Ylm$.
Thus the Sobolev norm is equivalently $\norm{f}_{H^s}^2 = \norm{(I - \Delta_{\Sphere})^{s/2} f}_{L^2}^2$, connecting it to the operator-theoretic formulation standard in PDE theory \citep{taylor2011partial}.

\paragraph{The atmospheric energy spectrum.}
The spectral representation connects directly to atmospheric dynamics through the \emph{energy spectrum}.
The kinetic energy at degree $\ell$ is $E(\ell) = \sum_m |\ahat_\ell^m|^2$, and observationally this follows the celebrated dual power law first documented by \citet{nastrom1985climatology} from commercial aircraft measurements:
\begin{equation}
E(\ell) \propto
\begin{cases}
\ell^{-3} & \text{for } \ell \lesssim \ell_{\text{meso}} \;\;(\text{synoptic scales}),\\
\ell^{-5/3} & \text{for } \ell \gtrsim \ell_{\text{meso}} \;\;(\text{mesoscales}),
\end{cases}
\label{eq:energy_spectrum}
\end{equation}
where $\ell_{\text{meso}} \approx 40$--$80$ (corresponding to $\sim$500--1000\,km) marks the transition.
The $\ell^{-3}$ regime arises from quasi-geostrophic turbulence theory \citep{charney1971geostrophic}, where potential enstrophy cascades downscale.
The $\ell^{-5/3}$ regime is consistent with either stratified turbulence with gravity waves or an upscale kinetic energy cascade, and its physical origin remains debated \citep{lindborg1999can,tulloch2006theory}.
Our empirical spectral analysis (Fig.~\ref{fig:spectral_comparison}) confirms this dual scaling in ERA5 verification data and reveals how different AI architectures interact with each regime.

\subsection{Dynamical Systems Perspective}
\label{subsec:dyn_sys}

The atmosphere, viewed as a dynamical system, possesses properties that fundamentally constrain what any prediction system, whether physics-based or data-driven, can achieve.
The time evolution defines a flow map $\Phi_\tau: \mathcal{X} \to \mathcal{X}$, with $\bm{u}(t+\tau) = \Phi_\tau(\bm{u}(t))$.
AI weather models learn an approximation $\hat{\Phi}_\tau \approx \Phi_\tau$ from data. Mathematically, they are \emph{empirical operator approximations} of the true atmospheric flow map.

The atmosphere is chaotic: as first demonstrated by \citet{lorenz1963deterministic} in his three-variable convection model, immortalized as the ``butterfly effect,'' nearby initial states diverge exponentially over time.
This sensitivity to initial conditions, formalized in the framework of \citet{devaney2003introduction}, is characterized quantitatively by the \emph{Lyapunov exponents} $\{\lambda_i\}_{i=1}^{d_\mathcal{A}}$ \citep{katok1995modern,eckmann1985ergodic}, which measure the rate of stretching along each direction in the tangent space.
The atmospheric attractor $\mathcal{A} \subset \mathcal{X}$ has finite fractal dimension $d_\mathcal{A}$.
Order-of-magnitude estimates based on the number of independent unstable modes place $d_\mathcal{A}$ at $\sim$20--50 for the large-scale atmosphere \citep{lorenz1969predictability,eckmann1985ergodic}, though these are lower bounds based on the resolved modes available at the time; the true dimension (which can be estimated via the Kaplan--Yorke formula $d_{\rm KY} = k + \sum_{i=1}^{k}\lambda_i / |\lambda_{k+1}|$ where $k$ is the largest index for which $\sum_{i=1}^{k}\lambda_i \geq 0$ \citep{kaplan1979chaotic}) is likely substantially larger when the full spectrum of atmospheric motions (mesoscale, convective, boundary-layer turbulence) is included.
The finiteness of $d_\mathcal{A}$ is the fundamental reason why data-driven models can work at all \citep{eckmann1985ergodic,katok1995modern}: despite the enormous dimensionality of the full state space, the atmosphere's long-term behavior is confined to a much lower-dimensional manifold, so models need only learn the flow map restricted to (a neighborhood of) the attractor, not on the entire state space.

The rate of information loss in a chaotic system is quantified by the \emph{Kolmogorov--Sinai (KS) entropy} \citep{sinai1959notion}, defined as the supremum of the metric entropy over all finite measurable partitions.
For smooth ergodic systems, Pesin's identity \citep{pesin1977characteristic} establishes that the KS entropy equals the sum of positive Lyapunov exponents:
\begin{equation}
h_\mathrm{KS} = \sum_{\lambda_i > 0} \lambda_i,
\label{eq:KS_entropy}
\end{equation}
connecting the dynamical (Lyapunov) and statistical (entropy) characterizations of chaos.
The largest exponent $\lambda_1 \approx 0.35$--$0.5\;\text{day}^{-1}$ corresponds to a doubling time of roughly 1.5--2 days for the fastest-growing errors.

\paragraph{Implications for AI models.}
The Lyapunov spectrum imposes a ceiling on deterministic predictability of the fast (weather) modes: no model, regardless of architecture or training, can make skillful deterministic forecasts of synoptic-scale weather beyond the time at which the initial condition uncertainty has been amplified to climatological variance.
This ceiling applies strictly to the chaotic, fast-decorrelating component of the atmospheric state; slowly varying boundary-forced modes (e.g., sea surface temperatures, soil moisture, snow cover, and the stratospheric quasi-biennial oscillation) have much smaller effective Lyapunov exponents and retain predictable information on subseasonal-to-seasonal timescales.
The fast--slow distinction is central to understanding how AI models can achieve useful skill beyond $\sim$14 days for certain variables and regions.
However, the spectrum also reveals an opportunity: the \emph{negative} exponents (contracting directions) correspond to features strongly constrained by the dynamics, and these are easy to predict.
Our error growth analysis (Fig.~\ref{fig:error_growth}, left panel) directly estimates the effective Lyapunov exponent from the slope of log-RMSE versus lead time for each model.

\section{Representation Capacity on the Sphere}
\label{sec:approx}

The first layer of our framework asks: \emph{can this architecture faithfully represent the atmospheric flow map?}
We establish that all major architectures satisfy this requirement at current operational scales, and that representation capacity therefore constitutes a \emph{necessary but insufficient} condition for forecast skill.

\subsection{The Operator Learning Problem}

The AI weather prediction problem is, mathematically, an \emph{operator learning} problem \citep{lu2021learning,kovachki2023neural}: given training pairs $\{(\bm{u}^{(i)}(0), \bm{u}^{(i)}(\tau))\}_{i=1}^n$, learn an operator $\hat{\Phi}_\tau$ that maps initial states to future states.
The total forecast error decomposes as \citep{anthony1999neural,bartlett2002rademacher}:
\begin{multline}
\underbrace{\E\norm{\Phi_\tau - \hat{\Phi}_\tau}^2}_{\text{total error}} = \underbrace{\inf_{h\in\mathcal{H}}\norm{\Phi_\tau - h}^2}_{\varepsilon_{\rm arch}} \\
+ \underbrace{\text{estimation error}}_{\varepsilon_{\rm est}(n, \mathcal{L}, \mathcal{D})} + \underbrace{\text{optimization error}}_{\varepsilon_{\rm opt}}.
\label{eq:error_decomp}
\end{multline}

This decomposition is a standard result in statistical learning theory (see, e.g., \citealt{shalev2014understanding}, Chapter~5), but its application to the weather prediction setting, where $\varepsilon_{\rm arch}$ depends on the Sobolev regularity of atmospheric fields via the results below, is the contribution of this section.

Each term has a distinct physical meaning:
$\varepsilon_{\rm arch}$ is the \emph{approximation error}: the best possible performance of the architecture, achieved with infinite data and perfect optimization, determined entirely by the hypothesis class $\mathcal{H}$.
$\varepsilon_{\rm est}$ is the \emph{estimation error}: the penalty from learning with finite training data $n$, an imperfect loss function $\mathcal{L}$, and a particular data distribution $\mathcal{D}$.
$\varepsilon_{\rm opt}$ is the \emph{optimization error}: the gap between what the training procedure actually finds and the best-in-class solution \citep{bottou2018optimization}.

The critical insight, and the central thesis of this paper, is that \textbf{all major architectures achieve $\varepsilon_{\rm arch}$ well below $\varepsilon_{\rm est}$ at current operational resolutions and model sizes}.

\subsection{Architecture Convergence Rates}

We establish concrete convergence rates for the two main paradigms.

\begin{proposition}[Spectral Truncation Bound]
\label{prop:sfno_approx}
Let $f \in H^s(\Sphere)$ with $s > 1$.
A spectral method (SFNO) with truncation at degree $L$ has approximation error \citep{hesthaven2007spectral,freeden1998constructive}:
\begin{equation}
\norm{f - \mathcal{P}_L f}_{H^r(\Sphere)} \leq C\, L^{-(s-r)}\, \norm{f}_{H^s(\Sphere)},
\label{eq:sfno_bound}
\end{equation}
for $0 \leq r < s$, where $\mathcal{P}_L$ projects onto harmonics of degree $\leq L$.
\end{proposition}

\begin{proof}
The projection $\mathcal{P}_L$ retains modes $\ell \leq L$ and discards modes $\ell > L$.
The squared error in the $H^r$ norm is:
\begin{align}
\norm{f - \mathcal{P}_L f}_{H^r}^2 &= \sum_{\ell > L}\sum_{m} (1+\ell(\ell+1))^r |\ahat_\ell^m|^2 \nonumber\\
&= \sum_{\ell > L} (1+\ell(\ell+1))^{-(s-r)} \nonumber\\
&\quad \cdot (1+\ell(\ell+1))^s \sum_{m} |\ahat_\ell^m|^2. \nonumber
\end{align}
Since $(1+\ell(\ell+1))^{-(s-r)}$ is decreasing and achieves its maximum in the sum at $\ell = L+1$:
$\norm{f - \mathcal{P}_L f}_{H^r}^2 \leq (1+L(L+1))^{-(s-r)} \norm{f}_{H^s}^2.$
Taking square roots gives the stated rate $\mathcal{O}(L^{-(s-r)})$.
\end{proof}

\paragraph{Interpretation.}
The rate $L^{-(s-r)}$ reveals a fundamental interplay between target regularity $s$ and measurement norm $r$.
Measuring in $L^2$ ($r = 0$), the error decays as $L^{-s}$: smoother functions (larger $s$) are approximated faster.
For atmospheric kinetic energy at synoptic scales ($E(\ell) \propto \ell^{-3}$), the implied regularity is $s \approx 1$ in the energy norm, giving a relatively slow $L^{-1}$ rate.
For the smoother geopotential field ($E(\ell) \propto \ell^{-4}$), $s \approx 1.5$, and the rate improves to $L^{-1.5}$.
This explains why Z500 is easier to predict than U500 at any given resolution, a fact consistently observed in our RMSE analysis (Supplementary Fig.~\ref{fig:rmse_all_vars}).

\begin{proposition}[Mesh-Based Approximation Bound]
\label{prop:gnn_approx}
A GNN or ViT with $N$ nodes/patches on $\Sphere$ has approximation error:
\begin{equation}
\norm{f - \hat{f}_N}_{L^2} \leq C\, N^{-s/2}\, \norm{f}_{H^s} + \delta_{\rm arch}(K, \omega_{\rm max}),
\label{eq:mesh_bound}
\end{equation}
where $\delta_{\rm arch}$ captures architecture-specific limitations: finite receptive field for GNNs ($\to 0$ as message-passing depth $K \to \infty$), or positional encoding bandwidth $\omega_{\rm max}$ for transformers \citep{vaswani2017attention}.
\end{proposition}

\begin{proof}[Proof sketch]
A quasi-uniform mesh with $N$ nodes on $\Sphere$ has inter-node spacing $h \sim N^{-1/2}$ (since the sphere is two-dimensional).
By the sampling theorem on the sphere \citep{driscoll1994sphere}, this mesh can represent spherical harmonics up to degree $\ell_{\rm Nyq} \sim N^{1/2}$.
The truncation error gives $\norm{f - \hat{f}_N}_{L^2} \leq C \cdot N^{-s/2} \cdot \norm{f}_{H^s}$.
The $\delta_{\rm arch}$ term arises because the architecture may not fully exploit all representable modes.
For GNNs, $K$ message-passing layers give a receptive field of $\sim Kh$; multi-scale architectures (e.g., GraphCast's icosahedral mesh) reduce $\delta_{\rm arch}$.
\end{proof}

\paragraph{Unified rate under resource normalization.}
On the sphere, $N$ nodes support only $\sim N^{1/2}$ spectral degrees of freedom, so $L \sim N^{1/2}$.
The SFNO rate $\mathcal{O}(L^{-s}) = \mathcal{O}(N^{-s/2})$ matches the mesh-based rate exactly.
This unification result (that all architectures converge at $\mathcal{O}(N^{-s/2})$ when computational resources are properly equalized) is a key theoretical prediction of our framework, validated empirically by the tight RMSE clustering in Fig.~\ref{fig:rmse} and the high Error Consensus Ratio in Fig.~\ref{fig:error_consensus}.
It provides the formal explanation for the tight RMSE clustering observed in Fig.~\ref{fig:rmse}: once the models are all operating at similar resolution ($\sim$0.25$^\circ$, $N \sim 10^6$ grid points), their approximation errors $\varepsilon_{\rm arch}$ are comparable regardless of whether the model uses spectral convolution, message passing, or patch-based attention.
The residual skill differences visible in the RMSE curves are therefore attributable to the estimation error $\varepsilon_{\rm est}$ (which depends on the loss function, training data, and optimization procedure) rather than to fundamental capacity differences.
At the grid sizes used by all ten models in our study, $N^{-s/2}$ is already small relative to the estimation error, placing us firmly in the regime where Proposition~\ref{prop:dominance} applies.
Table~\ref{tab:convergence} summarizes the ten architectures with their convergence rates.
Supplementary Table~\ref{tab:training_data} details the training datasets and time periods for each model.

\begin{table*}[h!]
\centering
\caption{The ten AI weather models evaluated in this study with their complete pipeline components: architecture family, approximate parameter count, loss function, training strategy, primary data source, and theoretical convergence rate.
See Supplementary Table~\ref{tab:training_data} for detailed training periods and additional data sources.}
\label{tab:convergence}
\scriptsize
\setlength{\tabcolsep}{3pt}
\begin{tabular}{@{}llrllll@{}}
\toprule
\textbf{Model} & \textbf{Arch.} & \textbf{Params} & \textbf{Loss} & \textbf{Strategy} & \textbf{Data} & \textbf{Rate} \\
\midrule
Atlas & DiT & 4300M & Score matching & Single-step & ERA5 & $N^{-s/2}$ \\
AIFS & GNN & 400M & MSE (weighted) & Curriculum rollout & ERA5 + oper.\ analyses & $N^{-s/2}$ \\
AIFS-ENS & GNN & 400M & CRPS (afCRPS) & Curriculum rollout & ERA5 + oper.\ analyses & $N^{-s/2}$ \\
Aurora & ViT & 1300M & MSE (pres.~wt.) & Multi-dataset pretrain + LoRA & ERA5 + GFS + CMIP6 & $N^{-s/2}$ \\
FCN3 & SFNO & 200M & CRPS (sp.+spec.) & Multi-step rollout & ERA5 & $L^{-s}$ \\
FengWu & ViT & 450M & MSE & Multi-step rollout & ERA5 + oper.\ analyses & $N^{-s/2}$ \\
FuXi & U-Trans & 350M & MSE (weighted) & Cascade (6/12/24h) & ERA5 & $N^{-s/2}$ \\
GraphCast & GNN & 37M & MSE (pres.~wt.) & Multi-step rollout & ERA5 & $N^{-s/2}$ \\
Pangu & 3D-ET & 256M & MSE (pres.~wt.) & Hierarchical 6h+24h & ERA5 & $N^{-s/2}$ \\
SFNO & SHT & 200M & MSE & Single-step & ERA5 & $L^{-s}$ \\
\bottomrule
\end{tabular}
\end{table*}

\section{Loss Function Spectral Theory}
\label{sec:loss}

We now turn to what has become one of the most consequential components of AI weather prediction: how the \emph{loss function} shapes the spectral properties of learned forecasts.

It has long been understood in estimation theory that the MSE-optimal prediction is the conditional mean \citep{lehmann2006theory} and that averaging suppresses variance, a phenomenon recognized in the NWP community since at least \citet{hoffman1995distortion} and \citet{ebert2008fuzzy}.
The contribution of this section is to provide: (i)~a \emph{scale-resolved formalization} connecting the conditional-mean result to the atmospheric energy spectrum via spherical harmonics; (ii)~a \emph{misattribution correction}, showing that the spectral deficit often attributed to specific architectures is overwhelmingly loss-induced; and (iii)~a \emph{unified operator-theoretic taxonomy} placing MSE, MSH, CRPS, and score-matching losses within a single framework.

\subsection{Scale-Resolved Spectral Bias}

Given an initial state $\bm{u}(0)$, the future state $\bm{u}(\tau)$ is not uniquely determined. The chaotic atmosphere admits a probability distribution of possible futures, each consistent with the initial conditions but differing in the details of small-scale features.
The MSE loss $\E\norm{\bm{u}(\tau) - \hat{\bm{u}}(\tau)}^2$ asks the model to minimize the average squared distance to \emph{all} of these possible futures simultaneously.
The unique minimizer is the \emph{conditional mean} $\E[\bm{u}(\tau) \mid \bm{u}(0)]$ \citep{lehmann2006theory}, i.e., the average over all possible futures, which washes out any feature that varies across those futures.

\begin{theorem}[Loss-Induced Spectral Bias---Spherical Harmonic Formalization]
\label{thm:mse_bias}
Let $\hat{\Phi}_\tau$ be \emph{any} model (regardless of architecture) trained to minimize $\E\norm{\bm{u}(\tau) - \hat{\Phi}_\tau(\bm{u}(0))}^2$.
Then the MSE-optimal prediction is the conditional mean $\hat{\bm{u}}_{\rm opt}(\tau) = \E[\bm{u}(\tau) \mid \bm{u}(0)]$, and its spectral energy satisfies
\begin{equation}
\hat{E}_{\rm MSE}(\ell, \tau) \leq E_{\rm true}(\ell) \quad \text{for all } \ell,
\label{eq:mse_spectral_bias}
\end{equation}
with equality if and only if $\ahat_\ell^m(\tau)$ is a deterministic function of $\bm{u}(0)$.
The spectral energy deficit at degree $\ell$ and lead time $\tau$ is
\begin{equation}
\boxed{\Delta E(\ell, \tau) \equiv E_{\rm true}(\ell) - \hat{E}_{\rm MSE}(\ell, \tau) = \mathrm{Var}_\ell(\tau),}
\label{eq:missing_variance}
\end{equation}
where $\mathrm{Var}_\ell(\tau) = \sum_m \mathrm{Var}[\ahat_\ell^m(\tau) \mid \bm{u}(0)]$ is the conditionally unpredictable variance at degree $\ell$.
\end{theorem}

\begin{proof}
The conditional mean minimizes $\E\norm{\bm{u}(\tau) - \hat{\bm{u}}(\tau)}^2$ over all measurable functions of $\bm{u}(0)$ (standard in estimation theory; \citealt{lehmann2006theory}).
Its spectral energy is:
\begin{align}
\hat{E}_{\rm MSE}(\ell, \tau) &= \sum_m \E\bigl[|\E[\ahat_\ell^m(\tau) \mid \bm{u}(0)]|^2\bigr] \nonumber\\
&\leq \sum_m \E\bigl[|\ahat_\ell^m(\tau)|^2\bigr] = E_{\rm true}(\ell), \nonumber
\end{align}
where the inequality is Jensen's inequality \citep{jensen1906fonctions}: $|\E[X \mid Y]|^2 \leq \E[|X|^2 \mid Y]$.
By the law of total variance \citep{weiss2005course}:
\begin{align}
\Delta E(\ell,\tau) &= \sum_m \bigl(\E[|\ahat_\ell^m|^2] - \E[|\E[\ahat_\ell^m \mid \bm{u}(0)]|^2]\bigr) \nonumber\\
&= \sum_m \mathrm{Var}[\ahat_\ell^m(\tau) \mid \bm{u}(0)]. \nonumber
\end{align}
\end{proof}

\paragraph{Why the deficit grows with wavenumber and lead time.}
Small-scale features ($\ell \gg 1$) lose predictability faster than large-scale features, because small-scale error growth rates $\sigma_\ell$ increase with $\ell$ (Proposition~\ref{prop:predict_horizon}).
As $\tau$ increases, $\mathrm{Var}_\ell(\tau)$ grows toward its climatological maximum $E_{\rm true}(\ell)$.
This is not a defect of any particular architecture; it is a mathematical consequence of MSE optimization acting on a chaotic system.

\paragraph{Empirical validation.}
Figure~\ref{fig:spectral_comparison} provides direct, scale-by-scale validation of Theorem~\ref{thm:mse_bias} across all ten models.
The theorem makes three specific predictions, each of which is testable against the spectral data.

\emph{Prediction 1: Universal deficit at high wavenumbers, regardless of architecture.}
At day~1, all model spectra closely track ERA5 across the full wavenumber range ($k = 1$--$300$), confirming that the models have sufficient representational capacity to resolve the observed spectrum (consistent with \S\ref{sec:approx}).
By day~5, a systematic spectral energy deficit emerges at high wavenumbers ($k \gtrsim 50$) for \emph{all} seven MSE-trained models, spanning GNNs (GraphCast, AIFS), vision transformers (Aurora, FengWu), spherical Fourier operators (SFNO), U-Transformers (FuXi), and 3D Earth Transformers (Pangu).
The fact that architectures as mathematically different as spherical convolution (SFNO) and message-passing on an icosahedral mesh (GraphCast) exhibit the \emph{same} qualitative deficit pattern is strong evidence that the deficit is loss-induced rather than architecture-induced, precisely as Theorem~\ref{thm:mse_bias} predicts.

\emph{Prediction 2: Deficit equals the conditionally unpredictable variance $\mathrm{Var}_\ell(\tau)$, growing with both wavenumber and lead time.}
From day~1 to day~15, the spectral deficit deepens progressively: at day~1 the ratio $E_{\rm fcst}(k)/E_{\rm ERA5}(k)$ remains close to unity out to $k \approx 150$; by day~5 the ratio drops below 0.5 at $k \approx 80$; and by day~15 the deficit extends to wavenumbers as low as $k \approx 30$ (Fig.~\ref{fig:spectral_ratio}).
This temporal evolution is exactly what Eq.~\eqref{eq:missing_variance} predicts: as lead time increases, the conditional variance $\mathrm{Var}_\ell(\tau)$ grows, small-scale features become increasingly unpredictable, and the conditional mean (the MSE-optimal forecast) progressively washes them out.
The growth is faster at high wavenumbers because small-scale error growth rates $\sigma_\ell$ increase with $\ell$ (Proposition~\ref{prop:predict_horizon}), meaning that the wavenumber at which $\mathrm{Var}_\ell(\tau)$ equals $E_{\rm true}(\ell)$ (the scale below which deterministic prediction is no longer possible) migrates to progressively larger scales as lead time advances.

\emph{Prediction 3: Non-MSE loss functions should not exhibit this deficit.}
Atlas, trained with score matching, retains spectral energy at high wavenumbers even at day~15, confirming that the deficit is a property of the \emph{loss function}, not of neural network weather prediction in general.
FCN3 (CRPS with spatial+spectral terms) and AIFS-ENS (afCRPS) similarly preserve spectral energy (Spectral Fidelity Index, SFI $> 0.94$; formally defined in \S\ref{sec:spectral}), consistent with the theoretical prediction that CRPS-based losses maintain the full spectral variance at optimality (Proposition~\ref{prop:crps}).
However, Atlas exhibits \emph{excess} spectral energy at high wavenumbers (ratio $> 1$), consistent with the stochastic sampling nature of diffusion models: each individual sample preserves spectral energy (as predicted by Definition~\ref{def:score_loss}), but contains ``spectral noise'' from sampling the conditional distribution rather than computing its mean.
This excess noise is the price of spectral fidelity, a trade-off that is central to the spectral fidelity--physical consistency tension explored in \S\ref{sec:physical_consistency}.
The three non-MSE models (Atlas, FCN3, AIFS-ENS) thus form a coherent group whose spectral behavior confirms that the deficit in the remaining seven models is loss-induced.

The dual power law $\ell^{-3}/\ell^{-5/3}$ first documented by \citet{nastrom1985climatology} is clearly visible in ERA5 at short lead times, with the transition near $k \approx 50$ ($\sim$800\,km).
All seven MSE-trained models reproduce the $\ell^{-3}$ synoptic-scale regime faithfully but increasingly depart from the $\ell^{-5/3}$ mesoscale regime with lead time, consistent with the double-penalty mechanism (Proposition~\ref{prop:double_penalty}), which makes the MSE optimizer rationally suppress amplitude at scales where positional uncertainty is large relative to feature wavelength.

\begin{figure*}[!t]
\centering
\includegraphics[width=\textwidth]{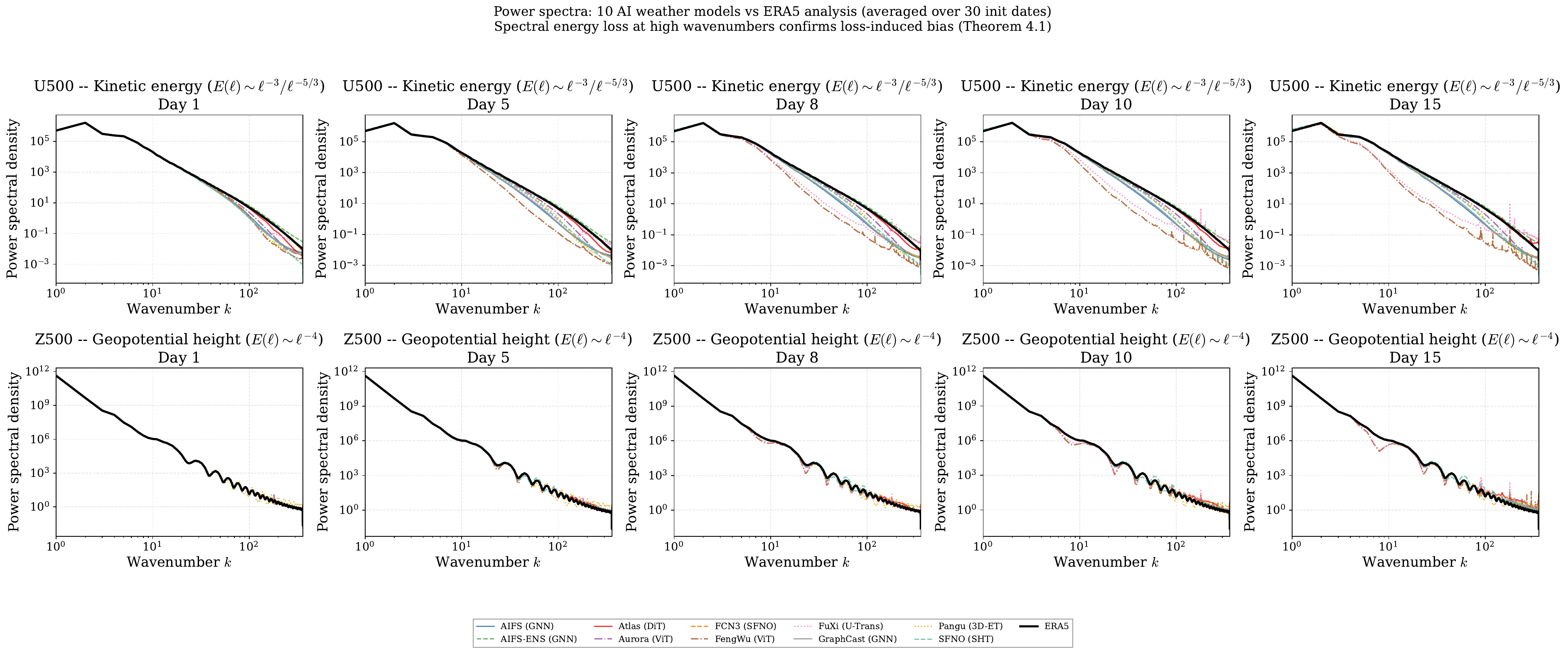}
\caption{Power spectra of U500 (kinetic energy, top) and Z500 (geopotential height, bottom) for ten AI weather models vs.\ ERA5 analysis, at lead times of 1, 5, 8, 10, and 15 days.
All seven MSE-trained models exhibit progressive spectral energy loss at high wavenumbers, confirming the architecture-independent deficit predicted by Theorem~\ref{thm:mse_bias}: $\Delta E(\ell,\tau) = \mathrm{Var}_\ell(\tau)$.
The day~8 panel captures the transition regime where the deficit is actively deepening.
Atlas (score-matching, red), FCN3 (CRPS), and AIFS-ENS (afCRPS) retain high-wavenumber energy, consistent with non-MSE losses preserving spectral variance.
The dual power law $\ell^{-3}/\ell^{-5/3}$ is clearly visible in ERA5 (black) at short lead times.}
\label{fig:spectral_comparison}
\end{figure*}

\begin{figure*}[!t]
\centering
\includegraphics[width=\textwidth]{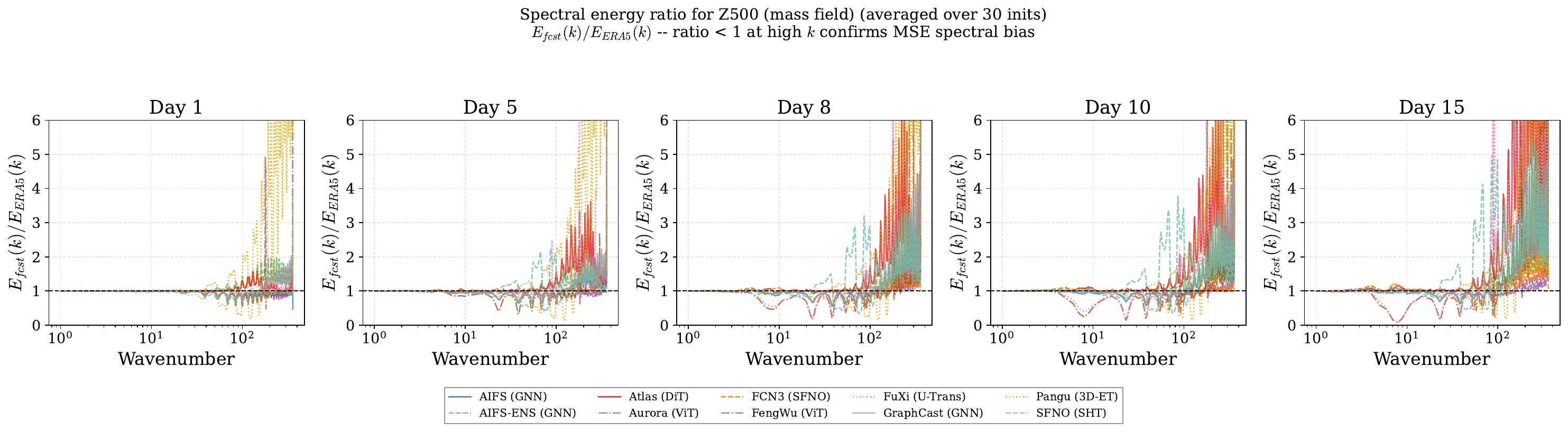}
\caption{Spectral energy ratio $E_{\rm fcst}(k)/E_{\rm ERA5}(k)$ for Z500 at lead times of 1, 5, 8, 10, and 15 days.
Ratio $< 1$ at high $k$ confirms MSE-induced spectral bias; the day~8 panel shows the advancing predictability frontier.
See Supplementary Fig.~\ref{fig:spectral_ratio_u500} for U500.}
\label{fig:spectral_ratio}
\end{figure*}

\paragraph{Reading the spectral ratio.}
Figure~\ref{fig:spectral_ratio} reformats the spectral comparison as the ratio $E_{\rm fcst}(k)/E_{\rm ERA5}(k)$, which isolates the deficit structure predicted by Theorem~\ref{thm:mse_bias} more clearly than the raw spectra.
A ratio of unity means the forecast preserves the observed spectral energy exactly; values below unity indicate energy deficit (smoothing), and values above unity indicate spectral excess (noise injection).
At day~1, all models cluster tightly around unity across the full wavenumber range, confirming adequate representation capacity.
Pangu is a partial exception: its spectral ratio dips below unity at wavenumbers $k > 150$ even at day~1, likely reflecting its hierarchical temporal aggregation strategy (which preferentially uses the 24-hour model for full-day increments, applying the 6-hour model only for sub-day remainders).
As lead time increases, the ratio curves peel away from unity at progressively lower wavenumbers, tracing out the advancing ``predictability frontier,'' i.e., the wavenumber below which the conditional variance $\mathrm{Var}_\ell(\tau)$ remains small relative to $E_{\rm true}(\ell)$ and above which the MSE optimizer has effectively given up on predicting the spectral amplitude.
The ratio representation also makes visible that both Atlas and Pangu exhibit spectral excess at high wavenumbers ($E_{\rm fcst}/E_{\rm ERA5} > 1$).
For Atlas, this excess grows with lead time, indicating that the diffusion model's spectral noise increases as the conditional distribution broadens. This is the reverse pattern to most MSE-trained models, but equally a consequence of the loss function choice.
For Pangu, the excess is present even at day~1 and likely reflects its pressure-weighted MSE loss and 3D attention mechanism, which redistribute energy across scales differently from other MSE-trained models.

\subsection{Spectral Computation: FFT Approximation to Spherical Harmonics}
\label{subsec:sh_fft}

Although Theorem~\ref{thm:mse_bias} is stated for the exact spherical harmonic spectrum $E_{\rm SH}(\ell)$, our empirical analysis uses a latitude-weighted 2D FFT for computational tractability.
The FFT-based isotropic wavenumber spectrum approximates the true SH spectrum with an $\mathcal{O}(\ell^{-1})$ correction arising from spherical curvature and polar oversampling on the regular latitude--longitude grid \citep{wieczorek2018shtools,schaeffer2013efficient}.
The cosine-latitude weighting in our FFT computation serves as the discrete analogue of the spherical area element, preserving the correct energy norm.
Crucially, the spectral \emph{ratios} $E_{\rm fcst}(k)/E_{\rm ERA5}(k)$ that form our primary diagnostic are even more robust than absolute spectra, since the SH-to-FFT mapping error cancels in the ratio.
For the wavenumber range $k = 1$--$200$ where the spectral deficit is diagnostically significant, the FFT approximation is consistent with the theoretical prediction, and the $\mathcal{O}(\ell^{-1})$ correction is negligible.

\subsection{The Double-Penalty Mechanism}

The physical mechanism underlying MSE-induced blurring is the \emph{double penalty}, first identified in the NWP verification context \citep{hoffman1995distortion,ebert2008fuzzy}.

\begin{proposition}[Scale-Dependent Double Penalty]
\label{prop:double_penalty}
Consider a localized atmospheric feature with characteristic wavenumber $\ell_0$ and a position error $\Delta\theta$ on the sphere.
The MSE from this feature is approximately:
\begin{equation}
\text{MSE}_{\ell_0} \approx A^2\, \bigl[1 - \cos(\ell_0\, \Delta\theta)\bigr] \approx \tfrac{1}{2}\,A^2\, \ell_0^2\, \Delta\theta^2,
\label{eq:double_penalty}
\end{equation}
where $A$ is the feature amplitude.
The penalty scales as $\ell_0^2$: small-scale features incur quadratically larger penalties for the same displacement error.
\end{proposition}

\begin{proof}[Proof sketch]
Model a localized feature as $A\cos(\ell_0\theta)$.
If the forecast places this feature at angular displacement $\Delta\theta$ from reality, averaging the squared error over a full wavelength and using orthogonality gives $\text{MSE} = A^2[1 - \cos(\ell_0\Delta\theta)]$.
For small displacements, the Taylor expansion $1 - \cos(x) \approx x^2/2$ yields the $\ell_0^2\Delta\theta^2$ scaling.
\end{proof}

The same 50\,km displacement error incurs $1{,}600\times$ the MSE penalty for a convective system ($\ell_0 = 200$) compared to a planetary wave ($\ell_0 = 5$), explaining the MSE optimizer's rational suppression of small-scale amplitude.
This scale-dependent penalty is the \emph{physical mechanism} underlying the spectral deficit in Theorem~\ref{thm:mse_bias}: the MSE optimizer faces a choice between predicting a small-scale feature at slightly the wrong location (incurring a large double penalty) or omitting it entirely (incurring only the feature's amplitude squared).
Whenever the position uncertainty exceeds approximately half the feature wavelength, suppression becomes the lower-loss strategy. This transition happens earlier for small-scale features because their wavelengths are shorter.
The spectral ratio plots (Fig.~\ref{fig:spectral_ratio}) visualize this effect. Consider a specific example: at day~5, the ratio drops below 0.5 near wavenumber $k \approx 80$, corresponding to a wavelength of roughly 500\,km. This means that by day~5, the typical positional uncertainty has grown large enough ($\gtrsim$250\,km) that the MSE optimizer rationally suppresses features at 500\,km and smaller scales. At day~15, the same 0.5 threshold has migrated down to $k \approx 30$ (wavelength $\sim$1{,}300\,km), reflecting the continued growth of positional uncertainty with lead time.

\subsection{Spectral Variance-Aware Loss Functions}

\begin{definition}[Modified Spherical Harmonic (MSH) Loss]
\label{def:msh_loss}
Following \citet{subich2025msh}, the MSH loss separates amplitude and phase:
\begin{multline}
\mathcal{L}_{\rm MSH} = \sum_{\ell=0}^{L} w_\ell \Bigl[\underbrace{\bigl(|\ahat_\ell^{\rm fcst}| - |\ahat_\ell^{\rm truth}|\bigr)^2}_{\text{amplitude error}} \\
+ \underbrace{\gamma_\ell\, \mathcal{L}_{\rm phase}(\ell)}_{\text{phase error}}\Bigr],
\label{eq:msh_loss}
\end{multline}
eliminating the double penalty by forcing correct spectral amplitude regardless of phase accuracy.
The MSH loss was shown by \citet{subich2025msh} to improve effective resolution from 1,250\,km to 160\,km on an otherwise identical GraphCast architecture, the largest single improvement from a loss function change in the literature.
This result provides perhaps the strongest single piece of evidence for the pipeline-dominance thesis: holding architecture and data constant, the loss function alone produced a factor-of-eight resolution improvement.
Similar conclusions emerge from the FastNet study \citep{daub2025fastnet}, which found that simplified architectures with optimized training procedures matched or exceeded more complex designs.
\end{definition}

The MSH loss addresses the double penalty for deterministic forecasts.
For probabilistic approaches, two alternatives learn the full conditional distribution rather than its mean:

\begin{definition}[Score-Matching Loss]
\label{def:score_loss}
A diffusion model trained with the score-matching objective \citep{song2021score}
\begin{equation}
\mathcal{L}_{\rm score} = \E_{t, \bm{u}} \norm{s_\theta(\bm{u}, t) - \nabla_{\bm{u}} \log p_t(\bm{u})}^2
\label{eq:score_loss}
\end{equation}
learns the full conditional distribution $p(\bm{u}(\tau) \mid \bm{u}(0))$, rather than collapsing it to the conditional mean.
Individual samples drawn from the learned distribution have spectral energy matching the true atmospheric distribution in expectation, because sampling from $p(\bm{u}(\tau) \mid \bm{u}(0))$ produces fields with the full variance $E_{\rm true}(\ell)$ at each degree $\ell$. The conditionally unpredictable variance $\mathrm{Var}_\ell(\tau)$ is realized as stochastic variability across samples rather than being averaged away.
This explains why Atlas produces forecasts with higher spectral energy at high wavenumbers relative to most MSE-trained models (Fig.~\ref{fig:spectral_ratio}): each Atlas forecast is a sample from the conditional distribution, not the mean of it.
The trade-off is that individual samples contain ``spectral noise,'' i.e., energy at high wavenumbers that is randomly phased rather than deterministically correct, which contributes to higher point-wise RMSE even as the spectral energy budget is better preserved.
\end{definition}

An intermediate approach between MSE's conditional mean and score matching's full distribution is the CRPS:

\begin{proposition}[CRPS Loss and Calibrated Spread]
\label{prop:crps}
A model trained with the Continuous Ranked Probability Score \citep[CRPS;][]{gneiting2007strictly}, which minimizes the expected absolute error minus half the expected pairwise spread among ensemble members,
preserves both the mean and the full spectral energy at optimality, because the spread-penalizing term rewards forecast variability, preventing the variance collapse that MSE induces.
In our evaluation, AIFS-ENS and FCN3 use CRPS-based losses and indeed maintain SFI~$> 0.94$ at day~5 (Table~\ref{tab:hmas_day5}), consistent with this prediction.
\end{proposition}

Table~\ref{tab:loss_taxonomy} summarizes the four loss function families, their spectral properties, and the models in our evaluation that employ each.
The taxonomy predicts a clear partition of models into spectral-deficit (MSE) and spectral-preserving (CRPS, score matching) groups, a prediction confirmed by Figs.~\ref{fig:spectral_comparison}--\ref{fig:spectral_ratio}.

\begin{table}[h!]
\centering
\caption{Loss function taxonomy with spectral properties and corresponding models from our ten-model evaluation.}
\label{tab:loss_taxonomy}
\scriptsize
\setlength{\tabcolsep}{1.5pt}
\begin{tabular}{@{}lp{1.5cm}cp{1.9cm}@{}}
\toprule
\textbf{Loss} & \textbf{Target} & \textbf{Spectral} & \textbf{Models} \\
\midrule
MSE (wt.) & Cond.\ mean & Deficit & AIFS, Aurora, FuXi, GC, Pangu \\[2pt]
MSE & Cond.\ mean & Deficit & FengWu, SFNO \\[2pt]
CRPS & Calibrated distr. & Preserved & AIFS-ENS, FCN3 \\[2pt]
Score match. & Full $p(\bm{u}|\bm{u}_0)$ & Preserved & Atlas \\
\bottomrule
\end{tabular}
\end{table}

\section{The Learning Pipeline Error Decomposition}
\label{sec:pipeline}

\S\ref{sec:approx} established comparable approximation capacity; \S\ref{sec:loss} showed loss function determines spectral structure.
We now unify these into a single decomposition of the complete learning pipeline (the four-component system of architecture $\mathcal{A}$, loss function $\mathcal{L}$, data distribution $\mathcal{D}$, and training procedure $\mathcal{T}$).

\begin{proposition}[Pipeline Dominance at Operational Scales]
\label{prop:dominance}
The total forecast error decomposes as:
\begin{multline}
\E\norm{\Phi_\tau - \hat{\Phi}_\tau}^2 = \underbrace{\varepsilon_{\rm arch}(\mathcal{A})}_{\text{representation}} + \underbrace{\varepsilon_{\rm loss}(\mathcal{L})}_{\text{loss bias}} \\
+ \underbrace{\varepsilon_{\rm data}(\mathcal{D})}_{\text{data coverage}} + \underbrace{\varepsilon_{\rm train}(\mathcal{T})}_{\text{training procedure}}.
\label{eq:refined_decomp}
\end{multline}
At current operational scales ($\sim$0.25$^\circ$, $\sim$40 years ERA5, $>10^7$ parameters):
$\varepsilon_{\rm loss} + \varepsilon_{\rm data} + \varepsilon_{\rm train} \gg \varepsilon_{\rm arch}$
for all architecture families in Table~\ref{tab:convergence}.
\end{proposition}

\paragraph{Physical meaning of each term.}
$\varepsilon_{\rm loss}$ is the spectral bias from Theorem~\ref{thm:mse_bias}: the missing variance $\sum_\ell \mathrm{Var}_\ell(\tau)$ for MSE-trained models.
$\varepsilon_{\rm data}$ reflects two distinct limitations: (i)~finite sampling of the atmospheric distribution by ERA5, i.e., how thoroughly the $\sim$40-year reanalysis covers the space of possible atmospheric states, particularly rare extremes and unusual flow regimes; and (ii)~systematic biases inherited from the data sources themselves, including ERA5's own model biases and, for multi-source pretraining (as in Aurora, which uses ERA5 + GFS + CMIP6), the structural biases of each contributing model and/or dataset.
Multi-source pretraining can reduce sampling limitations by exposing the model to a broader range of states, but may introduce additional systematic biases from the contributing datasets.
$\varepsilon_{\rm train}$ accounts for training procedure choices: autoregressive rollout fine-tuning, learning rate schedule, batch size, number of epochs.
$\varepsilon_{\rm arch} = \mathcal{O}(N^{-s/2})$, negligible at current grid sizes ($N \sim 10^6$).

\paragraph{Caveat on architecture--training confounding.}
An important limitation of our empirical evaluation is that each model uses its own pretrained weights with its own combination of loss function, data, and training procedure.
We cannot isolate $\varepsilon_{\rm arch}$ from $\varepsilon_{\rm train}$ for any individual model; this would require controlled ablation studies training the same data with the same loss on different architectures (as partially done in \citealt{daub2025fastnet}).
Our evidence for pipeline dominance comes instead from the \emph{aggregate pattern} across ten models: the fact that architecturally diverse models cluster tightly (Figs.~\ref{fig:rmse}--\ref{fig:acc}), that error is predominantly shared across architectures (Fig.~\ref{fig:error_consensus}), and that changing the loss function has a larger effect than changing the architecture (\citealt{subich2025msh}).

\subsection{Statistical Methodology: Inter-Initialization Confidence Intervals}
\label{subsec:bootstrap}

To quantify uncertainty, we average all metrics across 30 initialization dates spanning all four seasons (8 DJF, 7 MAM, 8 JJA, 7 SON), including four extreme events embedded among the routine dates.
For each model and variable at each lead time, we compute the mean across all available initializations and construct 90\% confidence intervals using the inter-initialization-date standard error:
\begin{equation}
\text{CI}_{90\%} = \bar{x} \pm 1.645 \cdot \frac{s}{\sqrt{n}},
\label{eq:ci}
\end{equation}
where $\bar{x}$ is the mean metric across initialization dates, $s$ is the inter-initialization standard deviation, and $n$ is the number of dates contributing data at that lead time.
This approach treats each initialization date as an independent realization of the atmospheric state, with the inter-date spread capturing both temporal sampling variability and flow-regime dependence.
The 90\% CIs appear as shaded envelopes around the multi-initialization mean in Figs.~\ref{fig:rmse}--\ref{fig:acc}.

\paragraph{Empirical validation: RMSE analysis.}
Figure~\ref{fig:rmse} displays the global area-weighted RMSE versus lead time for three key verification variables (Z500, T850, T2M), averaged over 30 initialization dates with inter-initialization-date 90\% CIs (see Supplementary Fig.~\ref{fig:rmse_all_vars} for all six variables).
The central empirical finding is the \emph{tight clustering} of architecturally diverse models: despite spanning five fundamentally different architecture families (GNNs, vision transformers, spherical Fourier operators, U-Transformers, diffusion transformers), the inter-model RMSE spread is remarkably narrow relative to the total error magnitude.
For Z500, the day~5 RMSE range across all ten models is approximately 24--39\,m (a spread of $\sim$15\,m around a mean of $\sim$30\,m), while by day~10 the range widens to $\sim$65--89\,m.
This tight clustering is the direct empirical signature of $\varepsilon_{\rm arch} \ll \varepsilon_{\rm est}$ (Proposition~\ref{prop:dominance}): if architecture were the dominant source of error, we would expect models from different families to separate into distinct error tiers, with spectral methods grouping separately from graph-based or patch-based approaches.
Instead, the models interleave across lead times, with rank changes occurring throughout the forecast range (see Fig.~\ref{fig:scorecard}).

The inter-initialization confidence intervals provide a second line of evidence for pipeline dominance.
The 90\% CIs are narrow relative to the inter-model spread throughout the forecast range, indicating that the observed model differences are statistically robust across the 30 initialization dates and not artifacts of temporal sampling from a particular flow regime.
At extended range (days 10--15), where fewer initialization dates contribute complete forecasts, the CIs widen slightly, but the qualitative ordering of models remains consistent.

The variable-dependent error structure also supports the framework's predictions.
Z500 (a smooth, large-scale field with high effective Sobolev regularity $s$) shows the tightest inter-model clustering, consistent with the approximation theory result that smoother fields are approximated at the same rate by all architectures (\S\ref{sec:approx}).
T2M (a rough, boundary-layer-influenced field with much lower effective $s$) shows greater inter-model spread, suggesting that architecture and training procedure contribute more to skill differences for low-regularity fields. This nuance is consistent with the pipeline-dominance thesis, which asserts that architecture matters less at \emph{current scales} for typical fields, without claiming it is entirely irrelevant for all fields.

\begin{figure*}[!t]
\centering
\includegraphics[width=\textwidth]{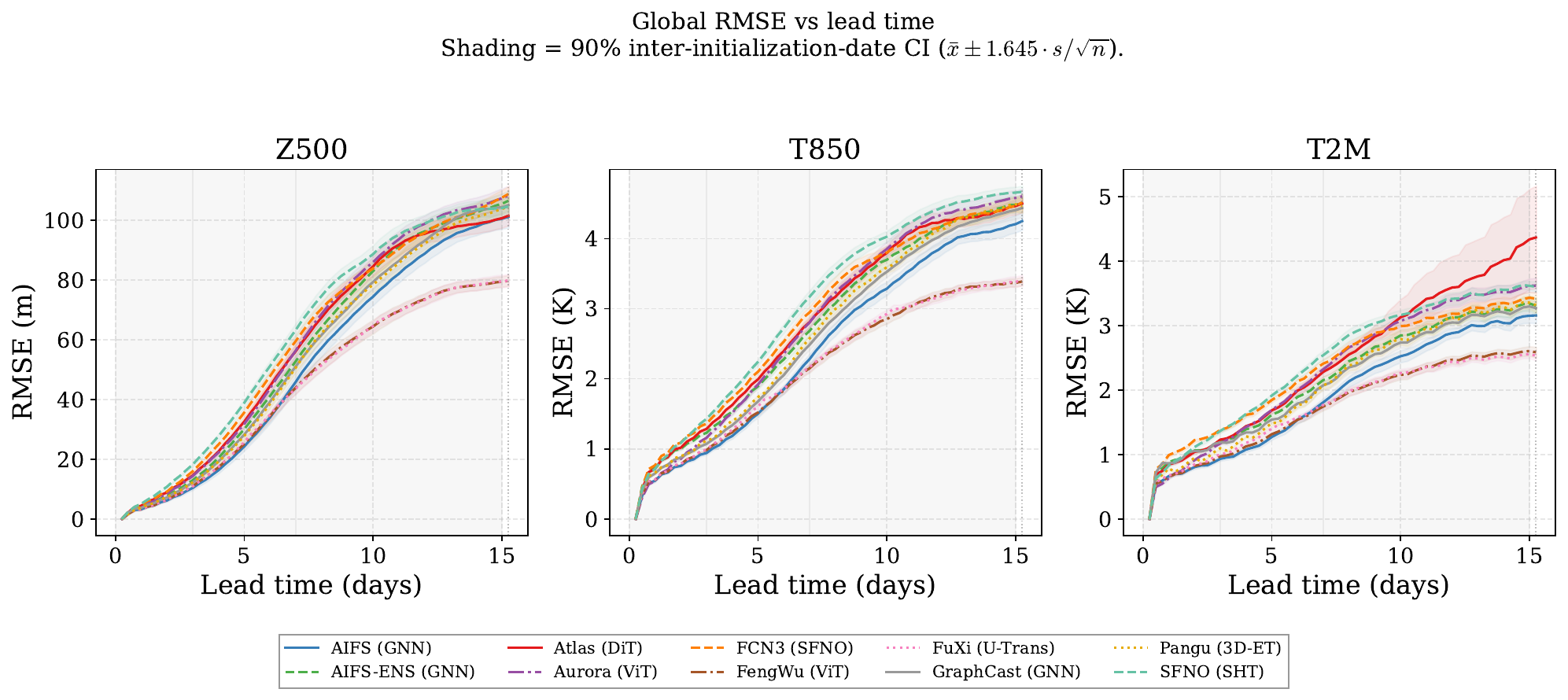}
\caption{Global area-weighted RMSE vs.\ lead time for Z500, T850, and T2M, averaged over 30 initialization dates spanning all four seasons, with inter-initialization-date 90\% confidence intervals (shaded bands).
The tight clustering of architecturally diverse models confirms $\varepsilon_{\rm arch} \ll \varepsilon_{\rm est}$ (Proposition~\ref{prop:dominance}).
See Supplementary Fig.~\ref{fig:rmse_all_vars} for all six verification variables.}
\label{fig:rmse}
\end{figure*}

\paragraph{Empirical validation: ACC analysis.}
The Anomaly Correlation Coefficient \citep[ACC;][]{murphy1989skill} provides a complementary view to RMSE by measuring the pattern correlation between forecast and observed anomalies (relative to climatology), weighted by $\cos\phi$ to account for grid convergence.
Unlike RMSE, which penalizes all errors equally, ACC rewards forecasts that correctly capture the spatial pattern of anomalies even if the absolute magnitudes are imperfect.
The traditional threshold ACC$= 0.6$ defines ``useful'' forecast skill in operational meteorology \citep{murphy1989skill}; below this threshold, the forecast pattern correlation with reality is considered too weak to inform decision-making.

Figure~\ref{fig:acc} reveals several features relevant to the pipeline-dominance framework.
For Z500, all models maintain ACC $> 0.6$ through day~7; the best-performing models (FengWu, FuXi, AIFS) sustain useful skill to approximately day~10, while the majority cross below the 0.6 threshold between days~7 and~10.
The useful skill horizon varies by only $\sim$3 days across architectures, a narrow spread given the architectural diversity.
The ACC decay is approximately exponential, consistent with the information-theoretic prediction of Theorem~\ref{thm:info_decay}: as mutual information $I(\bm{u}(0); \bm{u}(\tau))$ decays at rate $h_{\rm KS}$, the forecast's ability to capture anomaly patterns diminishes accordingly.
The decay rate is faster for T2M than for Z500, consistent with the higher effective Lyapunov exponents associated with boundary-layer dynamics compared to mid-tropospheric geopotential.

Notably, the model ranking by ACC differs from the ranking by RMSE at certain lead times.
A model can achieve low RMSE through aggressive smoothing (which reduces point-wise errors) while simultaneously achieving lower ACC (because smoothing washes out the anomaly pattern).
This divergence between RMSE and ACC rankings is itself evidence for the multi-dimensional nature of forecast quality that motivates our holistic assessment (\S\ref{sec:holistic}).

\begin{figure*}[!t]
\centering
\includegraphics[width=\textwidth]{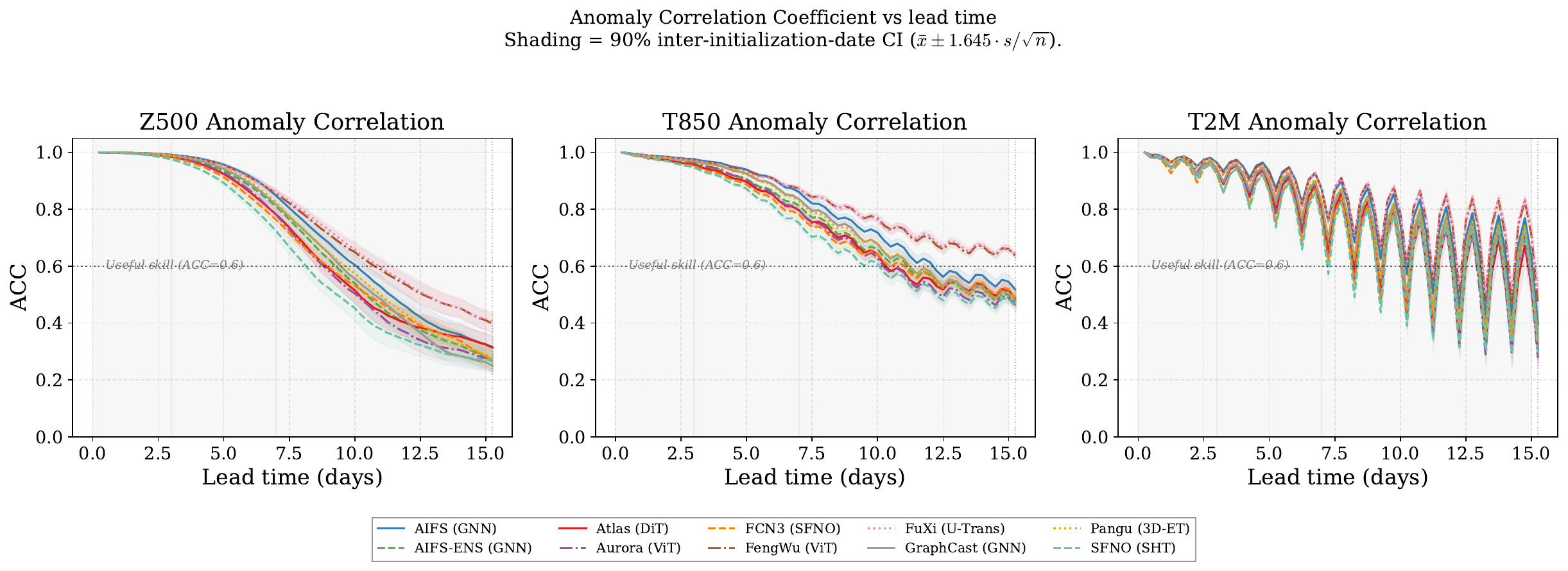}
\caption{Anomaly Correlation Coefficient vs.\ lead time for Z500, T850, and T2M, with inter-initialization-date 90\% CIs.
Most models sustain useful skill (ACC $> 0.6$) for Z500 through $\sim$7 days; the best models extend to $\sim$10 days.}
\label{fig:acc}
\end{figure*}

\paragraph{Empirical validation: Scorecard analysis.}
The model scorecard (Fig.~\ref{fig:scorecard}) translates RMSE into lead-time-resolved rankings, revealing a pattern that is central to the pipeline-dominance argument: \emph{the top-performing models span multiple architecture families}.
AIFS (GNN), FengWu (ViT), and FuXi (U-Transformer), three fundamentally different architectures, occupy the top three positions at most lead times.
If architecture were the primary determinant of skill, we would expect models from the same family to cluster together; instead, the top tier contains one representative from each of three distinct paradigms.
While the ordering within tiers shifts modestly across lead times (e.g., FuXi overtakes FengWu at days~7--13), the key observation is that \emph{architecture family does not predict rank tier}.

This rank instability has a clear physical interpretation.
At short lead times, when the atmosphere is still highly predictable, forecast skill depends primarily on the model's ability to resolve fine-scale features and avoid phase errors in rapidly propagating systems.
Here, the training procedure (timestep, rollout strategy) and loss function details (latitude weighting, pressure-level weighting) are decisive.
At extended lead times, when small-scale predictability has been exhausted and the forecast must capture the evolution of large-scale patterns, the model's representation of slow dynamics, energy conservation properties (Autoregressive Stability Index, ASI; defined in \S\ref{sec:spectral}), and resistance to autoregressive error accumulation become dominant.
These are influenced more by architectural properties (e.g., Pangu's 3D attention across pressure levels enabling better vertical coupling, GraphCast's multi-scale icosahedral mesh providing natural scale separation) and training methodology (loss weighting, inference strategy) than by raw representational capacity.

The scorecard thus provides direct evidence that the \emph{relative importance} of pipeline components shifts with lead time, supporting the multi-dimensional assessment framework developed in \S\ref{sec:holistic}.

\begin{figure*}[!t]
\centering
\includegraphics[width=\textwidth]{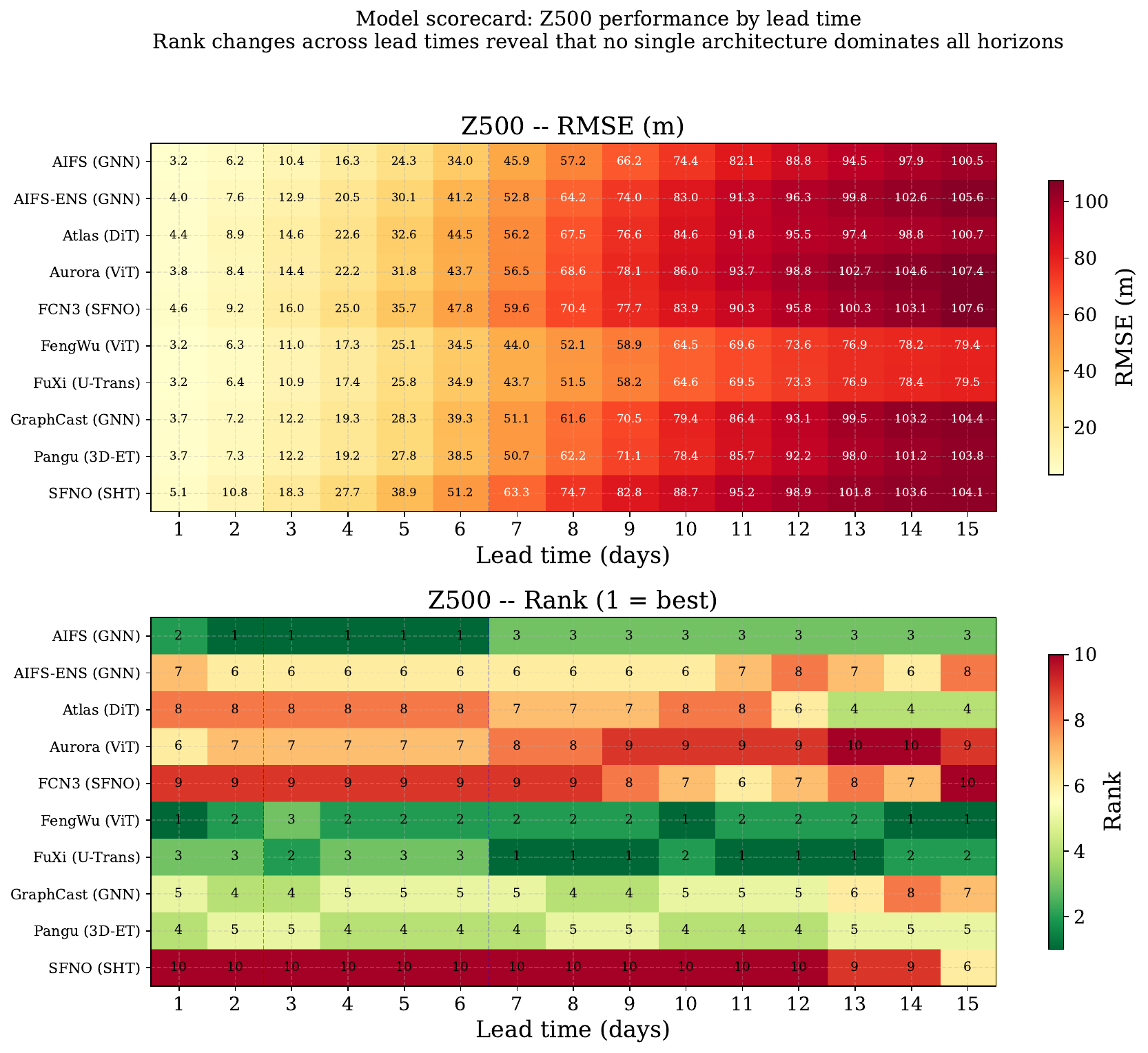}
\caption{Model scorecard for Z500: RMSE (top) and rank (bottom, 1=best) by lead time.
Rank changes confirm that no single architecture dominates all horizons.
See Supplementary Figs.~\ref{fig:scorecard_t2m}--\ref{fig:scorecard_t850} for T2M and T850.}
\label{fig:scorecard}
\end{figure*}

\paragraph{Empirical validation: Regional RMSE decomposition.}
Global RMSE averages mask important latitude-dependent skill variations driven by the distinct dynamical regimes of the tropics, extratropics, and polar regions.
We decompose RMSE into three latitude bands: Tropics ($|\phi| < 20^\circ$), Extratropics ($20^\circ \leq |\phi| < 60^\circ$), and Polar ($|\phi| \geq 60^\circ$), using area-weighted RMSE within each band.

Figure~\ref{fig:regional_rmse} reveals several physically important patterns for Z500.
The tropical result is unsurprising: Z500 RMSE in the tropics remains low ($\sim$10--22\,m) and nearly flat through day~10, reflecting the low geopotential variance where dynamics are governed by the slowly varying Hadley circulation rather than by baroclinic instability (see Supplementary Figs.~\ref{fig:regional_t2m}--\ref{fig:regional_t850} for T2M and T850, where tropical variance is larger and the decomposition is more informative).
The extratropical and polar bands are more revealing:
(i)~\emph{Polar amplification}: Polar RMSE grows fastest with lead time, reaching $\sim$100--140\,m by day~10 compared to $\sim$70--100\,m in the extratropics, consistent with the higher effective Lyapunov exponents at high latitudes driven by baroclinic instability, stratospheric sudden warmings, and polar vortex dynamics.
(ii)~\emph{Comparable inter-model spread}: Both the extratropical and polar bands show similar relative model spread (coefficient of variation $\sim$11\% at day~10), indicating that pipeline differences---loss weighting, timestep, autoregressive rollout---manifest comparably across both dynamically active latitude bands.

\begin{figure*}[!t]
\centering
\includegraphics[width=\textwidth]{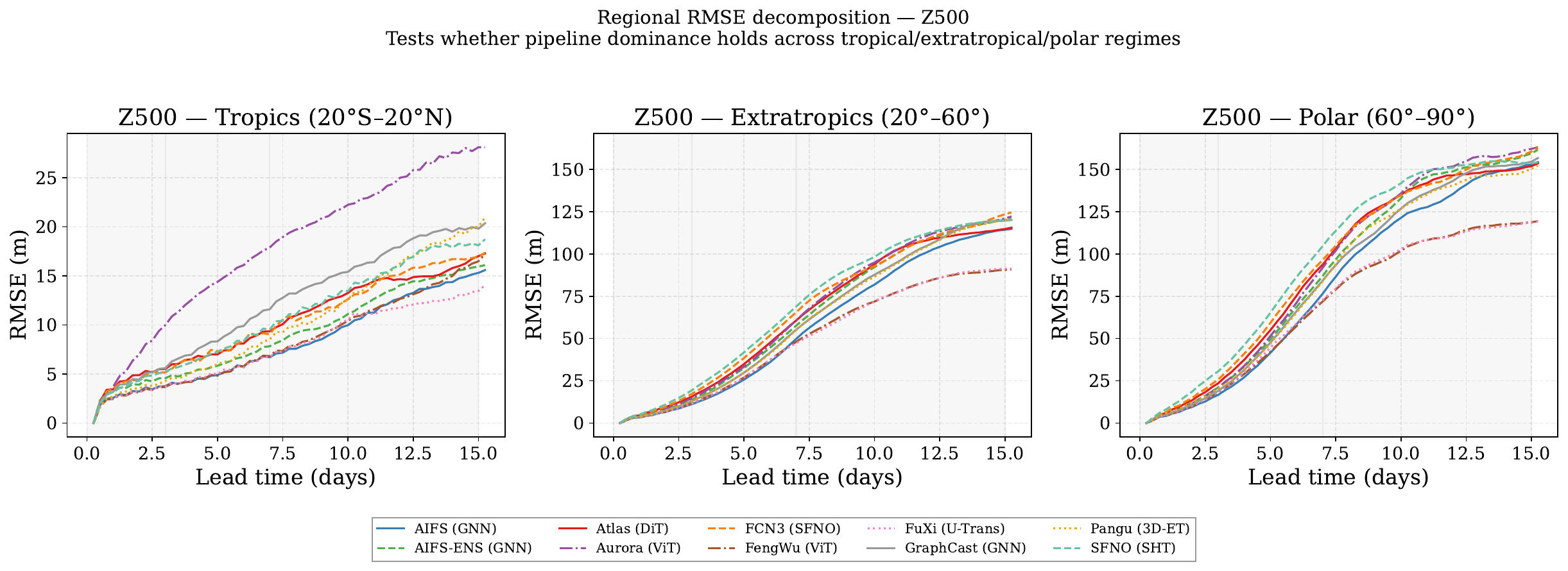}
\caption{Regional RMSE decomposition for Z500 by latitude band: Tropics ($|\phi| < 20^\circ$), Extratropics ($20^\circ \leq |\phi| < 60^\circ$), Polar ($|\phi| \geq 60^\circ$).
Note that y-axis ranges differ across panels to resolve intra-band model spread; cross-panel comparison of absolute RMSE values requires reading the axis labels.
Polar errors amplify fastest (reaching $\sim$100--140\,m by day~10 vs.\ $\sim$70--100\,m in the extratropics); extratropical and polar bands show comparable relative inter-model spread.
See Supplementary Figs.~\ref{fig:regional_t2m}--\ref{fig:regional_t850} for T2M and T850.}
\label{fig:regional_rmse}
\end{figure*}

\section{Dynamical Predictability Theory}
\label{sec:spectral}

\subsection{The Spectral Transfer Operator}

The true atmospheric dynamics in spectral space (linearized about a reference state) are \citep{vallis2017atmospheric}:
\begin{equation}
\ahat_\ell^m(t+\Delta t) = \sum_{\ell',m'} \mathcal{T}_{\ell\ell'}^{mm'}\, \ahat_{\ell'}^{m'}(t) + \text{nonlinear},
\label{eq:spectral_transfer}
\end{equation}
where $\mathcal{T}_{\ell\ell'}^{mm'}$ is the spectral transfer operator encoding wave interactions, advection, and dissipation.
A learned model $\hat{\Phi}_\tau$ implicitly approximates this operator, and its fidelity at each scale $\ell$ can be quantified by comparing forecast and verification spectra.
We introduce two complementary metrics for this purpose:

\begin{definition}[Spectral Fidelity Index (SFI)]
\label{def:SFI}
For a model $\hat{\Phi}$ with forecast spectrum $\hat{E}_f(\ell)$ and ERA5 verification spectrum $E_a(\ell)$:
\begin{equation}
\text{SFI} = 1 - \frac{1}{2}\, \frac{1}{|\mathcal{K}|}\sum_{\ell \in \mathcal{K}} \bigl|\log_{10}\bigl(\hat{E}_f(\ell)/E_a(\ell)\bigr)\bigr|,
\label{eq:SFI}
\end{equation}
where $\mathcal{K} = \{\ell : 1 \leq \ell \leq 200, E_a(\ell) > 0, \hat{E}_f(\ell) > 0\}$.
SFI = 1 indicates perfect spectral fidelity.
The logarithmic ratio $\log_{10}(E_f/E_a)$ symmetrically penalizes both spectral deficit (MSE smoothing) and spectral excess (diffusion model noise).
SFI provides an integrated measure across wavenumbers; the following complementary metric identifies the \emph{cutoff scale} at which a model's spectrum breaks down:
\end{definition}

\begin{definition}[Effective Resolution]
\label{def:eff_res}
The effective resolution $\ell_{\rm eff}$ is the highest wavenumber at which the spectral energy ratio exceeds a threshold \citep{skamarock2004evaluating}:
\begin{equation}
\ell_{\rm eff} = \max\bigl\{\ell : \hat{E}_f(\ell)/E_a(\ell) \geq 0.5\bigr\}.
\label{eq:eff_res}
\end{equation}
We normalize to $[0,1]$ by $\ell_{\rm eff}/300$.
This concept derives from the NWP practice of defining effective resolution as the scale at which the model's energy spectrum falls significantly below the observed spectrum \citep{skamarock2004evaluating}.
Both SFI and $\ell_{\rm eff}$ are computed for each model at each lead time and initialization date; the multi-initialization averages appear in Tables~\ref{tab:hmas_day3}--\ref{tab:hmas_day15} and on the SFI and $\ell_{\rm eff}$ axes of the HMAS radar chart (Fig.~\ref{fig:radar}).
At day~5, the SFI--$\ell_{\rm eff}$ correlation is $\rho = 0.73$: the two metrics are related but not redundant, as SFI integrates over the full spectral range while $\ell_{\rm eff}$ captures the specific wavenumber at which fidelity collapses.
\end{definition}

\paragraph{Caveat: one-sided threshold pathology.}
An important caveat applies at extended lead times: because $\ell_{\rm eff}$ uses a one-sided threshold ($E_f/E_a \geq 0.5$), it does not penalize spectral energy \emph{excess}.
A model experiencing severe energy redistribution or noise amplification can push $E_f(\ell)/E_a(\ell)$ above 0.5 at all wavenumbers, yielding $\ell_{\rm eff} = 1.0$ despite poor spectral fidelity.
This pathology manifests in FengWu at day~15: despite catastrophic energy drift (ASI $\approx 0.01$) and very low SFI ($= 0.078$), its $\ell_{\rm eff}$ saturates at 1.0 because the inflated spectral energy exceeds the half-power threshold everywhere.
SFI, which symmetrically penalizes both deficit and excess via the log-ratio, correctly diagnoses this as poor spectral fidelity.
The $\ell_{\rm eff}$ values in Table~\ref{tab:hmas_day15} should therefore be interpreted with caution for models exhibiting spectral inflation at extended range.

\subsection{Scale-Dependent Error Growth and Lyapunov Analysis}

The atmospheric error at spectral degree $\ell$ grows according to the scale-dependent framework originating with \citet{lorenz1969predictability} (see also \citealt{vallis2017atmospheric}, \S8.6):
\begin{equation}
\varepsilon_\ell(t) \approx \varepsilon_\ell(0)\, e^{\sigma_\ell\, t} + \int_0^t e^{\sigma_\ell(t-t')} S_\ell(t')\,dt',
\label{eq:error_growth}
\end{equation}
with growth rate $\sigma_\ell \propto [\ell^2\, E(\ell)]^{1/2}$.
The first term represents exponential amplification of initial errors at rate $\sigma_\ell$, while the second captures the upscale cascade of energy from smaller scales via the source term $S_\ell(t')$.
The scale-dependence of $\sigma_\ell$ through the energy spectrum $E(\ell)$ has a direct consequence for predictability:

\begin{proposition}[Predictability Horizon by Spectral Regime]
\label{prop:predict_horizon}
In the $\ell^{-3}$ regime: $\sigma_\ell \propto \ell^{-1/2}$, yielding unlimited predictability absent upscale cascade.
In the $\ell^{-5/3}$ regime: $\sigma_\ell \propto \ell^{1/6}$, yielding finite predictability $T^* \sim 14$ days.
\end{proposition}

\begin{proof}
For $E(\ell) = C\ell^{-3}$: $\sigma_\ell \propto [\ell^2 \cdot \ell^{-3}]^{1/2} = \ell^{-1/2}$.
Since $\sigma_\ell \to 0$ as $\ell \to 0$, cascade time diverges, yielding unlimited predictability.
For $E(\ell) = C\ell^{-5/3}$: $\sigma_\ell \propto \ell^{1/6}$.
Cascade time from $\ell_{\max}$ to $\ell_{\rm meso}$ is finite: $\propto \ell_{\rm meso}^{-1/6}\ln(\ell_{\max}/\ell_{\rm meso}) \sim 14$ days.
\end{proof}

We quantify error growth through two diagnostics (Fig.~\ref{fig:error_growth}).
The effective Lyapunov exponent $\lambda_{\rm eff}$ is estimated from the slope of log-RMSE versus lead time during the exponential growth phase (before saturation).
The \emph{error doubling time} $\tau_d$ converts this growth rate into a physically interpretable timescale:
\begin{equation}
\tau_d = \frac{\ln 2}{\lambda_{\rm eff}},
\label{eq:error_doubling}
\end{equation}
normalized to $[0,1]$ by $\tau_d / (2 \times 24\;\text{hours})$, where the 2-day reference corresponds to the theoretical maximum error doubling time for synoptic-scale weather \citep{lorenz1969predictability}.
Longer doubling times indicate slower error growth and more skillful forecasts; the $\tau_d$ values in Tables~\ref{tab:hmas_day3}--\ref{tab:hmas_day15} range from 0.69 (SFNO, fastest growth) to 0.77 (GraphCast and FCN3, slowest growth), corresponding to raw doubling times of $\sim$33--37 hours, consistent with classical NWP estimates.

A complementary diagnostic tests whether the model conserves global kinetic energy during autoregressive rollout.
The \emph{Autoregressive Stability Index} (ASI) is defined through the area-weighted global kinetic energy ratio $E_{\rm KE}(\tau)/E_{\rm KE}(0)$:
\begin{equation}
\text{ASI} = 1 - \frac{|\gamma|\, T}{\ln 2},
\label{eq:ASI}
\end{equation}
where $\gamma$ is the exponential drift rate of area-weighted global KE and $T$ is the forecast window length.
ASI $= 1$ indicates perfect energy conservation; ASI $= 0$ indicates severe energy drift (either dissipation or inflation) over the forecast window.
In our evaluation (Fig.~\ref{fig:error_growth}, right panel), most MSE-trained models maintain ASI~$> 0.84$, while FuXi and FengWu show catastrophic energy collapse (ASI~$\approx 0.02$ and $\approx 0.01$ respectively).
The severity of this collapse suggests a fundamental instability in the autoregressive rollout.
While both FuXi and FengWu use 6-hour inference timesteps (requiring more autoregressive steps to reach extended range), Aurora also uses 6-hour steps without similar degradation, indicating that timestep alone is not the cause.
The instability more likely reflects an interaction between training methodology (single-step vs.\ rollout fine-tuning, loss weighting) and model architecture, with Aurora's LoRA-based rollout fine-tuning \citep{bodnar2025aurora} potentially mitigating the error accumulation that affects FuXi and FengWu.
This instability is particularly consequential because it occurs despite both models achieving competitive RMSE at shorter lead times, highlighting the diagnostic value of ASI as a metric that captures failure modes invisible to traditional skill scores.

\paragraph{Empirical validation: Error growth and energy stability.}
Figure~\ref{fig:error_growth} presents two complementary views of dynamical predictability.
The left panel shows Z500 RMSE on a logarithmic scale versus lead time.
The approximately linear growth in log-RMSE over the first 5--8 days confirms exponential error growth, consistent with the Lyapunov-based prediction of Proposition~\ref{prop:predict_horizon}: errors at each scale $\ell$ grow at rate $\sigma_\ell$, and the globally averaged RMSE inherits the dominant growth rate from the most energetic scales.
The estimated effective Lyapunov exponents cluster tightly across models ($\lambda_{\rm eff} \approx 0.3$--$0.5\;\text{day}^{-1}$), corresponding to error doubling times of $\sim$1.5--2.3 days, consistent with classical estimates from NWP \citep{lorenz1969predictability}.
Beyond day~8, the log-RMSE curves begin to flatten as errors approach climatological variance (the saturation regime), reflecting the transition from the exponential-growth phase to the predictability-limited phase predicted by the information-theoretic bound (Theorem~\ref{thm:info_decay}).

The right panel reveals a physically important diagnostic: global kinetic energy stability, measured as $E_{\rm KE}(\tau)/E_{\rm KE}(0)$.
A perfect forecast would maintain this ratio at unity throughout the forecast window.
Deviations indicate either energy dissipation (ratio $< 1$) or energy inflation (ratio $> 1$), both of which are physically undesirable.
Most MSE-trained models exhibit systematic energy \emph{dissipation}, with $E_{\rm KE}$ declining by 5--15\% over 15 days.
This is a direct physical consequence of MSE-induced spectral smoothing: by suppressing high-wavenumber energy (Theorem~\ref{thm:mse_bias}), the MSE optimizer removes kinetic energy from the forecast, violating energy conservation.
FCN3 is a notable exception, showing energy \emph{inflation}: as a CRPS-trained model, its loss function preserves (and can slightly overestimate) spectral energy at high wavenumbers, leading to net kinetic energy injection rather than the dissipation characteristic of MSE-trained models.
FuXi and FengWu show the most severe energy drift (ASI $\approx 0.02$ and $\approx 0.01$ respectively), suggesting that their autoregressive rollout is particularly unstable at extended lead times.
We note that the FuXi evaluated here is the medium-range deterministic model; FuXi-S2S \citep{chen2024fuxi_s2s}, which achieves 42-day predictions, is a distinct model with different architecture and training methodology specifically designed for subseasonal timescales and may exhibit different energy stability properties.
More broadly, all ten models in our study operate at similar resolutions ($\sim$0.25$^\circ$), which limits their ability to resolve small-scale features and may contribute to energy dissipation through unresolved scale interactions; models designed for longer rollouts typically adopt coarser resolutions and spectral regularization to maintain stability.
This energy drift diagnostic is orthogonal to RMSE: a model can have reasonable RMSE while losing energy systematically, producing forecasts that are ``right on average'' but physically inconsistent.
Indeed, both FengWu and FuXi achieve competitive Z500 RMSE at several lead times (Fig.~\ref{fig:scorecard}) despite catastrophic energy dissipation, a paradox that underscores why single-metric evaluation is insufficient.
FengWu's RMSE success despite ASI~$\approx 0$ arises because RMSE is dominated by the large-scale pattern (which carries most of the error weight), while the energy drift manifests primarily at smaller scales; the two diagnostics probe different aspects of forecast quality, confirming the value of multi-dimensional assessment.

\begin{figure*}[!t]
\centering
\includegraphics[width=\textwidth]{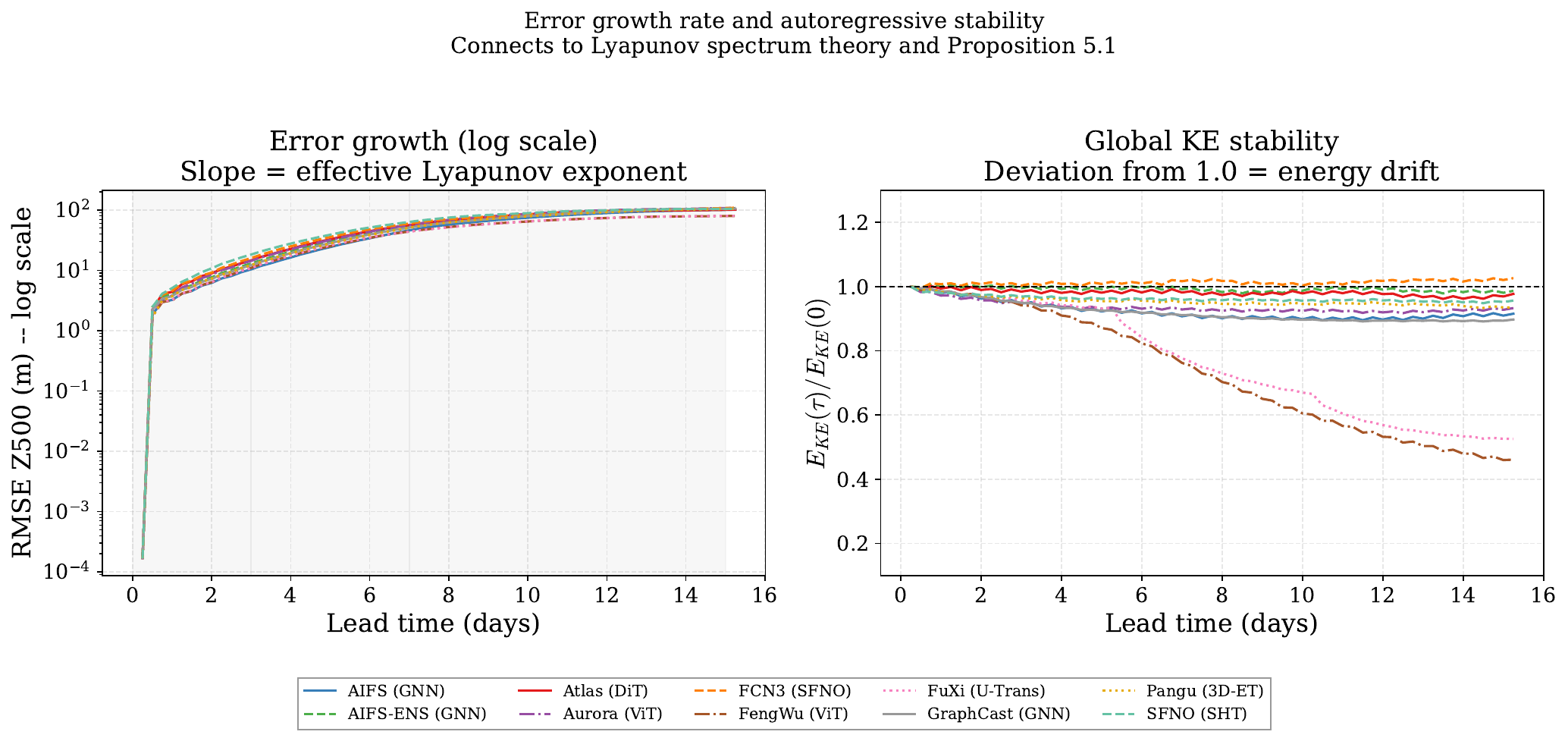}
\caption{\textbf{Left}: Z500 RMSE on log scale vs.\ lead time.
The approximately linear growth in log-RMSE confirms exponential error growth consistent with Proposition~\ref{prop:predict_horizon}.
\textbf{Right}: Global kinetic energy stability $E_{\rm KE}(\tau)/E_{\rm KE}(0)$.
Deviation from unity indicates energy drift.
FCN3 (CRPS-trained) shows energy inflation; most MSE-trained models show dissipation, consistent with MSE-induced smoothing.}
\label{fig:error_growth}
\end{figure*}

\section{Multi-Balance Physical Consistency}
\label{sec:physical_consistency}

A physically meaningful forecast should respect the dynamical balance relationships of the atmosphere.
Rather than relying on a single geostrophic balance diagnostic, we construct a \emph{four-component Physical Consistency Score} (PCS) that tests four independent dynamical constraints:

\subsection{The Four Balance Components}

\paragraph{(1) Geostrophic balance.}
The equilibrium between pressure gradient force and Coriolis force in midlatitudes \citep{holton2013introduction,vallis2017atmospheric} constrains the large-scale wind to be approximately geostrophic: $\bm{v}_g = (1/f)\,\hat{k} \times \nabla\Phi$, where $f = 2\Omega\sin\phi$.
The ageostrophic wind ratio $\text{AGR}_g = \sqrt{\langle|\bm{v} - \bm{v}_g|^2\rangle_{\rm mid} / \langle|\bm{v}|^2\rangle_{\rm mid}}$ quantifies departure from this balance.

\paragraph{(2) Non-divergence.}
For large-scale balanced flow, the horizontal divergence $\nabla \cdot \bm{v}$ should be much smaller than the vorticity $\zeta$ \citep{holton2013introduction}.
The non-divergence ratio $\text{NDR} = \langle|\nabla \cdot \bm{v}|^2\rangle_{\rm mid} / \langle|\zeta|^2\rangle_{\rm mid}$ quantifies this constraint.
A well-balanced midlatitude flow has NDR $\ll 1$ (divergence is order Rossby number $\text{Ro} \sim 0.1$ smaller than vorticity).

\paragraph{(3) Thermal wind balance.}
The vertical shear of the geostrophic wind is proportional to the horizontal temperature gradient \citep{holton2013introduction,vallis2017atmospheric}: $f\,\partial \bm{v}_g / \partial \ln p = -R\,\hat{k} \times \nabla T$, where $R$ is the dry gas constant.
We compute the thermal wind from the forecast temperature field and compare with the actual wind shear between 500 and 850\,hPa.

\paragraph{(4) Hydrostatic consistency.}
The hydrostatic relation connects geopotential thickness between pressure levels to the virtual temperature \citep{holton2013introduction}: $\Delta\Phi = -R\,\bar{T}_v \,\ln(p_{\rm top}/p_{\rm bot})$.
We compute the expected geopotential thickness from the forecast temperature and compare with the actual forecast geopotential difference.

\subsection{Composite Physical Consistency Score}

The composite PCS is the equally weighted mean of the four sub-scores:
\begin{equation}
\text{PCS}_{\rm composite} = \frac{1}{4}\sum_{i=1}^{4} \text{PCS}_i,
\label{eq:PCS_composite}
\end{equation}
where each $\text{PCS}_i = \max(0, 1 - R_i/R_{i,\rm max})$ normalizes the respective ratio to $[0,1]$, with $R_{i,\rm max}$ chosen so that PCS$_i = 0$ corresponds to severe imbalance.
Equal weighting is used because each balance constraint is physically fundamental, and weighting would introduce arbitrary choices.
We verify empirically that the equal-weighting choice is not consequential: the HMAS weight sensitivity analysis (Supplementary Fig.~\ref{fig:hmas_sensitivity}) shows Kendall's $W = 0.97$ concordance across five weighting schemes, confirming that model rankings are robust to weight perturbations.
We report the composite score alongside the sub-score breakdown (Fig.~\ref{fig:physical_consistency}, right panel) precisely because the composite may mask compensating effects: a model could improve one sub-score while degrading another, leaving the composite unchanged.
The composite is useful for inclusion in the multi-metric HMAS framework (\S\ref{sec:holistic}), where a single PCS dimension is needed, while the sub-score decomposition is essential for diagnosing the physical mechanisms behind each model's balance properties.

\paragraph{Empirical results.}
Figure~\ref{fig:physical_consistency} presents both the composite PCS over time (left panel) and the sub-score breakdown (right panel).
A critical baseline observation is that ERA5 itself achieves composite PCS $\approx 0.55$; the real atmosphere is \emph{not} perfectly balanced at any instant.
Ageostrophic motions (jet streaks, frontal circulations, gravity waves) and diabatic processes continuously drive departures from geostrophic and thermal wind balance.
This sets an upper bound on what any model should achieve: a PCS significantly exceeding $\sim$0.55 would indicate overly constrained dynamics, not superior physical realism.

The sub-score breakdown reveals the structure of balance fidelity across models.
The hydrostatic sub-score is consistently highest across all models ($\sim$0.89--0.90), because the hydrostatic relation is a strong, quasi-exact constraint linking temperature and geopotential thickness. Even models that violate other balance relations tend to maintain hydrostatic consistency, likely because the hydrostatic relationship is implicitly encoded in the vertical structure of ERA5 training data.
The non-divergence sub-score shows the greatest inter-model spread and provides the most discriminating diagnostic.
AIFS-ENS scores particularly poorly ($\sim$0.11) on non-divergence, suggesting that the Gaussian noise injected into the transformer processor to generate ensemble spread introduces spatially incoherent divergent motions. While the CRPS loss ensures calibrated probabilistic skill, it does not explicitly constrain the noise-driven perturbations to be non-divergent, producing unphysical local convergence and divergence patterns.

The most theoretically informative PCS finding concerns Atlas (DiT, score matching), which degrades to composite PCS $\approx 0.40$ by day~15, well below the ERA5 baseline of $\sim$0.55, driven by deterioration in both geostrophic balance and non-divergence sub-scores (Supplementary Fig.~\ref{fig:pcs_day15_subscore}).
While FuXi and FengWu reach comparably low PCS values at day~15 (discussed below), their degradation stems from energy drift in the autoregressive rollout.
Atlas's degradation is distinct and more informative because it occurs \emph{despite} excellent spectral energy preservation (high SFI).
This reveals a tension present across all AI weather models but \emph{heightened} in diffusion-based systems: while the score-matching loss preserves spectral energy (high SFI, Figs.~\ref{fig:spectral_comparison}--\ref{fig:spectral_ratio}), it does not enforce \emph{inter-variable dynamical relationships}.
Each sample from the diffusion model contains the correct amount of spectral energy in each field individually, but the spatial correlations between wind and geopotential (geostrophic balance), between wind components (non-divergence), and between temperature and wind shear (thermal wind) progressively degrade.
This is because the diffusion model samples from the marginal energy distribution of each field, but the joint distribution across fields becomes increasingly poorly approximated at extended lead times.

FuXi and FengWu exhibit the most severe PCS degradation among all models: FuXi drops from PCS~$= 0.61$ at day~1 to PCS~$= 0.34$ at day~15 (the steepest decline of any model), while FengWu degrades from PCS~$= 0.67$ at day~5 to PCS~$= 0.46$ at day~15.
This PCS collapse is physically coupled to the catastrophic energy drift documented in Fig.~\ref{fig:error_growth}: as kinetic energy dissipates over repeated autoregressive steps (ASI~$\approx 0$ for both models), the wind fields progressively lose the amplitude needed to maintain dynamical balance, leading to simultaneous deterioration of geostrophic balance, non-divergence, and thermal wind consistency.
The coupling between energy instability and balance degradation underscores the diagnostic complementarity of ASI and PCS: ASI captures the \emph{energetic} consequence of autoregressive drift, while PCS captures its \emph{dynamical} consequence.

In contrast, Pangu maintains high composite PCS throughout the forecast, consistent with its pressure-weighted training implicitly enforcing inter-level balance, though at the cost of spectral fidelity, completing the trade-off cycle identified in the HMAS cross-correlation analysis (Supplementary Fig.~\ref{fig:hmas_correlation}).

\begin{figure*}[!tb]
\centering
\includegraphics[width=\textwidth]{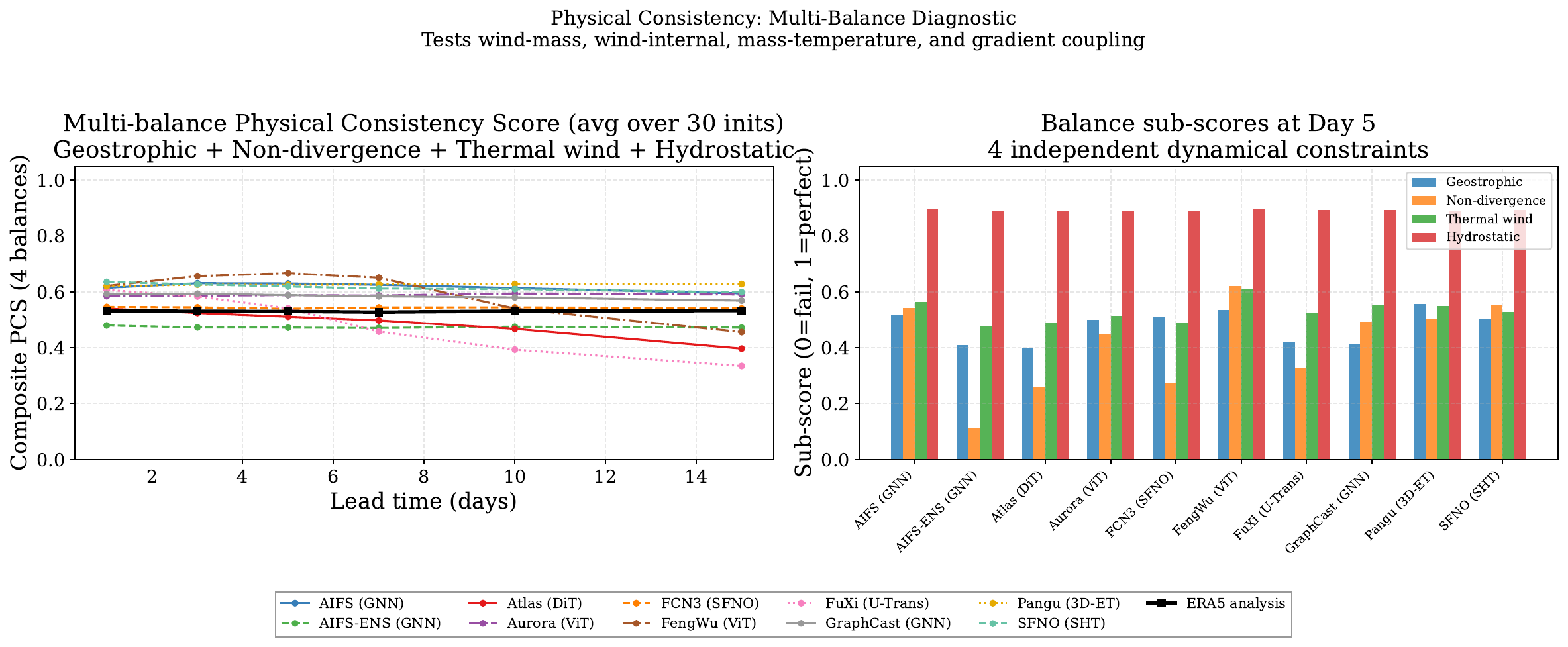}
\caption{Four-component Physical Consistency Score.
\textbf{Left}: Composite PCS (mean of four sub-scores) vs.\ lead time.
ERA5 analysis (black) maintains PCS $\approx 0.55$.
\textbf{Right}: Sub-score breakdown at day~5 for all ten models, showing geostrophic balance (geo.), non-divergence (n-div.), thermal wind (t.w.), and hydrostatic (hyd.) components.
Hydrostatic consistency is uniformly high; non-divergence shows the greatest inter-model spread.
See Supplementary Fig.~\ref{fig:pcs_day15_subscore} for the day~15 sub-score breakdown.}
\label{fig:physical_consistency}
\end{figure*}

\begin{remark}[The Physical Consistency--Accuracy Trade-off]
\label{rem:pcs_tradeoff}
The PCS results reveal a fundamental tension: maintaining physical balance (high PCS) does not imply low RMSE, and competitive RMSE does not guarantee physical consistency.
Pangu's hierarchical temporal aggregation inference strategy and pressure-weighted MSE loss produce well-balanced but less spectrally detailed forecasts (PCS~$\approx 0.63$ stable through day~15), while FuXi, despite achieving competitive Z500 RMSE, suffers the steepest PCS decline of any model (PCS~$= 0.61 \to 0.34$ from day~1 to day~15), driven by the same energy instability that produces its near-zero ASI.
This trade-off is not captured by any single metric and motivates our holistic assessment framework (\S\ref{sec:holistic}).
\end{remark}

\section{Information-Theoretic Predictability Bounds}
\label{sec:info}

\subsection{Information Content of Weather Forecasts}

The mutual information \citep{cover2006elements} between the initial state $\bm{u}(0)$ and the future state $\bm{u}(\tau)$ quantifies the theoretical information available for prediction:
\begin{equation}
I\bigl(\bm{u}(0);\, \bm{u}(\tau)\bigr) = H\bigl(\bm{u}(\tau)\bigr) - H\bigl(\bm{u}(\tau) \mid \bm{u}(0)\bigr),
\label{eq:mutual_info}
\end{equation}
where $H$ denotes the Shannon entropy.
The first term is the marginal entropy of the future state (the total uncertainty about $\bm{u}(\tau)$ without any knowledge of the initial conditions), and the second is the conditional entropy (the residual uncertainty \emph{given} the initial state).
As $\tau$ increases, the conditional entropy grows (the initial conditions become less informative about the future) and $I$ decreases toward zero.
The rate of this decrease is governed by the atmosphere's chaotic dynamics:

\begin{theorem}[Pesin-type Information Decay]
\label{thm:info_decay}
For an ergodic chaotic system with Lyapunov exponents $\{\lambda_i\}$, the mutual information decays as \citep{pesin1977characteristic}
\begin{equation}
I\bigl(\bm{u}(0);\, \bm{u}(\tau)\bigr) \leq I_0 - h_{\rm KS}\,\tau + \mathcal{O}(\tau^2),
\label{eq:info_decay}
\end{equation}
where $I_0 = H(\bm{u}(0))$ and $h_{\rm KS} = \sum_{\lambda_i > 0}\lambda_i$.
The forecast becomes uninformative when $I \to 0$, yielding $\tau^* \approx I_0 / h_{\rm KS}$.
\end{theorem}

\begin{proof}[Proof sketch]
Consider a small ball $B_0$ of initial conditions.
Under the flow, this ball is stretched by the Lyapunov exponents into an ellipsoid with volume growing at rate $\sum_i \lambda_i$.
By Pesin's theorem \citep{pesin1977characteristic}, $h_{\rm KS} = \sum_{\lambda_i > 0} \lambda_i$ for smooth ergodic systems.
Each positive exponent adds $\lambda_i \tau \log_2 e$ bits of positional uncertainty per unit time.
The conditional entropy grows as $H(\bm{u}(\tau) \mid \bm{u}(0)) \approx h_{\rm KS}\, \tau\, \log_2 e$, giving $I \leq I_0 - h_{\rm KS}\tau$.
\end{proof}

\paragraph{Connection to empirical predictability limits.}
Theorem~\ref{thm:info_decay} provides the information-theoretic underpinning for the empirical observations in Figs.~\ref{fig:rmse}--\ref{fig:acc}.
The linear decay of mutual information $I(\bm{u}(0); \bm{u}(\tau))$ at rate $h_{\rm KS}$ implies that the information available for prediction decreases at a rate determined by the atmosphere's positive Lyapunov exponents. This is a property of the \emph{atmosphere}, not of any model.
When $I$ approaches zero, no model of any kind (AI, physics-based NWP, hybrid, or subseasonal-to-seasonal) can produce forecasts more skillful than climatology, because this ceiling is a property of the atmospheric dynamics, not of any particular prediction methodology.
The empirical ACC decay (Fig.~\ref{fig:acc}) provides a proxy for this information loss: the approximately exponential ACC decline toward $\sim$0.6 at 7--10 days, and the tight clustering of all models along similar decay curves, is consistent with all models extracting a similar fraction of the available mutual information from the initial conditions.
The fact that $h_{\rm KS}$ is a sum over \emph{all positive} Lyapunov exponents, rather than just the largest, explains why the predictability limit is finite even for large-scale variables: although the individual growth rate $\sigma_\ell$ for low-$\ell$ modes may be small, the cumulative information loss from all unstable modes eventually exhausts the available information.

\subsection{The Vonich--Hakim Predictability Extension}

\citet{vonich2024,vonich2025} showed that by optimizing initial conditions via backpropagation through GraphCast, skillful prediction extends to 33+ days.

\begin{theorem}[Constrained Predictability Extension]
\label{thm:constrained_predict}
Consider a subspace $\mathcal{V} \subset \mathcal{X}$ spanned by the first $m$ modes ordered by decreasing decorrelation time.
The effective KS entropy on $\mathcal{V}$ satisfies:
\begin{equation}
h_{\rm KS}^{(\mathcal{V})} = \sum_{\mu_i > 0} \mu_i \leq h_{\rm KS},
\end{equation}
and the extended predictability is $\tau_\mathcal{V}^* = I_0^{(\mathcal{V})} / h_{\rm KS}^{(\mathcal{V})} \geq \tau^*$.
\end{theorem}

\begin{proof}[Proof sketch]
Restricting the tangent dynamics to $\mathcal{V}$ via projection $P_\mathcal{V}$ gives a reduced tangent linear equation.
By the Rayleigh--Ritz variational principle \citep{parlett1998symmetric}, eigenvalues of the projected operator are bounded by those of the full operator: $\mu_i \leq \lambda_i$.
IC optimization finds the $\mathcal{V}$ minimizing forecast error, effectively projecting out the fastest-growing perturbation directions.
\end{proof}

\paragraph{Implications and empirical context.}
This theorem formalizes a deep insight: the 14-day deterministic predictability limit is not an absolute ceiling but rather applies to \emph{full-state} prediction.
By restricting attention to a subspace $\mathcal{V}$ that excludes the directions of fastest error growth, the effective KS entropy is reduced, and the predictability horizon extends proportionally.
\citet{vonich2024,vonich2025} demonstrated this empirically by using backpropagation through GraphCast to find initial condition perturbations that minimize the 33-day forecast error, effectively discovering the subspace $\mathcal{V}$ of maximally predictable modes.
The resulting extended skill is consistent with our Theorem~\ref{thm:constrained_predict}: the optimized initial conditions project the forecast onto slowly decorrelating modes (large-scale planetary wave patterns, stratospheric state), reducing $h_{\rm KS}^{(\mathcal{V})}$ relative to $h_{\rm KS}$.
We note that this empirical demonstration has so far been conducted only with GraphCast; whether other architectures yield similar extended predictability under IC optimization remains an open question, though our theorem predicts they should, since the subspace $\mathcal{V}$ is a property of the atmospheric dynamics rather than the model.
This result also connects to our error consensus analysis (\S\ref{sec:consensus}): the high ECR values at extended range suggest that the shared component of forecast error is dominated by the fast-decorrelating modes that Vonich--Hakim optimization would project away, while the model-specific residuals correspond to the predictable subspace $\mathcal{V}$ where architecture differences still matter.

\section{Error Consensus and Predictability Limits}
\label{sec:consensus}

A central prediction of our framework is that when $\varepsilon_{\rm arch} \ll \varepsilon_{\rm est}$, the dominant forecast errors should be shared across architectures.
Intuitively, if all models are limited by the same factors (loss function bias, data coverage, atmospheric predictability) rather than by their architecture, then their errors should look similar to each other.
This section introduces two metrics to test this prediction: the Error Consensus Ratio (ECR), which measures what fraction of total error variance is shared across all models, and the Model Error Divergence (MED), which measures how similar the error patterns are at each spatial scale.

\subsection{Error Consensus Ratio}

\begin{definition}[Error Consensus Ratio (ECR)]
\label{def:ECR}
Given $M$ models with error fields $e_m = \hat{u}_m(\tau) - u_{\rm ERA5}(\tau)$, let $\bar{e} = M^{-1}\sum_m e_m$ and $\varepsilon_m = e_m - \bar{e}$.
The ECR is:
\begin{equation}
\text{ECR}(\tau) = \frac{\langle\|\bar{e}\|^2\rangle}{\langle\|\bar{e}\|^2\rangle + M^{-1}\sum_m \langle\|\varepsilon_m\|^2\rangle}.
\label{eq:ECR}
\end{equation}
ECR $\to 1$ indicates predictability-limited errors; ECR $\to 0$ indicates architecture-limited errors.
\end{definition}

The ECR is a spatial-field generalization of the intraclass correlation coefficient (ICC) \citep{shrout1979intraclass} applied to multi-model forecast error fields, measuring the fraction of total error variance attributable to the shared (predictability-limited) component.
In practice, we compute ECR at each lead time as follows: at each grid point, the ten-model mean error field $\bar{e}$ is computed, and each model's residual $\varepsilon_m = e_m - \bar{e}$ is obtained.
The area-weighted global means of $\|\bar{e}\|^2$ (shared error variance) and $M^{-1}\sum_m \|\varepsilon_m\|^2$ (model-specific error variance) then yield the ECR value reported in Fig.~\ref{fig:error_consensus}.

\begin{proposition}[Architecture-Independence of Dominant Errors]
\label{prop:error_consensus}
If $\varepsilon_{\rm arch} \ll \varepsilon_{\rm est}$ for all models, then ECR $\to 1$ as $M \to \infty$, because the shared error reflects the atmosphere's inherent unpredictability while model-specific residuals average toward zero.
\end{proposition}

\paragraph{Empirical validation.}
Figure~\ref{fig:error_consensus} tests this prediction quantitatively.
The left panel shows ECR versus lead time for Z500.
At day~1, ECR $\approx 0.50$, meaning that approximately half of the total forecast error variance is already shared across all ten models, even though the forecast is still at short range.
This is a notable result given the architectural diversity: it means that even at one day, a substantial fraction of the error is attributable to the atmosphere's inherent unpredictability rather than to limitations of any particular architecture.
As lead time increases, ECR shows a modest increase, reaching $\approx 0.60$ by day~5 and peaking at $\approx 0.62$ around day~8.
While the absolute change is modest, the key finding is not the temporal trend but the consistently high \emph{level}: the majority of error variance ($>$50\%) is shared across architectures at all lead times, supporting the pipeline-dominance thesis.
Beyond day~8, ECR gently declines, settling at $\approx 0.59$ by day~15.
This late-stage decline reflects the growing role of model-specific accumulated errors: autoregressive drift, numerical diffusion, and architecture-dependent energy dissipation or inflation (visible in the ASI diagnostic, Fig.~\ref{fig:error_growth}) increasingly differentiate the models at extended range.
However, even at day~15, the majority of error variance ($\sim$59\%) remains shared---the atmosphere's intrinsic forecast limit continues to dominate over architectural differences.

The right panel corroborates this through mean pairwise spatial error correlation.
At day~1, the mean correlation is $r \approx 0.51$: different models make their largest errors in similar geographical locations, pointing to common predictability barriers (e.g., jet exit regions, frontal zones, convective initiation areas) rather than architecture-specific failure modes.
The pairwise correlation increases with lead time in parallel with ECR, reaching $r \approx 0.60$ at day~7, before declining gently to $r \approx 0.57$ by day~15.
This temporal pattern (shared errors increasing in relative importance at medium range before model-specific divergence partially offsets them at extended range) is consistent with the scale-dependent error growth predicted by Proposition~\ref{prop:predict_horizon}: at medium range, the dominant error comes from large-scale predictability barriers shared by all models, while at extended range the accumulated model-specific drift contributes an additional, architecture-dependent component.
Importantly, the ECR values for T2M (Supplementary Fig.~\ref{fig:ecr_t2m}) reach $\approx 0.63$ at 6-hour lead time and settle to $\approx 0.56$ at day~1 and remain above $0.53$ through day~15, showing a broadly similar pattern.
This variable dependence is consistent with the pipeline-dominance thesis: for both smooth and rough fields, the shared (predictability-limited) errors constitute the majority of the total error budget across the forecast range.

\begin{figure*}[!t]
\centering
\includegraphics[width=\textwidth]{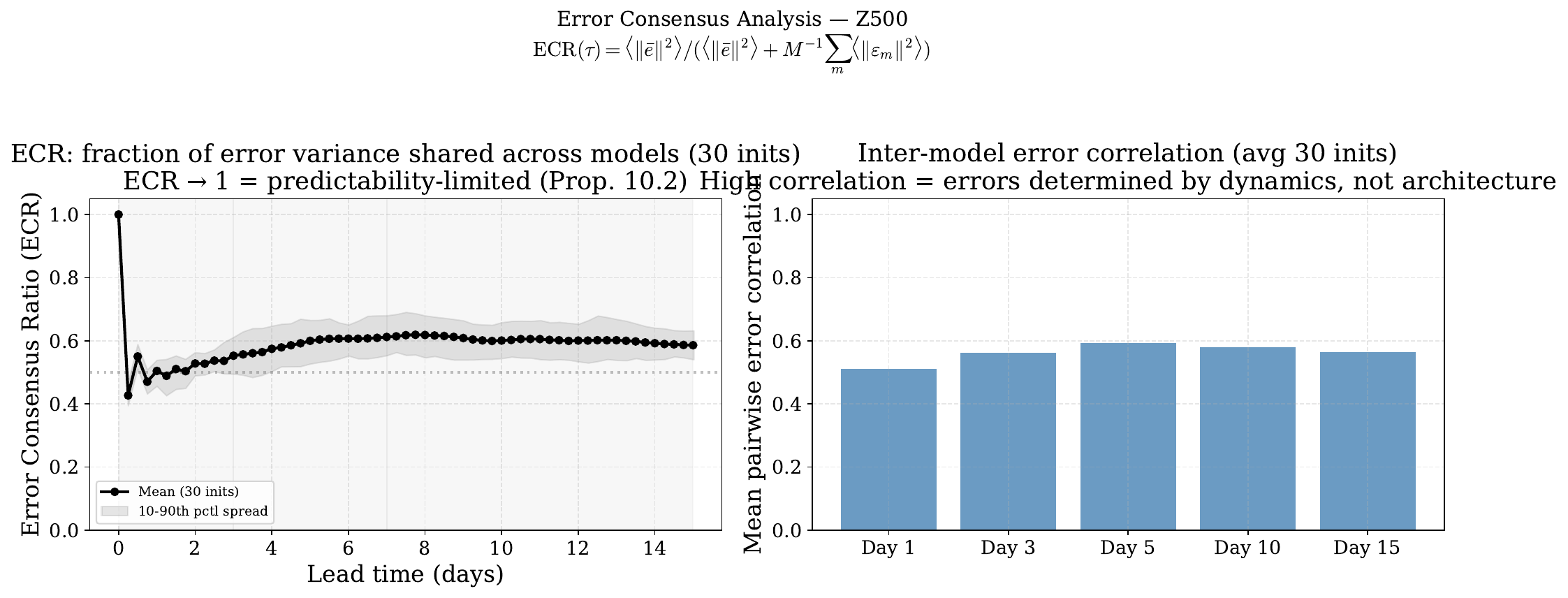}
\caption{Error Consensus Analysis for Z500.
\textbf{Left}: ECR vs.\ lead time; ECR rises from $\sim$0.50 at day~1 to $\sim$0.62 at day~8, then gently declines to $\sim$0.59 at day~15.
\textbf{Right}: Mean pairwise error correlation.
See Supplementary Figs.~\ref{fig:ecr_t2m}--\ref{fig:ecr_t850} for T2M and T850.}
\label{fig:error_consensus}
\end{figure*}

\subsection{Scale-Dependent Error Convergence}

The \emph{Model Error Divergence} (MED) at each wavenumber:
\begin{equation}
\text{MED}(\ell, \tau) = \Bigl\langle \frac{|E_{\rm err,m_1}(\ell) - E_{\rm err,m_2}(\ell)|}{E_{\rm err,m_1}(\ell) + E_{\rm err,m_2}(\ell)} \Bigr\rangle_{\rm pairs}.
\label{eq:MED}
\end{equation}
This normalized absolute difference is structurally equivalent to the Bray--Curtis dissimilarity \citep{bray1957ordination} applied in spectral space; MED $\in [0,1]$ with MED $= 0$ indicating identical error spectra across all model pairs.

Figure~\ref{fig:error_convergence} extends the consensus analysis to scale-dependent resolution, with MED computed separately for within-family model pairs (GNN: AIFS, AIFS-ENS, GraphCast; Spectral: FCN3, SFNO; ViT: Aurora, FengWu) and cross-family pairs.
At large scales ($\ell < 20$), all curves converge to low MED ($\sim$0.10--0.20), confirming that all models handle planetary waves (Rossby waves, the jet stream, hemispheric temperature gradients) in essentially the same way regardless of architecture family.
In terms of the error decomposition (Eq.~\ref{eq:refined_decomp}), large-scale errors are dominated by $\varepsilon_{\rm loss}$ and $\varepsilon_{\rm data}$, which are shared across models with similar loss functions and training data.

At high wavenumbers ($\ell > 50$), the family-specific curves separate clearly.
GNN pairs (blue) maintain the lowest within-family MED ($\sim$0.10--0.15 at $\ell \approx 100$), indicating that GNN-based models (AIFS, AIFS-ENS, GraphCast) produce the most similar error structures among same-family models, consistent with their shared message-passing architecture imposing similar inductive biases on fine-scale representations.
ViT pairs (orange) show intermediate within-family MED, while spectral pairs (green) exhibit the highest variability, likely because FCN3 and SFNO differ substantially in their loss functions (CRPS vs.\ MSE) despite sharing a spectral architecture.
The cross-family MED (red dashed) consistently exceeds all within-family curves at high wavenumbers, and the purple shaded region between the mean within-family and cross-family MED visually demarcates the ``architecture fingerprint,'' the portion of error structure attributable to architecture family rather than shared predictability limits.
This separation provides direct evidence that \emph{architecture does matter at fine scales} (here, wavenumbers $\ell > 50$, corresponding to spatial wavelengths shorter than $\sim$800\,km), enriching the pipeline-dominance narrative: the overall forecast error budget is dominated by shared (loss/data/predictability) contributions at the large scales that carry most of the energy ($\ell < 50$, wavelengths $>$800\,km), but architecture leaves a measurable fingerprint on the structure of errors at smaller scales.
This finding is consistent with the regional RMSE results (Fig.~\ref{fig:regional_rmse}), where both the extratropical and polar bands, driven by baroclinic dynamics at synoptic and mesoscales, show comparable inter-model spread that exceeds the tropics.

\begin{figure*}[!t]
\centering
\includegraphics[width=\textwidth]{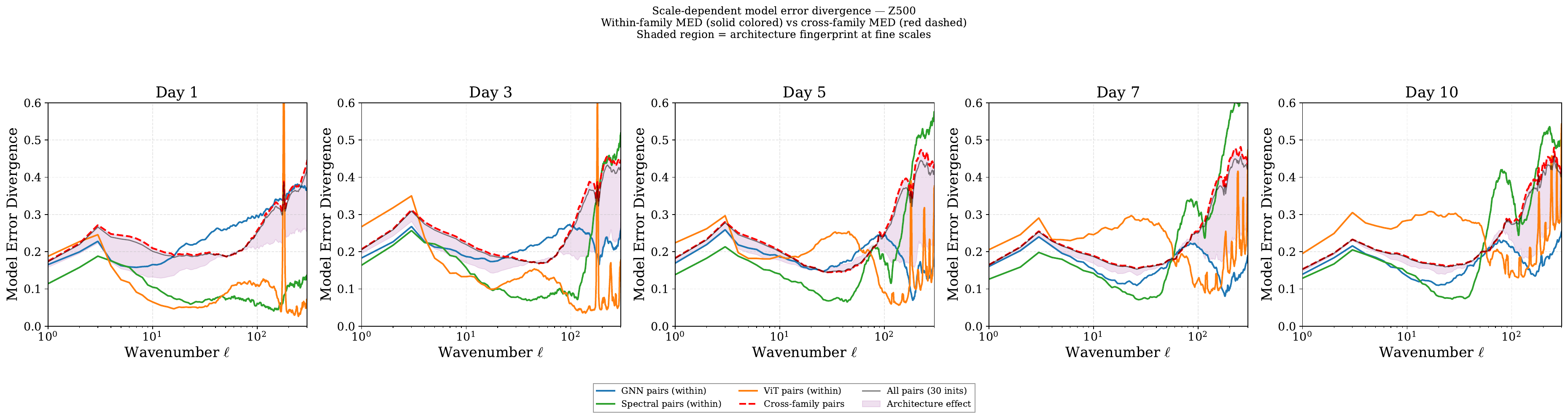}
\caption{Scale-dependent Model Error Divergence (MED) for Z500, decomposed by architecture family.
Solid colored lines show within-family MED for GNN pairs (blue), spectral pairs (green), and ViT pairs (orange); red dashed line shows cross-family MED.
The purple shaded region between mean within-family and cross-family MED highlights the ``architecture fingerprint'' at fine scales ($\ell > 50$), where models from the same family produce more similar error structures than models from different families.
At large scales ($\ell < 20$), all curves converge, confirming universal planetary-wave handling.}
\label{fig:error_convergence}
\end{figure*}

\section{Out-of-Distribution Extrapolation Bounds}
\label{sec:ood}

\subsection{Tail Attenuation Theory}

\begin{proposition}[Extreme Event Attenuation]
\label{thm:ood_bias}
Let $f_\theta$ be a model trained on data $\{u_i\}$ with $\max_i u_i = u_{\rm max}$, and let $u^* = u_{\rm max} + \delta$ with $\delta > 0$.
Under regularity conditions on the learned mapping:
\begin{equation}
\boxed{\E[f_\theta(u^*)] - f_{\rm true}(u^*) \approx -\alpha\, \delta + \mathcal{O}(\delta^2),}
\label{eq:ood_bias}
\end{equation}
where $\alpha > 0$ depends on the loss function and training distribution but \emph{not} on the architecture (given sufficient capacity).
\end{proposition}

\begin{proof}[Proof sketch]
\emph{Step 1.}
By Taylor expansion about $u_{\rm max}$: $f_\theta(u^*) = f_\theta(u_{\rm max}) + f_\theta'(u_{\rm max})\,\delta + \mathcal{O}(\delta^2)$.
\emph{Step 2.}
For MSE-trained models, $f_\theta$ approximates the conditional mean near the boundary, which is ``pulled'' toward the center by regression toward the mean, a fundamental property of conditional expectations \citep{lehmann2006theory}.
This manifests as $f_\theta'(u_{\rm max}) < f_{\rm true}'(u_{\rm max})$, giving $\alpha = f_{\rm true}'(u_{\rm max}) - f_\theta'(u_{\rm max}) > 0$.
\end{proof}

This result provides a theoretical explanation for the empirical observation first documented by \citet{zhang2025extremes}, who showed that AI weather models systematically underestimate record-breaking extremes with errors growing approximately linearly in exceedance.
Our contribution is the formal connection between this observation and the loss function via the gradient attenuation factor $\alpha$, which relates the bias to regression toward the mean operating at the boundary of the training distribution.
We note two important caveats.
First, our empirical validation uses ERA5 reanalysis as initial conditions, which benefits from hindsight data assimilation and provides the most accurate available representation of the atmospheric state at initialization time. In an operational setting, initial conditions from real-time analyses contain larger errors, and the interaction between IC uncertainty and OOD extrapolation could modify the effective $\alpha$ values. Similarly, ensemble-based initialization, which samples a distribution of plausible initial states rather than a single best estimate, would yield a distribution of $\alpha$ values rather than a point estimate.
Second, the attenuation is lead-time-dependent: $\alpha$ generally increases with lead time as the forecast becomes more uncertain and the conditional mean regresses further toward climatology (see $\alpha(\tau)$ evolution in Supplementary Figs.~\ref{fig:alpha_evolution_pnw}--\ref{fig:alpha_evolution_elliott}).

\begin{remark}[Heavy-Tailed Diffusion]
\label{rem:heavy_tails}
Standard score-matching with a Gaussian forward process generates samples whose tails decay as $\exp(-\norm{\bm{u}}^2/2\sigma^2)$.
For atmospheric fields with heavier-tailed extremes, replacing the Gaussian prior with a Student-$t$ prior \citep{kotz2004multivariate} ($\nu$ degrees of freedom) yields tails decaying as $\norm{\bm{u}}^{-(\nu+d)/2}$, improving extreme-event representation.
This approach has been recently validated empirically by \citet{pandey2024heavy}, who introduced t-EDM and t-Flow, Student-$t$ extensions of the Elucidated Diffusion Model (EDM) framework \citep{karras2022edm}, and demonstrated significantly improved coverage of extreme values on the NOAA HRRR high-resolution weather dataset.
Their results showed that standard Gaussian-prior diffusion models systematically underestimate the probability of rare meteorological events (high vertically integrated liquid content, strong updraft velocities), while t-EDM captures a broader range of extreme values with only a single additional hyperparameter ($\nu$).
Crucially, this improvement is a property of the \emph{noise schedule and prior distribution}, not the score network architecture.
While this specific result applies only to diffusion-based models (Atlas in our evaluation), it illustrates the broader principle that within any model class, the training pipeline choices (here, the diffusion prior) can have a larger effect on forecast properties than the network architecture itself.
For non-diffusion models, analogous pipeline-level interventions (e.g., loss function redesign as in \citealt{subich2025msh}, or training data augmentation) play a comparable role.
\end{remark}

\subsection{Empirical Validation: Tail Fidelity and \texorpdfstring{$\alpha$}{alpha} Evolution}

We test Proposition~\ref{thm:ood_bias} using four events spanning both warm and cold extremes: the June 2021 PNW heatwave (initialized 2021-06-22), the July 2023 European heatwave (initialized 2023-07-10), the February 2021 Texas freeze (initialized 2021-02-10), and the December 2022 Winter Storm Elliott (initialized 2022-12-20).
These events are ideal test cases because all produced temperatures exceeding the training climatology by several standard deviations, precisely the out-of-distribution regime where Eq.~\eqref{eq:ood_bias} predicts linear bias growth.

Figure~\ref{fig:tail_fidelity} confirms the linear bias--exceedance relationship predicted by Proposition~\ref{thm:ood_bias}.
For each model, we plot the mean forecast bias (in Kelvin) against the standardized exceedance $\delta$ (in units of the local climatological standard deviation $\sigma_{\rm clim}$) at grid points where the ERA5 verification exceeds $\mu_{\rm clim} + 2\sigma_{\rm clim}$.
The key observation is that \emph{all} models show increasingly negative bias, i.e., systematic underestimation, at higher exceedances.
The relationship is approximately linear for moderate $\delta$ (up to $\sim$3$\sigma$), exactly as the first-order Taylor expansion in Eq.~\eqref{eq:ood_bias} predicts.
At larger exceedances ($\delta > 3\sigma$), higher-order terms begin to contribute and the relationship curves slightly, consistent with the $\mathcal{O}(\delta^2)$ remainder.

The slope $\alpha$ of the bias--exceedance relationship, estimated by linear regression, varies across models more than a purely architecture-independent theory would predict.
Aurora exhibits the steepest slope at day~5 ($\alpha \approx 0.28$, $R^2 = 0.52$), meaning it underestimates each additional standard deviation of exceedance by $\sim$0.28\,$\sigma$ on average.
AIFS shows near-zero slope for heat events ($\alpha \approx 0.001$, $R^2 \approx 0$ for the PNW heatwave), suggesting that ECMWF's training pipeline (which includes fine-tuning on operational analyses with broader data coverage and augmentation strategies beyond raw ERA5) has partially mitigated warm-extreme tail attenuation.
However, AIFS exhibits among the \emph{highest} attenuation for cold extremes ($\alpha \approx 0.44$ for both the Texas freeze and Winter Storm Elliott), revealing a striking warm--cold asymmetry: the same training pipeline that effectively eliminates heat-event OOD bias provides no mitigation for cold-extreme bias, likely because cold outbreaks involve dynamically distinct mechanisms (polar vortex disruptions, Arctic air intrusions) that are more poorly sampled in ERA5.
The inter-model variation in $\alpha$ correlates more with training strategy (pressure-level weighting, data augmentation, loss function weighting) than with architecture family, consistent with the pipeline-dominance thesis.
GNN-based models span the full range of $\alpha$ values (AIFS at the low end, GraphCast in the middle), while vision transformers similarly span a wide range (Aurora moderate, FengWu higher), confirming that architecture alone does not determine extreme-event skill.

The cold extreme events (Texas freeze, Fig.~\ref{fig:tail_texas_main}; Winter Storm Elliott, Supplementary Fig.~\ref{fig:tail_elliott}) reveal substantially stronger tail attenuation than the heat events.
At day~5, the attenuation coefficient $\alpha$ for cold events ranges from $\sim$0.34 to $\sim$0.56 across all models, roughly 2--5$\times$ larger than for heat events ($\alpha \lesssim 0.28$).
This asymmetry is physically plausible: cold extremes (Arctic outbreaks, polar vortex disruptions) involve dynamically distinct mechanisms from heat extremes, with sharper temperature gradients and stronger nonlinear interactions that are more poorly sampled in the ERA5 training climatology.
The cold-event $\alpha$ values also exhibit stronger inter-model variation, with FCN3 showing the steepest slopes ($\alpha > 0.55$ for the Texas freeze), suggesting that some pipeline configurations are particularly vulnerable to cold-extreme OOD bias.

The tail fidelity analysis also validates the Extreme Event Skill (EES) metric: models with low $\alpha$ (good tail fidelity) correspond to high EES values, confirming that the metric captures the intended diagnostic dimension.
The EES quantifies how well a forecast preserves accuracy in the tails of the atmospheric distribution relative to its overall accuracy:
\begin{equation}
\text{EES} = 1 - \frac{\text{RMSE}_{\rm cond}}{3\,\text{RMSE}_{\rm uncond}},
\label{eq:EES}
\end{equation}
where $\text{RMSE}_{\rm cond}$ is computed over grid points where the ERA5 verification exceeds the grid-point-specific climatological threshold $\mu_{\rm clim} + 2\sigma_{\rm clim}$ (i.e., the extreme tail, with $\mu_{\rm clim}$ and $\sigma_{\rm clim}$ computed from ERA5 using a 15-day day-of-year window at each grid point), and $\text{RMSE}_{\rm uncond}$ is the standard global RMSE.
The factor of 3 in the denominator is chosen so that EES~$= 0$ when the conditional RMSE is three times the unconditional RMSE (indicating disproportionately large extreme-event errors), while EES~$\to 1$ when both are comparable.
In our evaluation, EES values range narrowly from $\sim$0.61 to $\sim$0.68 across models and lead times (Tables~\ref{tab:hmas_day3}--\ref{tab:hmas_day15}), with Pangu achieving the highest EES ($\approx 0.68$ at day~3) due to its pressure-weighted loss that implicitly upweights extreme surface temperatures.

\begin{figure*}[!t]
\centering
\includegraphics[width=\textwidth]{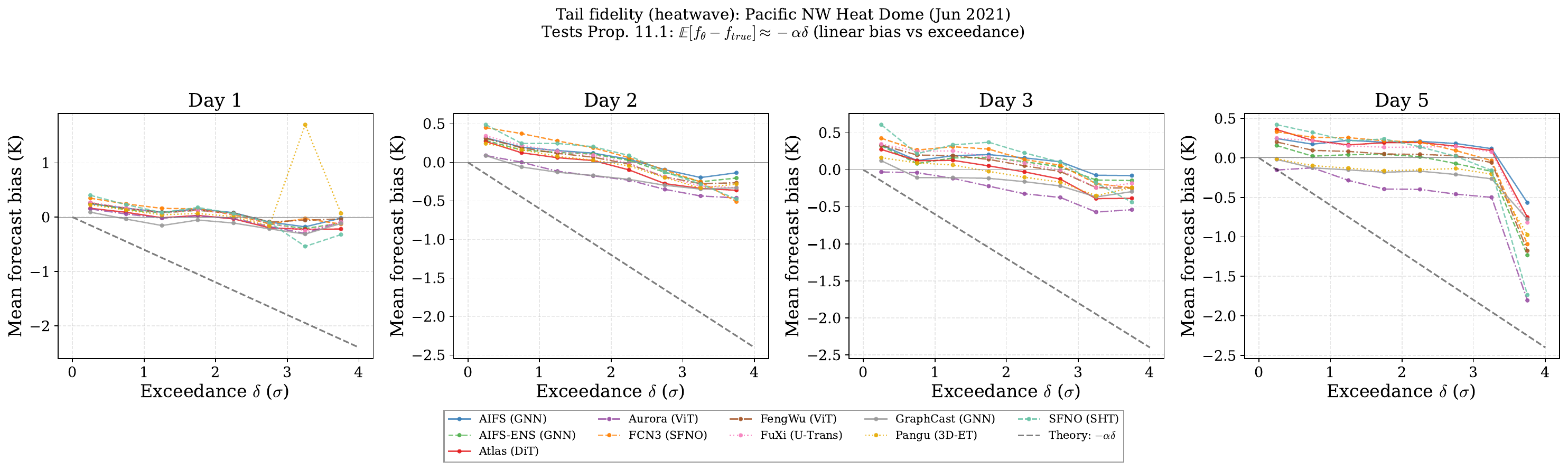}
\caption{Tail fidelity for the 2021 PNW heatwave.
Mean forecast bias (K) vs.\ standardized exceedance $\delta$ ($\sigma$) at lead times 1--5\,days.
All models show increasingly negative bias at higher exceedances, confirming Proposition~\ref{thm:ood_bias}.
See Supplementary Figs.~\ref{fig:tail_european}--\ref{fig:tail_elliott} for the 2023 European heatwave, 2021 Texas freeze, and Winter Storm Elliott, and Figs.~\ref{fig:alpha_evolution_pnw}--\ref{fig:alpha_evolution_elliott} for $\alpha$ evolution curves across all four events.}
\label{fig:tail_fidelity}
\end{figure*}

\begin{figure*}[!t]
\centering
\includegraphics[width=\textwidth]{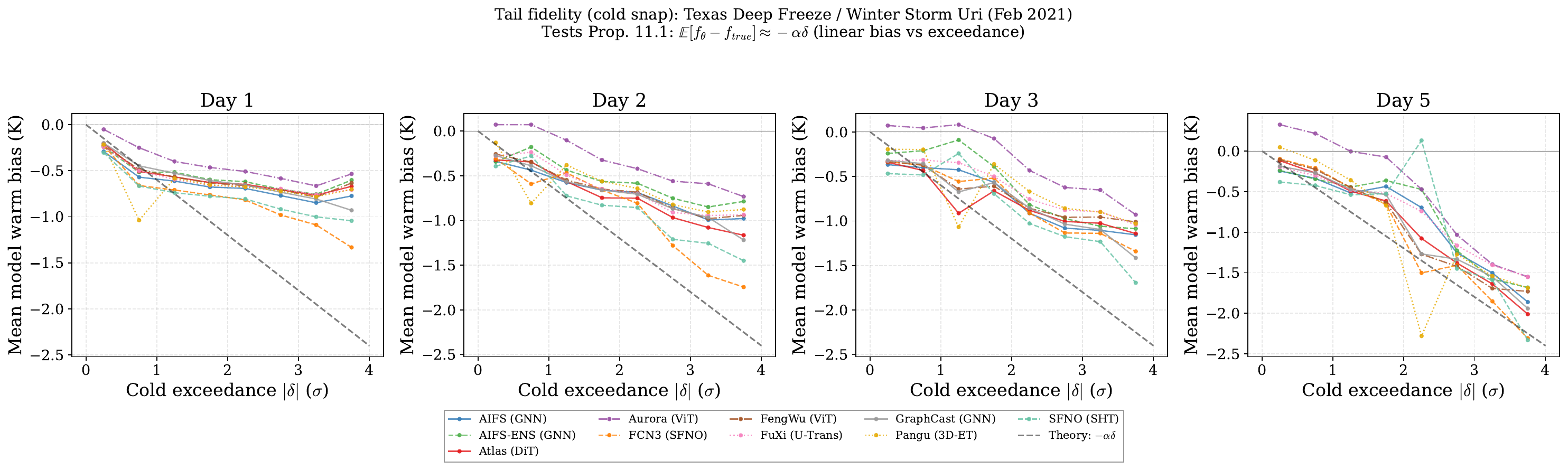}
\caption{Tail fidelity for the February 2021 Texas freeze (cold extreme).
Attenuation coefficients $\alpha \approx 0.34$--$0.56$ at day~5 are substantially larger than for heat events ($\alpha \lesssim 0.28$), revealing a systematic cold--warm asymmetry in OOD bias.
Notably, AIFS, which achieves near-zero $\alpha$ for heat events, shows $\alpha \approx 0.44$ here, demonstrating that warm-extreme mitigation does not transfer to cold extremes.
See Supplementary Figs.~\ref{fig:tail_elliott}--\ref{fig:alpha_evolution_elliott} for Winter Storm Elliott.}
\label{fig:tail_texas_main}
\end{figure*}

\section{Holistic Model Assessment}
\label{sec:holistic}

We synthesize the theoretical framework into a composite score.
Where the Physical Consistency Score (PCS, \S\ref{sec:physical_consistency}) is itself a composite of four balance sub-scores for a single diagnostic dimension, the Holistic Model Assessment Score (HMAS) operates at a higher level: it combines six \emph{distinct} diagnostic dimensions, of which PCS is one, into a single overall ranking.
The six dimensions are chosen to capture complementary aspects of forecast quality that no single metric can represent:

\begin{definition}[Holistic Model Assessment Score]
\label{def:HMAS}
\begin{equation}
\text{HMAS} = \sum_{i=1}^{6} w_i\, M_i,
\label{eq:HMAS}
\end{equation}
where the six metrics and their default weights are:
SFI ($w = 0.20$), $\ell_{\rm eff}$ ($w = 0.15$), $\tau_d$ ($w = 0.15$), EES ($w = 0.15$), PCS ($w = 0.15$), ASI ($w = 0.20$).
Each metric $M_i \in [0,1]$.
The weights balance operational importance with empirical discriminating power: ASI receives the highest weight because autoregressive energy stability is both critical for extended-range utility and the most discriminating metric across models (range $0.01$--$0.92$); SFI receives equal weight as the primary diagnostic of loss-function-induced spectral bias; PCS captures a unique, irreducible dimension (the SFI--PCS anti-correlation of $\rho = -0.86$ confirms this); $\tau_d$ receives a reduced weight because the error doubling time is primarily an atmospheric property (Lyapunov growth rate) with limited inter-model variation.
\end{definition}

The HMAS tables at three forecast horizons are shown in Tables~\ref{tab:hmas_day3}--\ref{tab:hmas_day15}.

\begin{table*}[!t]
\centering
\caption{Holistic Model Assessment Scores (HMAS) at Day 3 (short-range). Models ranked by HMAS.}
\label{tab:hmas_day3}
\scriptsize
\setlength{\tabcolsep}{5pt}
\begin{tabular}{lccccccc}
\toprule
\textbf{Model} & \textbf{SFI} & $\bm{\ell_{\textbf{eff}}}$ & $\bm{\tau_d}$ & \textbf{EES} & \textbf{PCS} & \textbf{ASI} & \textbf{HMAS} \\
\midrule
FCN3 (SFNO)      & 0.978 & 1.000 & 0.774 & 0.646 & 0.545 & 0.913 & \textbf{0.823} \\
Atlas (DiT)      & 0.936 & 1.000 & 0.713 & 0.643 & 0.524 & 0.904 & \textbf{0.800} \\
AIFS-ENS (GNN)   & 0.946 & 0.980 & 0.743 & 0.642 & 0.473 & 0.895 & \textbf{0.794} \\
Aurora (ViT)     & 0.799 & 1.000 & 0.721 & 0.640 & 0.586 & 0.900 & \textbf{0.782} \\
Pangu (3D-ET)    & 0.667 & 1.000 & 0.762 & 0.676 & 0.626 & 0.897 & \textbf{0.772} \\
GraphCast (GNN)  & 0.632 & 1.000 & 0.774 & 0.646 & 0.593 & 0.842 & \textbf{0.747} \\
AIFS (GNN)       & 0.624 & 0.365 & 0.749 & 0.645 & 0.631 & 0.857 & \textbf{0.655} \\
SFNO (SHT)       & 0.680 & 0.185 & 0.687 & 0.641 & 0.627 & 0.921 & \textbf{0.641} \\
FuXi (U-Trans)   & 0.692 & 1.000 & 0.732 & 0.642 & 0.583 & 0.020 & \textbf{0.586} \\
FengWu (ViT)     & 0.462 & 0.093 & 0.723 & 0.643 & 0.657 & 0.012 & \textbf{0.412} \\
\bottomrule
\end{tabular}
\end{table*}

\begin{table*}[!t]
\centering
\caption{Holistic Model Assessment Scores (HMAS) at Day 5 (medium-range). Models ranked by HMAS.}
\label{tab:hmas_day5}
\scriptsize
\setlength{\tabcolsep}{5pt}
\begin{tabular}{lccccccc}
\toprule
\textbf{Model} & \textbf{SFI} & $\bm{\ell_{\textbf{eff}}}$ & $\bm{\tau_d}$ & \textbf{EES} & \textbf{PCS} & \textbf{ASI} & \textbf{HMAS} \\
\midrule
FCN3 (SFNO)      & 0.977 & 1.000 & 0.774 & 0.634 & 0.540 & 0.913 & \textbf{0.820} \\
Atlas (DiT)      & 0.938 & 1.000 & 0.713 & 0.634 & 0.511 & 0.904 & \textbf{0.797} \\
AIFS-ENS (GNN)   & 0.946 & 0.976 & 0.743 & 0.632 & 0.472 & 0.895 & \textbf{0.792} \\
Aurora (ViT)     & 0.798 & 0.976 & 0.721 & 0.630 & 0.588 & 0.900 & \textbf{0.777} \\
Pangu (3D-ET)    & 0.663 & 0.948 & 0.762 & 0.655 & 0.625 & 0.897 & \textbf{0.761} \\
GraphCast (GNN)  & 0.615 & 0.910 & 0.774 & 0.642 & 0.588 & 0.842 & \textbf{0.729} \\
AIFS (GNN)       & 0.592 & 0.372 & 0.749 & 0.637 & 0.630 & 0.857 & \textbf{0.648} \\
SFNO (SHT)       & 0.679 & 0.183 & 0.687 & 0.625 & 0.619 & 0.921 & \textbf{0.637} \\
FuXi (U-Trans)   & 0.693 & 1.000 & 0.732 & 0.636 & 0.541 & 0.020 & \textbf{0.579} \\
FengWu (ViT)     & 0.341 & 0.050 & 0.723 & 0.633 & 0.667 & 0.012 & \textbf{0.382} \\
\bottomrule
\end{tabular}
\end{table*}

\begin{table*}[!t]
\centering
\caption{Holistic Model Assessment Scores (HMAS) at Day 15 (extended-range). Models ranked by HMAS.
Note: FengWu's $\ell_{\rm eff} = 1.0$ is a metric artifact caused by spectral energy inflation exceeding the half-power threshold at all wavenumbers despite very low SFI ($= 0.078$); see text following Definition~\ref{def:eff_res}.}
\label{tab:hmas_day15}
\scriptsize
\setlength{\tabcolsep}{5pt}
\begin{tabular}{lccccccc}
\toprule
\textbf{Model} & \textbf{SFI} & $\bm{\ell_{\textbf{eff}}}$ & $\bm{\tau_d}$ & \textbf{EES} & \textbf{PCS} & \textbf{ASI} & \textbf{HMAS} \\
\midrule
FCN3 (SFNO)      & 0.973 & 1.000 & 0.774 & 0.632 & 0.541 & 0.913 & \textbf{0.819} \\
AIFS-ENS (GNN)   & 0.948 & 0.982 & 0.743 & 0.634 & 0.472 & 0.895 & \textbf{0.793} \\
Atlas (DiT)      & 0.953 & 0.987 & 0.713 & 0.624 & 0.397 & 0.904 & \textbf{0.780} \\
Aurora (ViT)     & 0.790 & 0.863 & 0.721 & 0.631 & 0.590 & 0.900 & \textbf{0.759} \\
Pangu (3D-ET)    & 0.654 & 0.840 & 0.762 & 0.641 & 0.628 & 0.897 & \textbf{0.741} \\
GraphCast (GNN)  & 0.599 & 0.851 & 0.774 & 0.629 & 0.568 & 0.842 & \textbf{0.711} \\
SFNO (SHT)       & 0.678 & 0.182 & 0.687 & 0.609 & 0.599 & 0.921 & \textbf{0.631} \\
AIFS (GNN)       & 0.567 & 0.276 & 0.749 & 0.631 & 0.594 & 0.857 & \textbf{0.622} \\
FuXi (U-Trans)   & 0.464 & 0.857 & 0.732 & 0.634 & 0.335 & 0.020 & \textbf{0.480} \\
FengWu (ViT)     & 0.078 & 1.000 & 0.723 & 0.625 & 0.456 & 0.012 & \textbf{0.439} \\
\bottomrule
\end{tabular}
\end{table*}

Figure~\ref{fig:radar} presents the HMAS plot at day~5, which visually encodes the multi-dimensional ``profile'' of each model.
The parallel vertical axis representation reveals a key finding that scalar HMAS scores alone cannot convey: \emph{no model dominates all six dimensions simultaneously}.
Each model exhibits a distinctive shape, reflecting the trade-offs identified throughout this paper.
FCN3, which achieves the highest composite HMAS (Table~\ref{tab:hmas_day5}), does so through near-perfect spectral fidelity (SFI $= 0.977$) and high effective resolution ($\ell_{\rm eff} = 1.0$), leveraging its spectral architecture to preserve the energy spectrum across wavenumbers, but at the cost of moderate physical consistency (PCS $= 0.540$).
Atlas shows a similar spectral profile: the second-highest HMAS, driven by strong SFI and $\ell_{\rm eff}$ from its score-matching loss, but with degraded PCS ($= 0.511$) relative to ERA5, reflecting the spectral fidelity--physical consistency trade-off identified in Remark~\ref{rem:pcs_tradeoff}.
Pangu achieves the highest EES ($= 0.655$) among all models and strong PCS ($= 0.625$), indicating that its 3D Earth-Specific Transformer architecture, hierarchical temporal aggregation inference strategy, and pressure-weighted MSE loss produce well-balanced forecasts with good extreme-event skill, but at the cost of lower SFI ($= 0.663$) and reduced $\ell_{\rm eff}$ ($= 0.948$).
Notably, FengWu achieves the highest PCS at day~5 ($\approx 0.67$) despite having the lowest overall HMAS, illustrating that physical consistency alone does not guarantee overall forecast quality. FengWu's catastrophic energy instability (ASI $\approx 0$) overwhelms its balance fidelity advantage.
Moreover, this PCS advantage is transient: by day~15, FengWu's PCS drops to $\approx 0.46$ as energy drift progressively undermines dynamical balance.
FuXi exhibits an even steeper PCS collapse ($0.61 \to 0.34$), making both models cautionary examples of how energy instability cascades into balance degradation at extended range.
GraphCast and AIFS achieve high ASI values ($0.842$ and $0.857$ respectively), indicating good energy conservation, but their SFI scores are moderate: they conserve the total energy budget while redistributing it away from high-wavenumber modes.
The surprisingly low $\ell_{\rm eff}$ values for AIFS ($0.372$) and SFNO ($0.183$) at day~5 deserve explanation: the spectral metrics in Tables~\ref{tab:hmas_day3}--\ref{tab:hmas_day15} are computed on U500 (kinetic energy), and both models preserve Z500 effective resolution at $\ell_{\rm eff} = 1.0$ through day~5.
The large discrepancy reflects the variable-dependent spectral smoothing predicted by our framework (\S\ref{sec:prelim}): Z500 is a smooth, high-regularity field ($s \approx 1.5$) whose spectral energy is concentrated at low wavenumbers, making it robust to MSE-induced truncation.
U500, with lower effective Sobolev regularity, retains more energy at high wavenumbers and is therefore more vulnerable to MSE-driven suppression.
FCN3 and Atlas maintain $\ell_{\rm eff} = 1.0$ for both Z500 and U500, confirming that this variable-dependent deficit is loss-induced rather than architecture-dependent.

These distinct profiles are the principal justification for a composite assessment: a single metric (RMSE, ACC, or any other) would collapse these multi-dimensional trade-offs into a one-dimensional ranking, obscuring the fact that the ``best'' model depends critically on the application.
We caution that small HMAS differences (e.g., $<$0.02) should not be over-interpreted as meaningful, particularly at extended range where inter-initialization variability is larger; the value of HMAS lies in identifying distinct model tiers and diagnostic profiles rather than in fine-grained ordinal rankings.
A forecaster prioritizing medium-range deterministic skill would choose differently from one prioritizing spectral fidelity for downstream convective-scale modeling, or one prioritizing extreme-event early warning.

\begin{figure*}[!t]
\centering
\includegraphics[width=\textwidth]{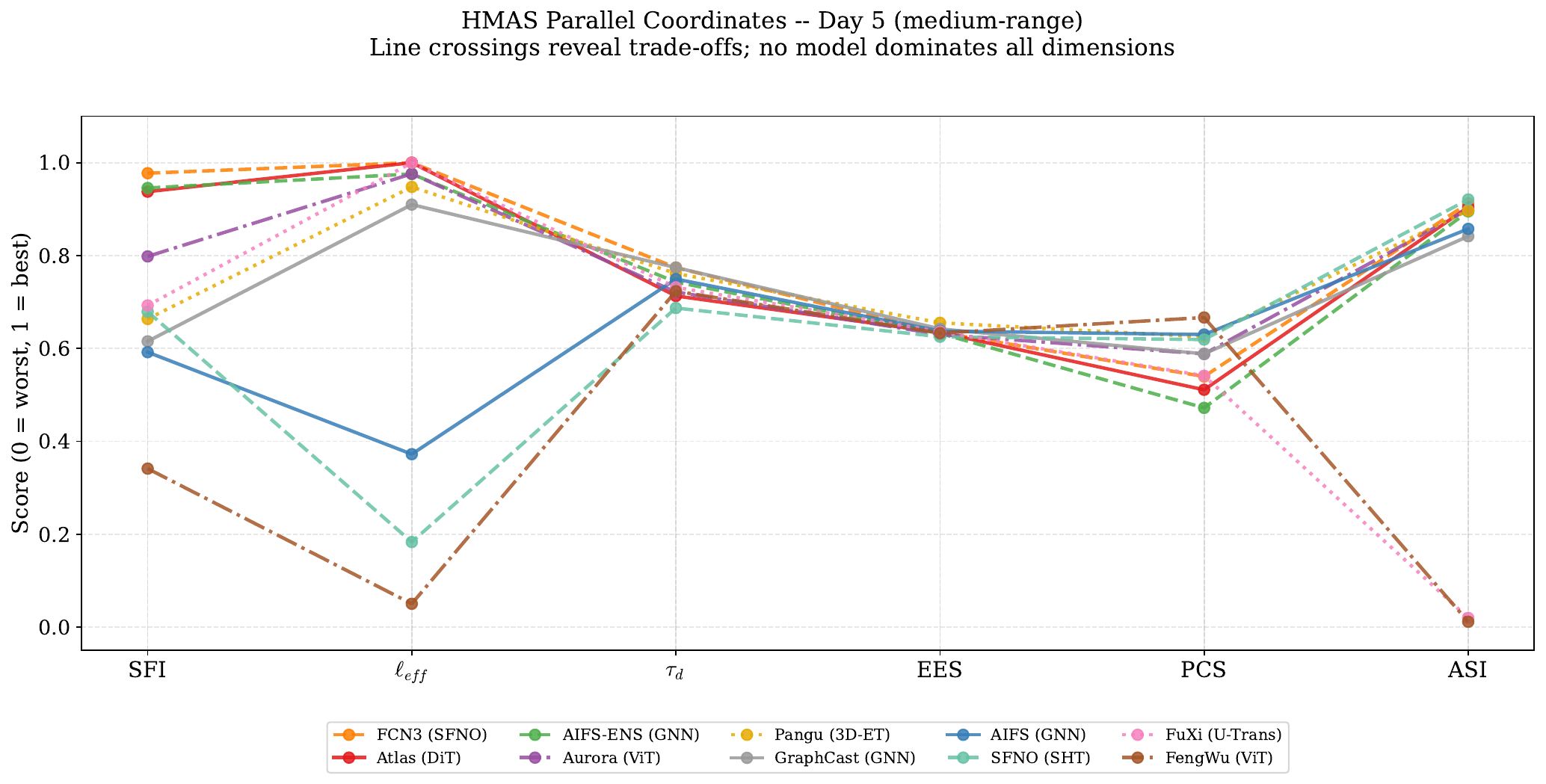}
\caption{HMAS parallel coordinates at day~5 (medium-range).
Each line represents one model; each vertical axis represents one of the six HMAS metrics.
Line crossings between adjacent axes reveal trade-offs: for example, models with high SFI (spectral fidelity) tend to show lower PCS (physical consistency), and vice versa.
No single model achieves the highest score on all six dimensions simultaneously.
Models are ordered in the legend by composite HMAS (highest first).
See Supplementary Figs.~\ref{fig:radar_day3}--\ref{fig:radar_day15} for radar chart representations at day~3 and day~15.}
\label{fig:radar}
\end{figure*}

\paragraph{Weight sensitivity and metric independence.}
A weight sensitivity analysis testing five distinct weighting schemes (equal, accuracy-focused, extremes-focused, stability-focused, and default) yields Kendall's $W = 0.97$ concordance, confirming that model rankings are robust to weight perturbation (Supplementary Fig.~\ref{fig:hmas_sensitivity}).
The HMAS dimension cross-correlation matrix (Supplementary Fig.~\ref{fig:hmas_correlation}) reveals that SFI and PCS are strongly anti-correlated ($\rho = -0.86$), confirming the spectral fidelity--physical consistency trade-off, while ASI is largely independent of other metrics, and the mean absolute off-diagonal correlation $\bar{|\rho|} = 0.38$ confirms that the six dimensions are sufficiently independent to justify a composite score.

\section{A Holistic Decision Framework}
\label{sec:decision}

Based on the theoretical analysis and empirical validation, we provide a decision framework prioritizing the \emph{complete learning pipeline} ($\mathcal{L}$, $\mathcal{D}$, $\mathcal{T}$, $\mathcal{A}$; defined in \S\ref{sec:pipeline}).
The ordering of decisions below reflects the error decomposition of Proposition~\ref{prop:dominance}: components that contribute more to total forecast error should be optimized first.
Our empirical results across ten models allow us to ground each recommendation in specific evidence rather than theoretical speculation alone.

\subsection{Decision Hierarchy}

For a forecast task characterized by target lead time $\tau$, scale range $[\ell_{\min}, \ell_{\max}]$, and output type (deterministic vs.\ probabilistic), the following hierarchy orders the pipeline decisions from highest to lowest impact on forecast skill at current operational scales.

\begin{enumerate}[leftmargin=*,itemsep=4pt]
\item \textbf{Step 1: Choose the loss function} (highest impact; \S\ref{sec:loss}).

The loss function is the single most consequential design choice in an AI weather prediction system.
This claim is supported by three converging lines of evidence from our analysis.
First, Theorem~\ref{thm:mse_bias} proves that MSE optimization necessarily produces a spectral energy deficit equal to the conditionally unpredictable variance $\mathrm{Var}_\ell(\tau)$ at each wavenumber, a deficit that grows with lead time and is entirely independent of architecture.
Figure~\ref{fig:spectral_comparison} confirms this: all seven MSE-trained models, despite spanning five architecture families, exhibit the same qualitative spectral deficit, while Atlas (score matching), FCN3 (CRPS), and AIFS-ENS (afCRPS) do not.
Second, the MSH loss results of \citet{subich2025msh} demonstrated that changing \emph{only} the loss function on an otherwise identical GraphCast GNN improved effective resolution from 1,250\,km to 160\,km, a factor-of-eight improvement.
While this was demonstrated on a single architecture (GraphCast), the theoretical basis (Theorem~\ref{thm:mse_bias}) is architecture-independent, and the empirical confirmation that all seven MSE-trained models in our study exhibit the same spectral deficit pattern (Fig.~\ref{fig:spectral_comparison}) supports the generality of loss-function dominance across architectures.
The FastNet study \citep{daub2025fastnet} provides additional cross-architecture evidence, showing that simplified architectures with optimized training matched more complex designs.
Third, the HMAS cross-correlation matrix (Fig.~\ref{fig:hmas_correlation}) shows that the SFI--PCS trade-off ($\rho = -0.86$) is fundamentally a loss-function trade-off: MSE-based losses sacrifice spectral fidelity for physical consistency, while score matching does the reverse.

The practical recommendation depends on the output type.
For \emph{deterministic forecasts}, the MSH loss \citep{subich2025msh} eliminates the double penalty (Proposition~\ref{prop:double_penalty}) by separating amplitude and phase errors in spectral space, preserving sharp features without requiring ensemble sampling.
For \emph{probabilistic forecasts}, CRPS \citep{gneiting2007strictly} preserves both the mean and the full spectral energy at optimality (Proposition~\ref{prop:crps}), while score matching \citep{song2021score} learns the complete conditional distribution but at higher computational cost.
For \emph{extreme event applications}, the choice of loss function is particularly consequential because standard MSE training produces the largest errors precisely where accuracy matters most (Proposition~\ref{thm:ood_bias}).
Several complementary approaches address this, spanning both diffusion and non-diffusion model classes.
For \emph{diffusion-based models}, heavy-tailed score matching with a Student-$t$ prior \citep{pandey2024heavy} should be preferred over the standard Gaussian prior, which systematically underestimates tail probabilities (Remark~\ref{rem:heavy_tails}).
For \emph{deterministic models}, the MSH loss \citep{subich2025msh} provides an indirect benefit: by preserving sharp gradients and small-scale amplitude, it reduces the smoothing that attenuates localized extremes, though it does not explicitly target the distributional tails.
Threshold-weighted scoring rules offer a more direct approach: by upweighting the loss contribution from grid points exceeding a climatological threshold (e.g., the 95th or 99th percentile), the optimizer is forced to allocate capacity to the tails of the distribution rather than minimizing average-case error.
For \emph{ensemble models}, tail-weighted CRPS variants that emphasize the upper or lower quantiles of the forecast distribution can improve calibration in the extremes without sacrificing overall probabilistic skill.
More broadly, quantile regression losses that target specific extreme quantiles (e.g., the $1^{st}$ and $99^{th}$ percentiles) provide a model-agnostic mechanism for improving tail fidelity, applicable to any architecture.
These loss-level interventions are complementary to data-level strategies (discussed in Step~2 below); in practice, the most robust extreme-event pipelines will combine tail-aware loss functions with augmented training data that oversamples rare events.

\item \textbf{Step 2: Design the data strategy} (second-highest impact; \S\ref{sec:pipeline}).

Our empirical results show the importance of $\varepsilon_{\rm data}$ in two ways.
The OOD analysis (Fig.~\ref{fig:tail_fidelity}) demonstrates that models trained only on ERA5 reanalysis systematically attenuate extremes exceeding their training climatology, with bias growing linearly in the exceedance $\delta$ (Proposition~\ref{thm:ood_bias}).
The inter-model variation in the attenuation coefficient $\alpha$ correlates with training data diversity: AIFS, which benefits from ECMWF's operational data pipeline and extensive data augmentation, achieves $\alpha \approx 0.001$ (near-zero attenuation) for heat events, while models trained on raw ERA5 alone show $\alpha$ up to $\sim$0.28 for heat events and $\sim$0.34--0.56 for cold events.
Aurora's multi-source pretraining strategy (ERA5 reanalysis + GFS operational analyses + CMIP6 climate simulations) similarly reduces $\varepsilon_{\rm data}$ by exposing the model to multiple data modalities spanning reanalysis, operational forecasts, and climate projections, each with distinct bias characteristics.

The recommendation is to maximize the diversity and coverage of training data before investing in architectural refinement.
For medium-range forecasting, multi-source pretraining on multiple reanalysis products and operational analyses reduces distribution shift.
For extreme event applications, augmenting training data with synthetic extremes or reweighting the loss toward tail events is essential to mitigate the OOD attenuation predicted by Proposition~\ref{thm:ood_bias}.

\item \textbf{Step 3: Select the training methodology} (moderate impact).

Training methodology encompasses the autoregressive rollout strategy, learning rate schedule, fine-tuning protocol, and timestep.
Our scorecard analysis (Fig.~\ref{fig:scorecard}) provides direct evidence that these choices matter: model rankings change across lead-time horizons, and the rank changes correlate with training methodology rather than architecture.
Pangu's hierarchical temporal aggregation inference strategy achieves the highest EES and strong PCS among all models (Tables~\ref{tab:hmas_day3}--\ref{tab:hmas_day15}) through excellent physical consistency, but at the cost of spectral fidelity. This is a consequence of its combined pipeline choices (pressure-level weighting in the MSE loss, 24-hour timestep, and the 3D Earth-Specific Transformer architecture), which together favor balanced, smooth forecasts over high-wavenumber fidelity.
FengWu's 6-hour timestep achieves better short-range accuracy but suffers catastrophic energy drift at extended range (ASI $\approx 0$, Fig.~\ref{fig:error_growth}).
Since Aurora also uses a 6-hour timestep without similar degradation, the instability is not attributable to timestep alone but rather to the interaction between timestep and training methodology (Aurora's LoRA-based rollout fine-tuning appears to mitigate the accumulation of autoregressive errors).
These are training methodology trade-offs, not architecture trade-offs: both Pangu and FengWu use transformer-based architectures, but their different pipeline choices (loss weighting, timestep, rollout strategy) produce fundamentally different forecast ``profiles.''

For iterative medium-range models, autoregressive rollout fine-tuning, where the model is fine-tuned on multi-step forecasts rather than single-step predictions, is critical for controlling error accumulation.
Parameter-efficient fine-tuning (LoRA) can adapt pretrained models to specific forecast tasks without full retraining.
Aurora is a notable example: its rollout fine-tuning stage uses LoRA on all self-attention layers \citep{bodnar2025aurora}, enabling multi-step autoregressive training on a 1.3B parameter model at manageable memory cost. This technique is not yet widely adopted by the other models in our evaluation but is readily transferable.
The choice of timestep represents a stability--fidelity trade-off: shorter timesteps capture sub-daily dynamics but require more autoregressive steps (amplifying accumulated error), while longer timesteps are more stable but miss diurnal and sub-synoptic variability.

\item \textbf{Step 4: Select the architecture} (lowest marginal impact at current scales; \S\ref{sec:approx}).

Architecture selection is deliberately placed last because the theoretical and empirical evidence consistently indicates that it contributes least to total forecast error at current operational scales.
The unified convergence rate $\mathcal{O}(N^{-s/2})$ (Propositions~\ref{prop:sfno_approx}--\ref{prop:gnn_approx}) shows that all major architectures achieve comparable approximation error at $N \sim 10^6$ grid points.
The ECR analysis (Fig.~\ref{fig:error_consensus}) confirms that $>$58\% of forecast error variance is shared across architectures even at day~15---meaning the majority of error is predictability-limited, not architecture-limited.

This does not mean architecture is irrelevant at all scales.
The MED analysis (Fig.~\ref{fig:error_convergence}) reveals that architecture leaves its fingerprint at fine scales ($\ell > 50$), and specific architectural properties become important for secondary criteria: spherical equivariance (SFNO, FCN3) provides built-in rotational consistency, which may contribute to the relatively high PCS scores of spectral models; global attention mechanisms (transformers) enable efficient representation of teleconnections (e.g., ENSO patterns, stratosphere--troposphere coupling); graph-based message passing (GraphCast, AIFS) naturally respects irregular mesh structures and achieves good energy conservation (ASI $\approx 0.84$--$0.86$); and diffusion backbones (Atlas) are required for sampling the full conditional distribution.
The recommendation is to choose architecture based on these secondary properties (stability, equivariance, probabilistic output capability) after the loss function, data strategy, and training methodology have been determined.
\end{enumerate}

\paragraph{Putting it together: task-specific recommendations.}
Table~\ref{tab:recommendations} synthesizes the decision hierarchy into concrete recommendations for four forecast tasks spanning different timescales and objectives.
The ``Primary determinant'' column identifies which pipeline component matters most for each task, based on the empirical evidence from our ten-model evaluation.
Short-range forecasting is primarily limited by the training method (timestep, rollout strategy), as all models have adequate spectral fidelity and data coverage at 1--3 day lead times.
Medium-range forecasting is limited by the loss function and data coverage, as MSE-induced spectral bias becomes the dominant error source by day~5 (Fig.~\ref{fig:spectral_ratio}) and data distribution gaps increasingly affect the tails of the forecast distribution.
Extreme event forecasting is limited by training data coverage and OOD extrapolation bounds (Proposition~\ref{thm:ood_bias}), requiring both data augmentation and loss function modifications (heavy-tailed diffusion).
Climate emulation prioritizes stability over short-range accuracy, as the model must remain on the atmospheric attractor over $\mathcal{O}(10^5)$ autoregressive steps. Here, the training methodology (spectral regularization, energy conservation constraints) and architecture (parsimonious to avoid overfitting) become the primary determinants.

\begin{table*}[t]
\centering
\caption{Holistic pipeline assessment across forecast timescales.}
\label{tab:recommendations}
\scriptsize
\setlength{\tabcolsep}{3pt}
\resizebox{\textwidth}{!}{%
\begin{tabular}{@{}lllllp{2.5cm}@{}}
\toprule
\textbf{Task} & \textbf{Loss function} & \textbf{Data strategy} & \textbf{Training method} & \textbf{Architecture} & \textbf{Primary determinant}\\
\midrule
Short-range (1--3\,d) & MSH or CRPS & ERA5 + oper.\ analyses & AR rollout fine-tune & Any hybrid & Training method\\[2pt]
Medium-range (3--15\,d) & MSH + ensemble & Multi-source & AR rollout + LoRA & Hybrid GNN--transformer & Loss + data\\[2pt]
Extremes & Heavy-tailed score match. & + synthetic extremes & Augmented OOD training & Any + physics constraints & Training data + OOD bounds\\[2pt]
\bottomrule
\end{tabular}%
}
\end{table*}

\section{From Diagnostic to Prescriptive: Multi-Objective Pipeline Optimality}
\label{sec:prescriptive}

The preceding sections have established a comprehensive \emph{diagnostic} framework: given a trained model, our metrics and theorems characterize its strengths and failure modes.
However, the more consequential question for advancing AI paradigms in atmospheric science is \emph{prescriptive}: given a forecast task, can we determine whether a proposed pipeline configuration (a specific $\{\mathcal{L}, \mathcal{D}, \mathcal{T}, \mathcal{A}\}$ tuple) will achieve adequate skill \emph{before} committing to training?
We argue that the mathematical machinery developed in this paper contains the essential ingredients for such a pre-training suitability theory, and we formalize the first steps here.

\subsection{From Single-Objective to Multi-Objective Optimality}

\citet{beucler2025pareto} recently proposed that Pareto-optimal model hierarchies, defined within an error-complexity plane, can guide the development of data-driven atmospheric models and distill the added value of machine learning.
Their framework operates in a two-dimensional space: a single error metric (e.g., MSE) on one axis and model complexity (typically parameter count) on the other.
The Pareto front identifies models achieving the minimum error for each level of complexity.

Our HMAS analysis (\S\ref{sec:holistic}) demonstrates that this two-dimensional formulation needs to be extended for weather prediction systems.
Forecast quality is inherently multi-dimensional (at minimum six-dimensional in the HMAS metric space) and the dimensions are not redundant.
The cross-correlation matrix (Fig.~\ref{fig:hmas_correlation}) reveals that SFI and PCS are anti-correlated at $\rho = -0.86$: no single pipeline configuration \emph{among those evaluated here} can simultaneously maximize spectral fidelity and physical consistency.
Whether this trade-off is fundamental or merely reflects the current state of the art is an open question.
Approaches such as physics-informed loss terms that explicitly penalize balance violations, multi-objective reinforcement learning, or hybrid architectures that combine generative sampling with dynamical constraints could potentially reduce this anti-correlation in future systems.
This anti-correlation is not a measurement artifact but a fundamental consequence of the loss function taxonomy (Table~\ref{tab:loss_taxonomy}): MSE-based losses sacrifice spectral fidelity for inter-variable balance (high PCS, low SFI), while score matching preserves the marginal energy spectrum of each field but degrades the joint dynamical relationships across fields (high SFI, low PCS).

We therefore extend the Pareto front concept to a \emph{multi-objective} setting where the ``complexity'' axis is replaced by the pipeline configuration space.

\begin{definition}[Pipeline Pareto Surface]
\label{def:pareto_surface}
Let $\mathbf{m}(\mathbf{p}) = (\text{SFI}, \text{RMSE}^{-1}, \text{PCS}, \text{EES}, \text{ACC}, \text{ASI})$ be the metric vector achieved by pipeline configuration $\mathbf{p} = (\mathcal{L}, \mathcal{D}, \mathcal{T}, \mathcal{A})$ at a given lead time $\tau$.
A configuration $\mathbf{p}^*$ is \emph{Pareto-optimal} if there exists no $\mathbf{p}'$ such that $m_i(\mathbf{p}') \geq m_i(\mathbf{p}^*)$ for all $i$ with strict inequality for at least one $i$.
The set of all Pareto-optimal configurations forms the \emph{Pipeline Pareto Surface} $\mathcal{P}(\tau)$.
\end{definition}

Unlike the single-objective Pareto front of \citet{beucler2025pareto}, the Pipeline Pareto Surface is a multi-dimensional manifold: movement along the surface represents trade-offs between competing forecast quality dimensions, while movement \emph{toward} the surface represents genuine improvement.
The practical implication is that a pipeline designer must first decide which region of the Pareto surface they are targeting (spectral fidelity for downstream convective modeling, physical consistency for coupled Earth system applications, or extreme event skill for hazard warning) before the framework can recommend a pipeline.

\subsection{Pre-Computable Bounds on the Pareto Surface}

A central advantage of our theoretical framework is that several quantities bounding the Pareto surface are computable from the atmosphere alone, without training any model.
These define an \emph{outer envelope} $\mathcal{E}(\tau)$ such that $\mathcal{P}(\tau) \subseteq \mathcal{E}(\tau)$.

\paragraph{(1) Atmospheric predictability ceiling.}
The mutual information decay (Theorem~\ref{thm:info_decay}) establishes an absolute upper bound on forecast skill at each lead time $\tau$.
The KS entropy $h_{\rm KS} = \sum_{\lambda_i > 0} \lambda_i$ is a property of the atmosphere and can be estimated from ERA5 via the effective Lyapunov exponents (Fig.~\ref{fig:error_growth}).
No pipeline configuration, regardless of loss, data, or architecture, can exceed this ceiling: $\text{ACC}(\tau) \leq \text{ACC}_{\max}(\tau)$ where $\text{ACC}_{\max}$ decays at rate $h_{\rm KS}$.

\paragraph{(2) Loss-induced spectral budget.}
Theorem~\ref{thm:mse_bias} provides the exact spectral energy that \emph{any} MSE-trained model will achieve at each wavenumber and lead time: $\hat{E}_{\rm MSE}(\ell, \tau) = E_{\rm true}(\ell) - \mathrm{Var}_\ell(\tau)$.
The conditional variance $\mathrm{Var}_\ell(\tau)$ is a property of the atmospheric dynamics, estimable from ERA5 by computing the spread among analysis states originating from similar initial conditions.
This means that for a given loss function choice $\mathcal{L}$, the achievable SFI is predictable before training:
\begin{equation}
\text{SFI}_{\rm predicted}(\mathcal{L}, \tau) = 1 - \frac{1}{2|\mathcal{K}|}\sum_{\ell \in \mathcal{K}} \bigl|\log_{10}\bigl(R_\mathcal{L}(\ell,\tau)\bigr)\bigr|,
\label{eq:predicted_sfi}
\end{equation}
where $R_{\rm MSE}(\ell,\tau) = 1 - \mathrm{Var}_\ell(\tau)/E_{\rm true}(\ell)$, $R_{\rm CRPS}(\ell,\tau) \approx 1$, and $R_{\rm score}(\ell,\tau) \approx 1 + \sigma_{\rm sample}^2(\ell)/E_{\rm true}(\ell)$ for score-matching losses (where $\sigma_{\rm sample}^2$ is the sampling noise variance).
This provides a pre-training prediction of whether a proposed loss function will satisfy a spectral fidelity requirement for the target application.

\paragraph{(3) Data coverage and OOD bound.}
Proposition~\ref{thm:ood_bias} predicts the extreme-event bias for any model trained on a dataset $\mathcal{D}$ with climatological maximum $u_{\rm max}$: the bias grows linearly as $-\alpha\,\delta$ for exceedances $\delta > 0$.
The attenuation coefficient $\alpha$ depends on the training distribution's tail behavior, which is computable from $\mathcal{D}$ without training.
Given a target extreme-event scenario (e.g., a 3$\sigma$ exceedance), one can predict the minimum EES achievable under a given $\{\mathcal{L}, \mathcal{D}\}$ pair.

\subsection{The Spectral Feasibility Score}

Combining these pre-computable quantities, we define a composite measure of pipeline suitability.

\begin{definition}[Spectral Feasibility Score]
\label{def:feasibility}
For a proposed pipeline configuration $\mathbf{p} = (\mathcal{L}, \mathcal{D}, \mathcal{T}, \mathcal{A})$ targeting lead time $\tau$, the Spectral Feasibility Score is:
\begin{multline}
\text{SFS}(\mathbf{p}, \tau) = w_1\, \text{SFI}_{\rm predicted}(\mathcal{L}, \tau) + w_2\, C_{\rm data}(\mathcal{D}) \\
+ w_3\, \bigl(1 - h_{\rm KS}\,\tau / I_0\bigr)^+,
\label{eq:sfs}
\end{multline}
where $C_{\rm data}(\mathcal{D})$ is the data coverage score (fraction of the operational atmospheric distribution covered by the training set, estimated via tail coverage and distributional divergence), and $(x)^+ = \max(0,x)$.
\end{definition}

The SFS provides a pre-training lower bound on achievable forecast quality.
A low SFS indicates that the proposed pipeline faces fundamental barriers (e.g., an MSE loss targeting high SFI at long lead times, or a narrow training set targeting extreme events) and should be redesigned before committing computational resources to training.
A high SFS is necessary but not sufficient for good performance: it confirms that the pipeline does not violate any theoretical constraint, but the actual skill depends additionally on the optimization procedure ($\varepsilon_{\rm opt}$) and the training methodology ($\varepsilon_{\rm train}$), which are not yet bounded analytically.

\subsection{Empirical Evidence: The Ten Models as Pareto Surface Samples}

Our ten-model evaluation provides empirical evidence that the Pipeline Pareto Surface exists and that the models trace distinct regions of it.
In the (SFI, PCS) plane at day~5, the projection that captures the strongest trade-off ($\rho = -0.86$), the models arrange along a clear frontier:

(i) Atlas (score matching) occupies the high-SFI, low-PCS corner (SFI $\approx 0.94$, PCS $\approx 0.51$), consistent with its loss function preserving spectral energy at the expense of inter-variable balance.

(ii) Pangu (MSE, pressure-weighted) occupies the low-SFI, high-PCS corner (SFI $\approx 0.66$, PCS $\approx 0.63$), consistent with MSE optimization enforcing smooth, balanced forecasts.

(iii) The remaining MSE-trained models (Aurora, GraphCast, AIFS, FuXi, FengWu, SFNO) cluster in the intermediate region, with their positions determined primarily by loss function weighting details and training methodology rather than architecture.

(iv) The CRPS-trained models (FCN3, AIFS-ENS) achieve high SFI ($> 0.94$) with moderate PCS ($0.47$--$0.54$), occupying a distinct region of the surface that is Pareto-dominated by neither MSE nor score-matching models on all dimensions simultaneously.

Critically, no model dominates all six HMAS dimensions, confirming that the Pipeline Pareto Surface is genuinely multi-dimensional and that trade-offs are inherent to the current state of the art.
The positions of the models on this surface are predictable from the loss taxonomy (Table~\ref{tab:loss_taxonomy}) and data strategy, consistent with the pipeline-dominance thesis.

\subsection{Toward Prescriptive Pipeline Design: Open Problems}

The SFS framework as formulated above bounds $\varepsilon_{\rm loss}$ (via Theorem~\ref{thm:mse_bias}) and $\varepsilon_{\rm data}$ (via Proposition~\ref{thm:ood_bias}) but leaves two components of the pipeline decomposition (Proposition~\ref{prop:dominance}) without tight pre-training bounds:

\paragraph{Open Problem 1: Quantifying $\varepsilon_{\rm train}$.}
The training methodology term (encompassing autoregressive rollout strategy, timestep, learning rate schedule, and fine-tuning protocol) is currently characterized only empirically.
A pre-training bound would require analyzing the error accumulation in $K$-step autoregressive rollout as a function of the single-step error and the spectral properties of the dynamics.
Our ASI diagnostic (Fig.~\ref{fig:error_growth}) and the catastrophic energy drift observed in FengWu and FuXi suggest that autoregressive stability is predictable from the loss function's spectral properties and the rollout length, but the formal connection remains to be established.

\paragraph{Open Problem 2: Tightening the data coverage bound.}
Proposition~\ref{thm:ood_bias} provides a tail-specific bound, but a general bound connecting training data distributional coverage (including the diversity of data sources, temporal coverage, and representation of rare flow regimes) to forecast skill across the full atmospheric distribution requires characterizing the divergence between $\mathcal{D}$ and the operational deployment distribution in a metric compatible with forecast error.
Wasserstein distance in a reduced state space (e.g., the first few principal components of the atmospheric state) is a natural candidate.

\paragraph{Open Problem 3: Physics-constrained multi-objective loss design.}
The SFI--PCS anti-correlation ($\rho = -0.86$) suggests that no single existing loss function achieves both spectral fidelity and physical consistency.
A natural next step is \emph{multi-objective loss design} that explicitly includes balance constraints (geostrophic, thermal wind, hydrostatic) as regularization terms alongside the primary loss.
The PCS framework (\S\ref{sec:physical_consistency}) provides the diagnostic targets; the open problem is to formulate differentiable loss terms that enforce these constraints without sacrificing spectral energy.

Closing these gaps would transform the SFS from a feasibility check into a quantitative predictor of forecast skill, enabling the atmospheric science community to navigate the pipeline design space systematically rather than through expensive trial and error.

\paragraph{Open Problem 4: Initial condition sensitivity.}
Our evaluation uses ERA5 initial conditions exclusively, but operational deployment involves initial conditions from diverse sources (operational analyses, ensemble perturbations, reduced-resolution analyses) that may interact differently with each model's learned dynamics.
Sensitivity to initial condition perturbations is a practically important diagnostic that our framework does not yet address: a model with low RMSE under ERA5 initialization may degrade disproportionately under perturbed or degraded initial conditions if its training has overfit to ERA5-specific features.
Preliminary work in our laboratory suggests that autoregressive fine-tuning (as in Aurora's LoRA approach) can alter IC sensitivity, and that ensemble spread behavior under perturbed ICs varies across architectures.
Formalizing IC sensitivity as a seventh HMAS dimension, or as a separate robustness diagnostic, is a natural extension.

\section{Experimental Design}
\label{sec:experimental}

All empirical results are produced through systematic inference using NVIDIA Earth2Studio \citep{e2studio2024} (v0.12+).

\paragraph{Models.}
Ten models spanning five architecture families: AIFS, AIFS-ENS, GraphCast (GNNs); Atlas (diffusion transformer); Aurora, FengWu \citep{chen2023fengwu} (vision transformers); FCN3 \citep{bonev2025fcn3}, SFNO (spherical Fourier operators); FuXi (U-Transformer); Pangu-Weather (3D Earth Transformer).
All use pretrained weights from Earth2Studio's model registry.

\paragraph{Note on diffusion model representation.}
Atlas is the sole diffusion-based model in our evaluation.
While this limits our ability to assess the generality of diffusion model behavior (e.g., whether the spectral energy preservation and PCS degradation patterns are inherent to diffusion or specific to Atlas's implementation), it reflects the current availability of pretrained diffusion weather models in Earth2Studio.
GenCast \citep{price2023gencast} was not available in the Earth2Studio registry at the time of this study.
Conclusions about diffusion models should therefore be interpreted as Atlas-specific until validated across additional diffusion architectures.

\paragraph{Initial conditions.}
ERA5 reanalysis at 0.25$^\circ$ resolution via CDS-Beta API.
Thirty initialization dates spanning all four seasons: 8 DJF (including the 2021 Texas freeze and 2022 Winter Storm Elliott), 7 MAM, 8 JJA (including the 2021 PNW heatwave and 2023 European heatwave), and 7 SON.
This sampling captures diverse flow regimes (blocked vs.\ zonal, ENSO phases) and spans 2021--2024.
All initialization dates postdate the training period cutoff for all ten models (the latest training cutoff is approximately 2021 for Aurora; most other models use ERA5 through 2017--2019), ensuring out-of-sample verification.
Four extreme events are embedded within the 30-date set rather than treated separately, so they contribute to all standard metrics as well as the dedicated tail-fidelity analysis.

\paragraph{Verification data.}
ERA5 analysis fields for Z500, T2M, U500, V500, T850, Q700.
WeatherBench2-compatible \citep{rasp2024weatherbench2} ERA5 climatology (1990--2020, daily $\times$ 6-hourly resolution) for anomaly correlation and extreme event threshold computation (15-day DOY window, 0.5\,K floor on $\sigma_{\rm clim}$).

\paragraph{Statistical methodology.}
All metrics are averaged over the 30 initialization dates, with 90\% confidence intervals computed as inter-initialization-date standard errors: $\bar{x} \pm 1.645 \cdot s/\sqrt{n}$ (\S\ref{subsec:bootstrap}).

\paragraph{Native stochastic models.}
AIFS-ENS (afCRPS-trained) and FCN3 (spatial+spectral CRPS-trained) are natively stochastic: they generate ensemble spread through learned noise injection during inference.
We run both with zero IC perturbation ($\delta\bm{u}_0 = 0$) and evaluate a single realization, as applying IC perturbation would double-source uncertainty and break calibration.
All metrics therefore evaluate their single-member deterministic behavior, not their ensemble skill.

\section{Limitations and Open Questions}
\label{sec:limitations}

\begin{enumerate}[leftmargin=*,itemsep=2pt]
\item \textbf{Linear approximations.}
The spectral transfer framework (\S\ref{sec:spectral}) linearizes the dynamics, missing nonlinear scale interactions.
Proposition~\ref{thm:ood_bias} uses a first-order Taylor expansion.
Both are accurate for small perturbations but may underestimate errors for strongly nonlinear events (e.g., rapid cyclogenesis).

\item \textbf{Spherical harmonic vs.\ FFT mismatch.}
As discussed in \S\ref{subsec:sh_fft}, our theoretical framework uses spherical harmonics while the empirical computation uses latitude-weighted FFT.
The correspondence is approximate ($\mathcal{O}(\ell^{-1})$ correction) and most accurate for high wavenumbers and spectral ratios.
Future work should implement full SHT-based diagnostics for exact validation.

\item \textbf{Single diffusion model.}
Atlas is the sole diffusion model in our evaluation.
Diffusion-specific findings (spectral energy preservation, PCS degradation) should be verified with additional diffusion weather models (e.g., GenCast) when available.

\item \textbf{Non-stationarity.}
The Lyapunov exponents and KS entropy are defined for a stationary attractor.
Under climate change, the attractor shifts and the framework needs extension.
\end{enumerate}

\section{Conclusions}
\label{sec:conclusions}

We have constructed a holistic mathematical framework for understanding AI weather prediction skill, grounded in approximation theory on the sphere \citep{freeden1998constructive}, dynamical systems theory \citep{katok1995modern}, information theory \citep{cover2006elements}, and statistical learning theory \citep{anthony1999neural}, and validated it empirically through systematic inference across ten architecturally diverse AI weather models.

The theoretical contributions include: (i)~a scale-resolved formalization of MSE-induced spectral bias in spherical harmonic coordinates (Theorem~\ref{thm:mse_bias}), explaining the universal blurring observed across all MSE-trained architectures; (ii)~a learning pipeline error decomposition (Proposition~\ref{prop:dominance}) with inter-initialization-date uncertainty quantification across 30 dates (\S\ref{subsec:bootstrap}); (iii)~out-of-distribution bounds (Proposition~\ref{thm:ood_bias}) predicting linear bias growth with record exceedance; and (iv)~an Error Consensus Ratio formalization (Proposition~\ref{prop:error_consensus}) for testing whether forecast errors are predictability-limited.

The empirical contributions confirm all major theoretical predictions:
spectral energy loss at high wavenumbers across all MSE-trained architectures (Fig.~\ref{fig:spectral_comparison});
ECR reaching $\approx 0.60$ at day~5 and remaining above $0.58$ at day~15 (Fig.~\ref{fig:error_consensus});
linear tail attenuation (Fig.~\ref{fig:tail_fidelity});
rank instability in the scorecard (Fig.~\ref{fig:scorecard}).
Additional empirical contributions include: (v)~a four-component Physical Consistency Score (Fig.~\ref{fig:physical_consistency}) revealing the spectral fidelity--physical consistency trade-off; (vi)~latitude-resolved RMSE decomposition (Fig.~\ref{fig:regional_rmse}) revealing polar error amplification; (vii)~HMAS weight sensitivity analysis with Kendall's $W = 0.97$ (Fig.~\ref{fig:hmas_sensitivity}); and (viii)~HMAS cross-correlation matrix confirming metric independence (Fig.~\ref{fig:hmas_correlation}).

The central message is that \textbf{the learning pipeline (loss function, data strategy, and training methodology) determines AI weather forecast skill at least as much as the choice of neural network architecture}.
Architecture remains important for specific capabilities, but at current operational scales, the primary path to improved forecasts runs through the loss function and data.

Beyond diagnosis, the framework takes a first step toward \emph{prescriptive} pipeline design (\S\ref{sec:prescriptive}).
By extending the Pareto-optimal model hierarchies of \citet{beucler2025pareto} to a multi-objective Pipeline Pareto Surface in HMAS metric space, and combining pre-computable atmospheric quantities (the conditional spectral variance $\mathrm{Var}_\ell(\tau)$ from Theorem~\ref{thm:mse_bias}, the OOD attenuation bounds from Proposition~\ref{thm:ood_bias}, and the information-theoretic predictability ceiling from Theorem~\ref{thm:info_decay}) into a Spectral Feasibility Score, we provide a mathematical basis for evaluating proposed pipelines before committing to training.
The remaining open problems, quantitative bounds on $\varepsilon_{\rm train}$ and a general data coverage metric, represent natural next steps toward a complete pre-training suitability theory for AI atmospheric prediction.

\paragraph{Reproducibility.}
All inference was performed using NVIDIA Earth2Studio (v0.12.1) between January and March 2026, with publicly available pretrained model weights and ERA5 initial conditions from the CDS-Beta API (accessed on February $20^{th}$, 2026).
Model weight versions correspond to the Earth2Studio model registry as of January 2026.
The analysis code computes all metrics and generates all figures from the forecast NetCDF files and will be available at the time of journal publication. 

\paragraph{Acknowledgments.}
We thank the NVIDIA Earth2Studio team for the inference framework, ECMWF for ERA5 data access, and the developers of all ten model architectures for making pretrained weights publicly available. We also thank the RWE Commercial AI Lab and RWE Supply and Trading for providing computational resources and institutional support for this analysis.

\paragraph{Author contributions.}
P.G.\ conceived the study, designed the mathematical framework, developed the analysis code, performed all model inference and empirical evaluation, and wrote the manuscript.
D.R.G., A.E.S., and G.J.Y.\ reviewed and edited the manuscript, provided critical feedback on the theoretical framing and empirical methodology, and contributed to the interpretation of results.


\FloatBarrier


\clearpage
\appendix
\section*{Supplementary Material}
\label{sec:supplementary}

The following supplementary figures complement the main text with additional variables, events, and forecast horizons.

\begin{table*}[p]
\centering
\caption{[Supplementary] Training datasets and time periods for the ten AI weather models evaluated in this study.
All models use ERA5 reanalysis as the primary training dataset; several incorporate additional data sources for pretraining or fine-tuning.
The ``Test cutoff'' column indicates the earliest year that can be used for out-of-sample evaluation.
Our 30 initialization dates (2021--2024) postdate all training cutoffs.}
\label{tab:training_data}
\scriptsize
\setlength{\tabcolsep}{3pt}
\resizebox{\textwidth}{!}{%
\begin{tabular}{@{}llllll@{}}
\toprule
\textbf{Model} & \textbf{Primary data} & \textbf{Training period} & \textbf{Additional data} & \textbf{Fine-tuning} & \textbf{Test cutoff} \\
\midrule
AIFS & ERA5 & $\sim$1979--2018 & --- & ECMWF oper.\ analyses (2019--2020) & 2021 \\
AIFS-ENS & ERA5 & $\sim$1979--2018 & --- & ECMWF oper.\ analyses (2019--2020) & 2021 \\
Atlas & ERA5 & 1980--2019 & --- & --- & 2020 \\
Aurora & ERA5 & 1979--2020 & GFS analyses, CMIP6, IFS-HR, HRES & LoRA rollout fine-tune (HRES-T0, 2016--2021) & 2022 \\
FCN3 & ERA5 & 1980--2016 & --- & --- & 2020 \\
FengWu & ERA5 & 1979--2017 & --- & --- & 2018 \\
FuXi & ERA5 & 1979--2015 & --- & --- & 2018 \\
GraphCast & ERA5 & 1979--2017 & --- & --- & 2018 \\
Pangu & ERA5 & 1979--2017 & --- & --- & 2018 \\
SFNO & ERA5 & 1979--2015 & --- & --- & 2018 \\
\bottomrule
\end{tabular}%
}
\end{table*}

\begin{figure*}[p]
\centering
\includegraphics[width=\textwidth]{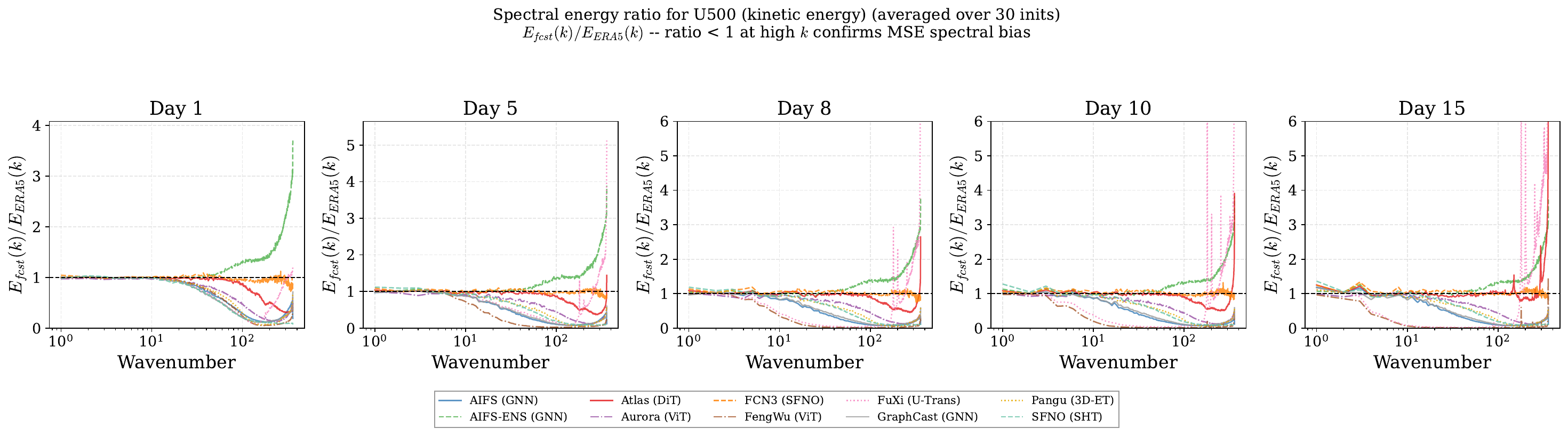}
\caption{[Supplementary] Spectral energy ratio for U500 at day~1 through day~15.
Same qualitative pattern as Z500: universal deficit at high $k$ for MSE-trained models.}
\label{fig:spectral_ratio_u500}
\end{figure*}

\begin{figure*}[p]
\centering
\includegraphics[width=\textwidth]{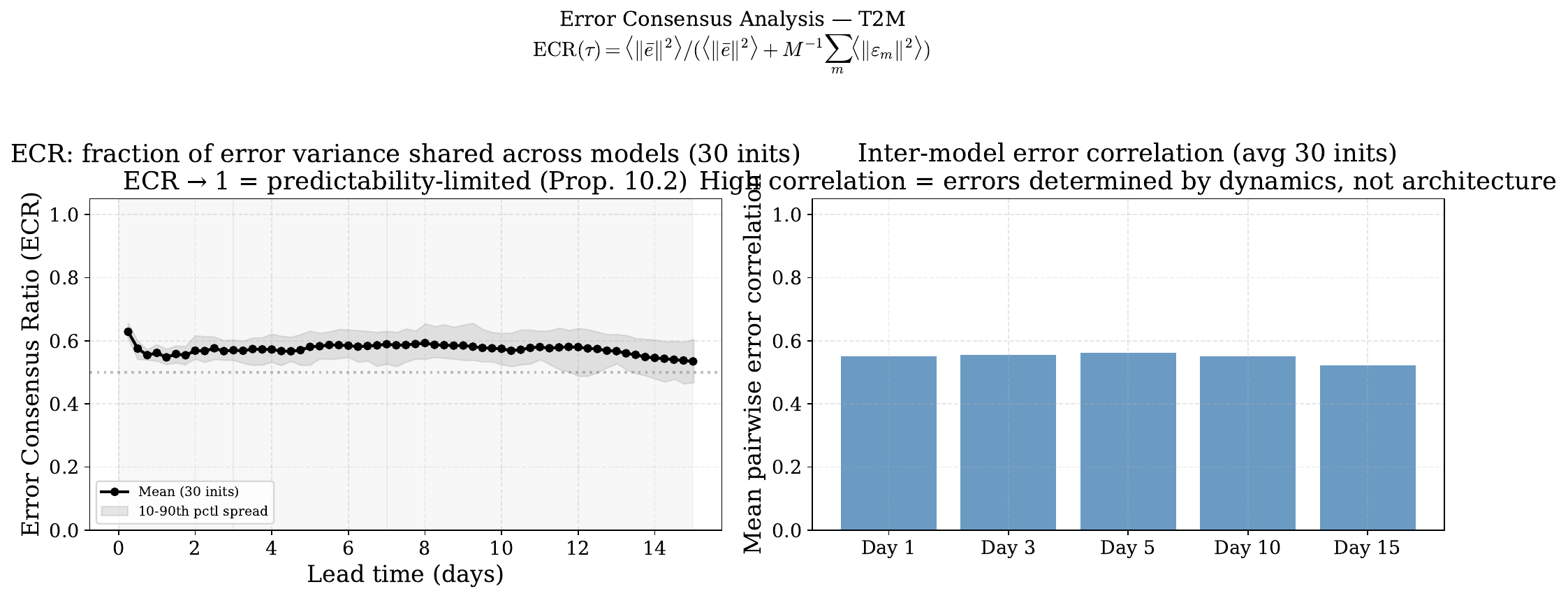}
\caption{[Supplementary] Error Consensus Analysis for T2M.
ECR starts at $\approx 0.63$ at short range and remains above $0.53$ through day~15.}
\label{fig:ecr_t2m}
\end{figure*}

\begin{figure*}[p]
\centering
\includegraphics[width=\textwidth]{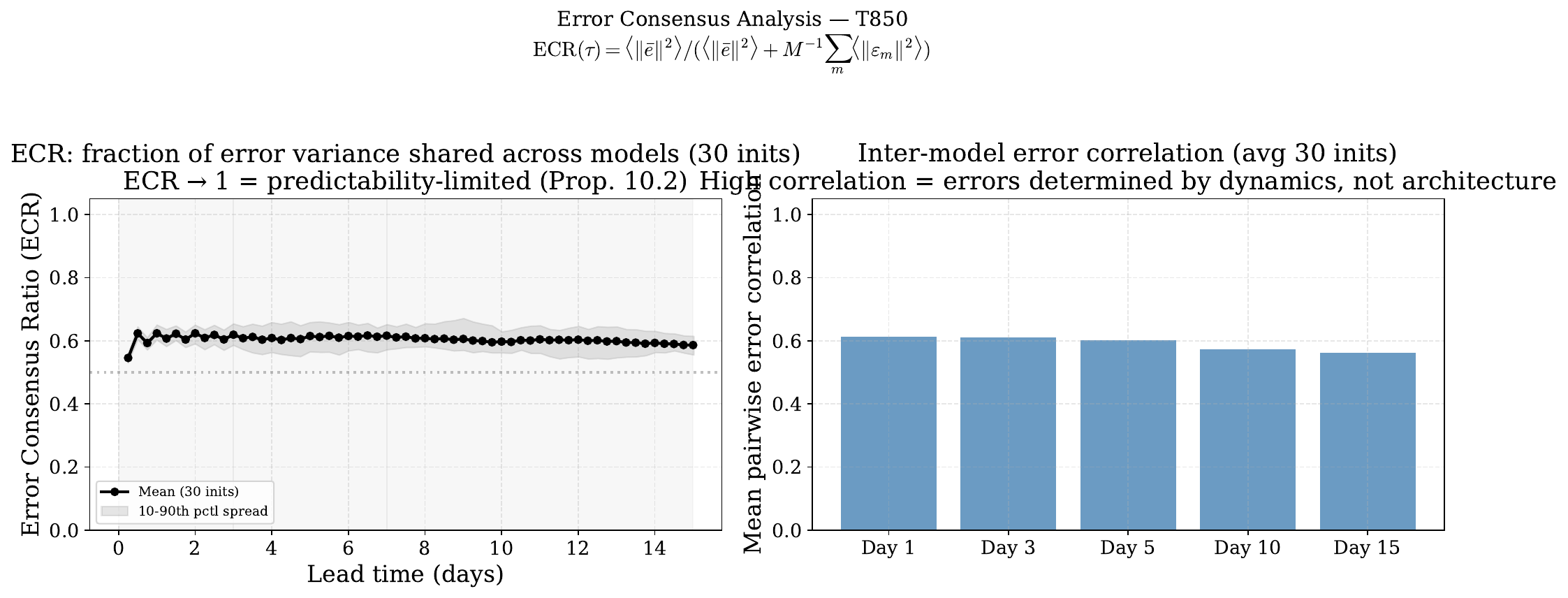}
\caption{[Supplementary] Error Consensus Analysis for T850.}
\label{fig:ecr_t850}
\end{figure*}

\begin{figure*}[p]
\centering
\includegraphics[width=\textwidth]{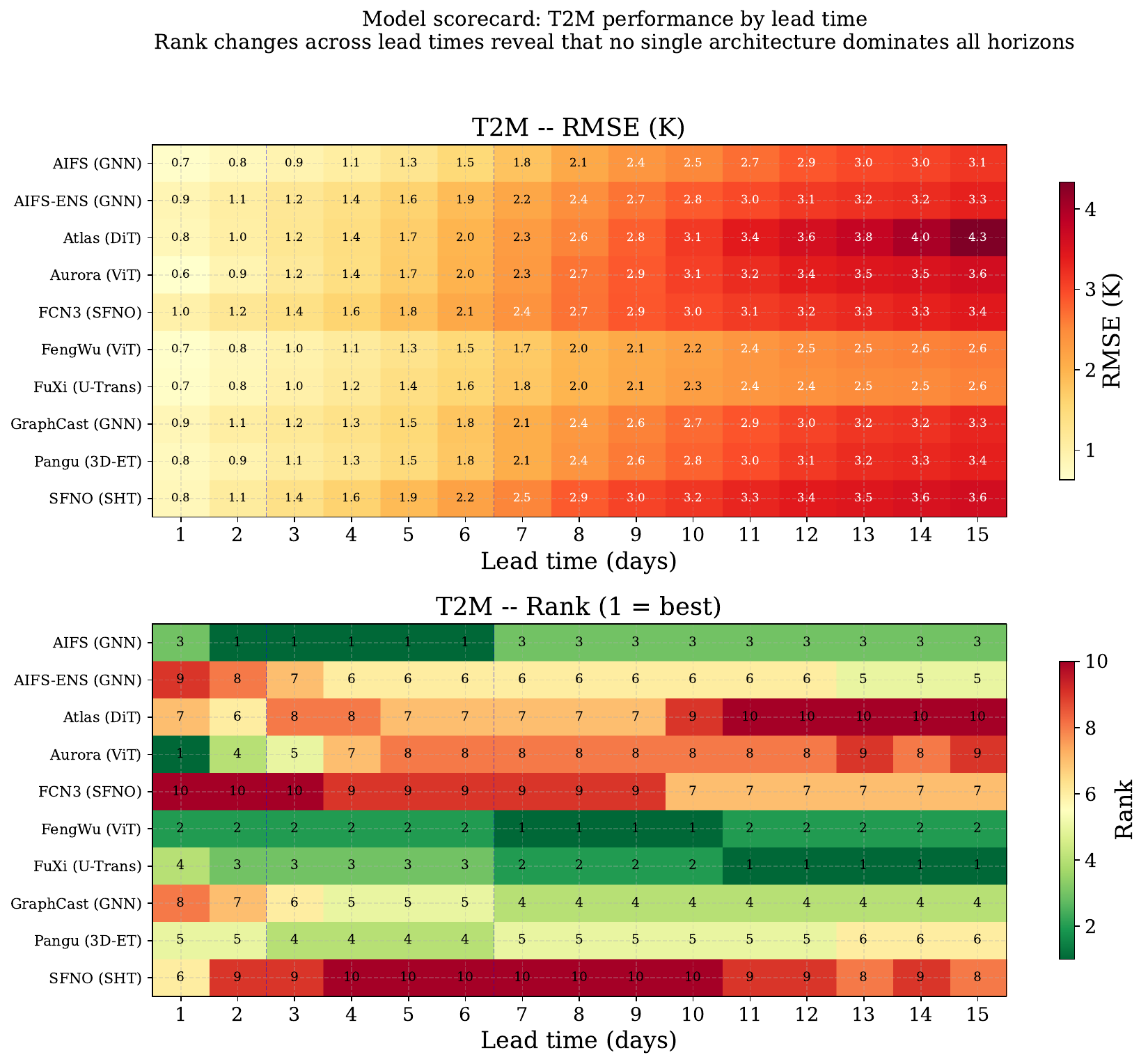}
\caption{[Supplementary] Model scorecard for T2M.}
\label{fig:scorecard_t2m}
\end{figure*}

\begin{figure*}[p]
\centering
\includegraphics[width=\textwidth]{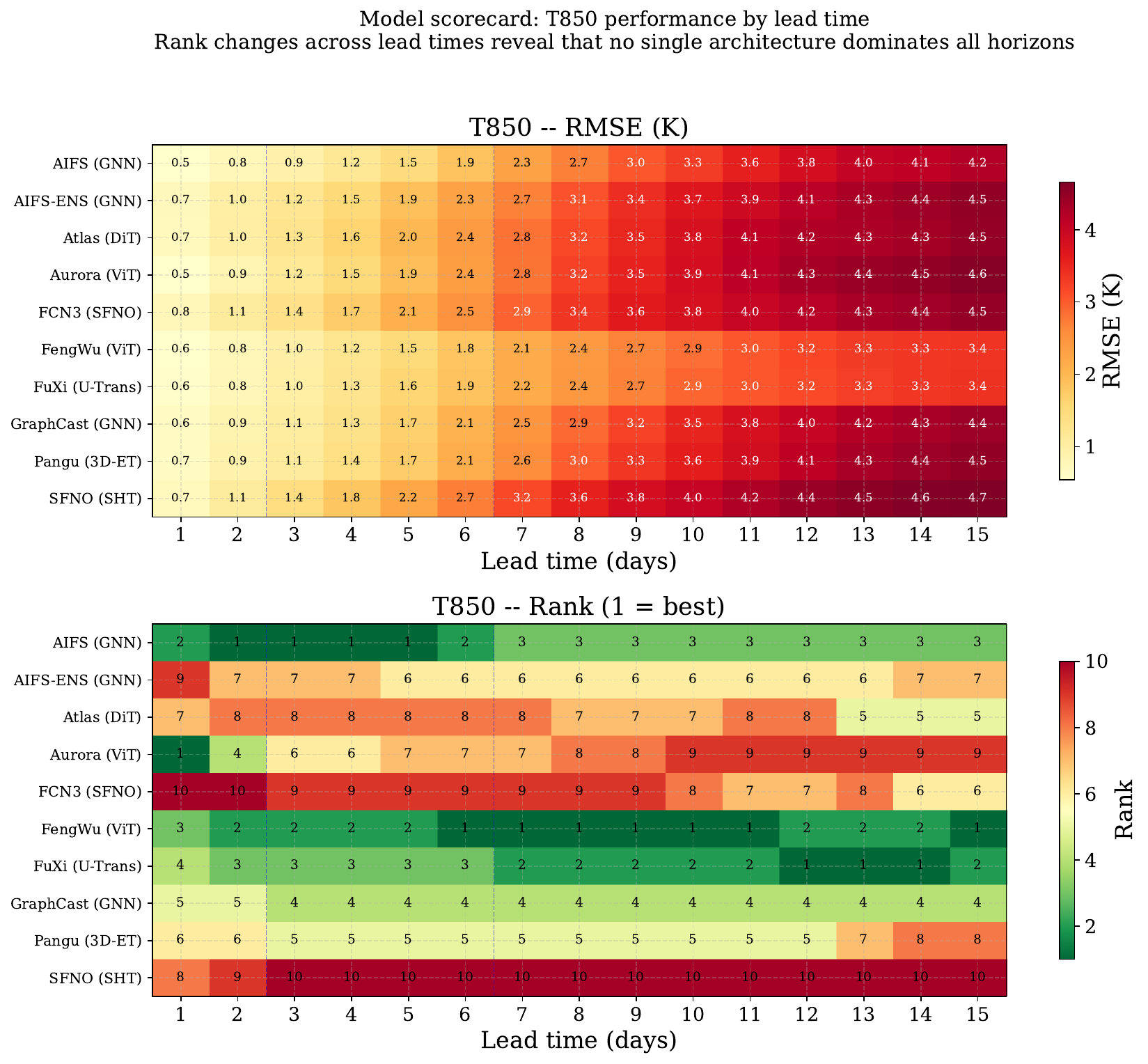}
\caption{[Supplementary] Model scorecard for T850.}
\label{fig:scorecard_t850}
\end{figure*}

\begin{figure*}[p]
\centering
\includegraphics[width=\textwidth]{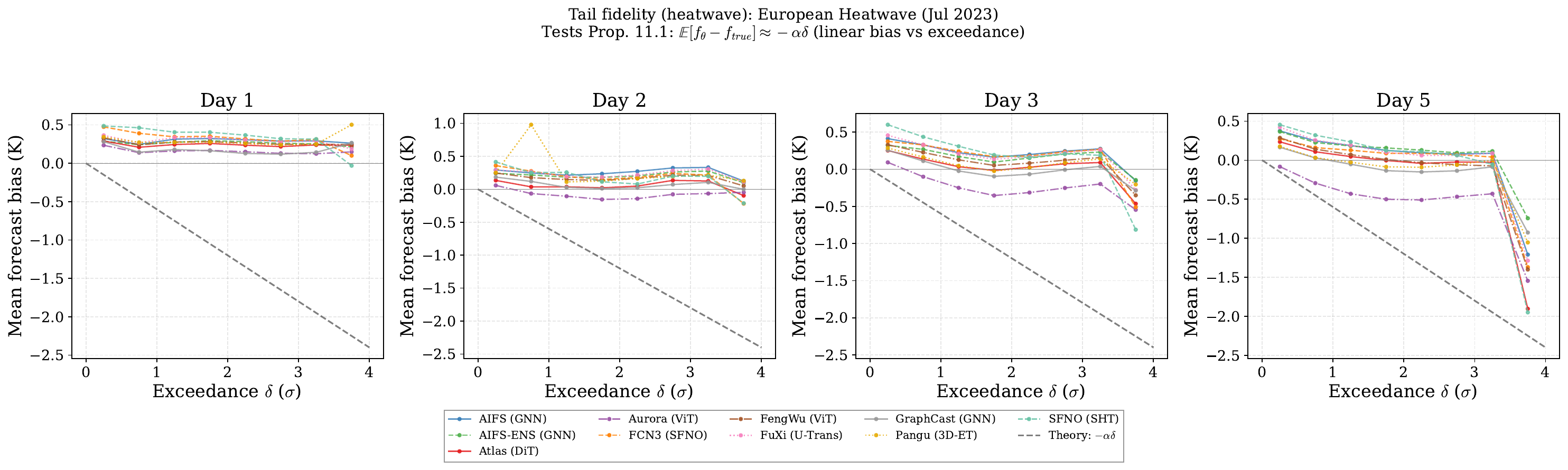}
\caption{[Supplementary] Tail fidelity for the 2023 European heatwave.}
\label{fig:tail_european}
\end{figure*}

\begin{figure*}[p]
\centering
\includegraphics[width=\textwidth]{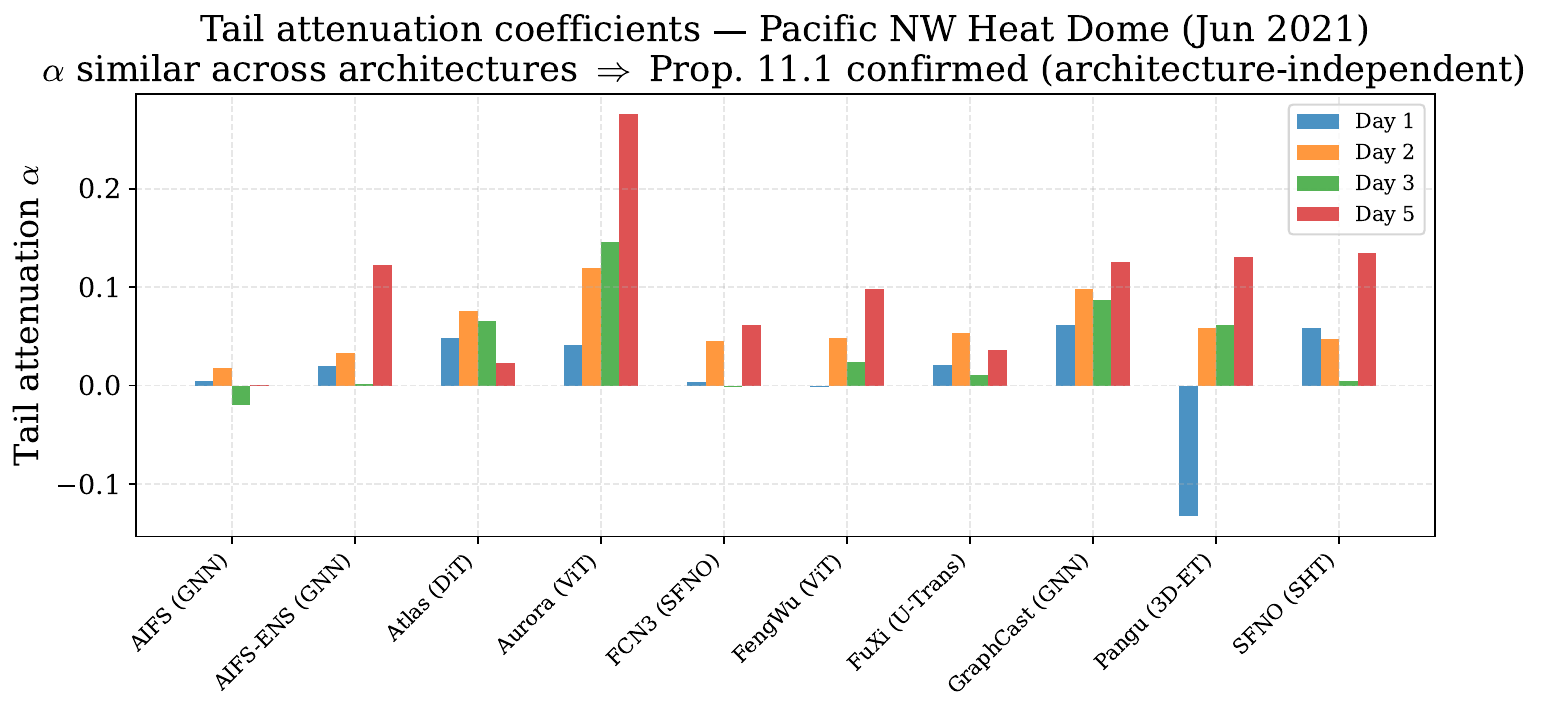}
\caption{[Supplementary] Tail attenuation coefficients $\alpha$ for the PNW heatwave.}
\label{fig:alpha_pnw}
\end{figure*}

\begin{figure*}[p]
\centering
\includegraphics[width=\textwidth]{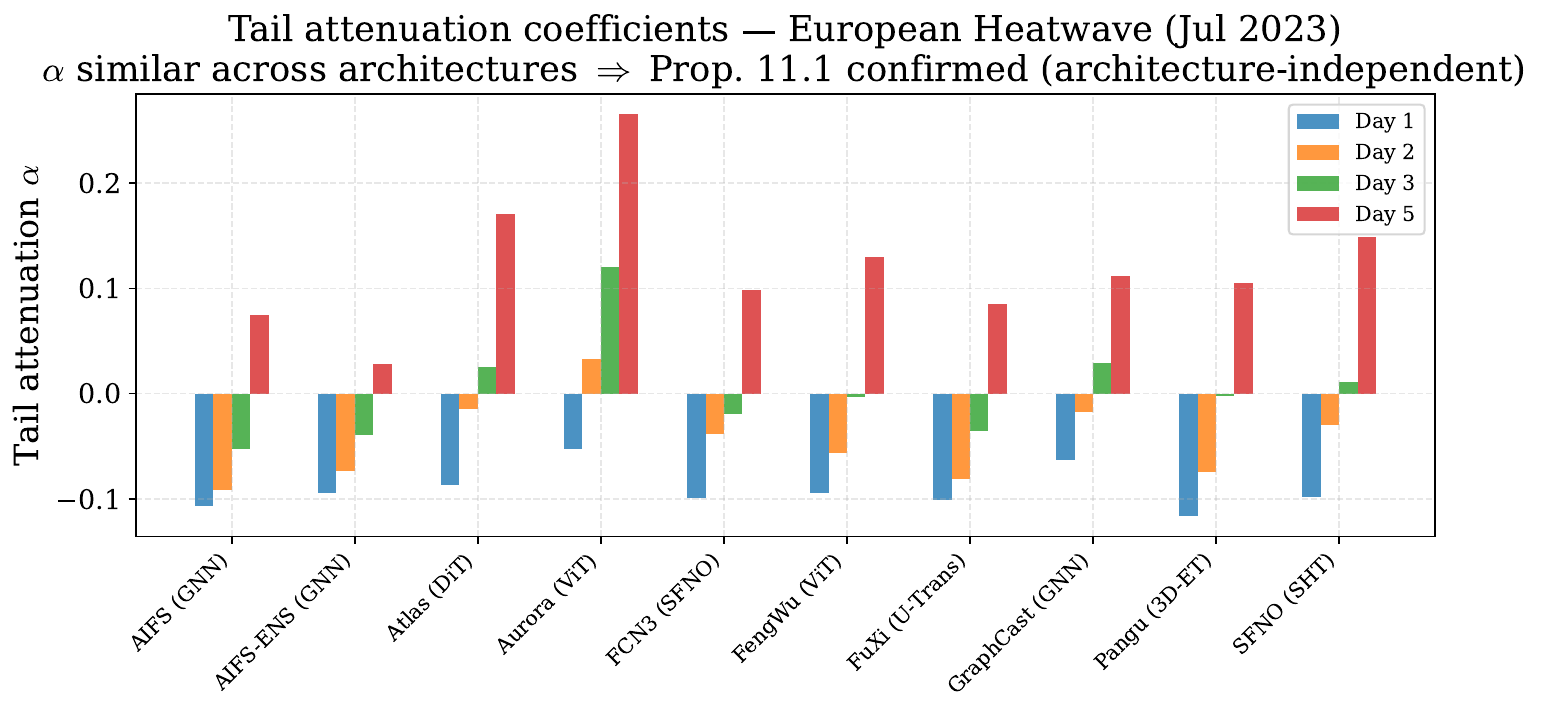}
\caption{[Supplementary] Tail attenuation coefficients $\alpha$ for the 2023 European heatwave.}
\label{fig:alpha_european}
\end{figure*}

\begin{figure*}[p]
\centering
\includegraphics[width=\textwidth]{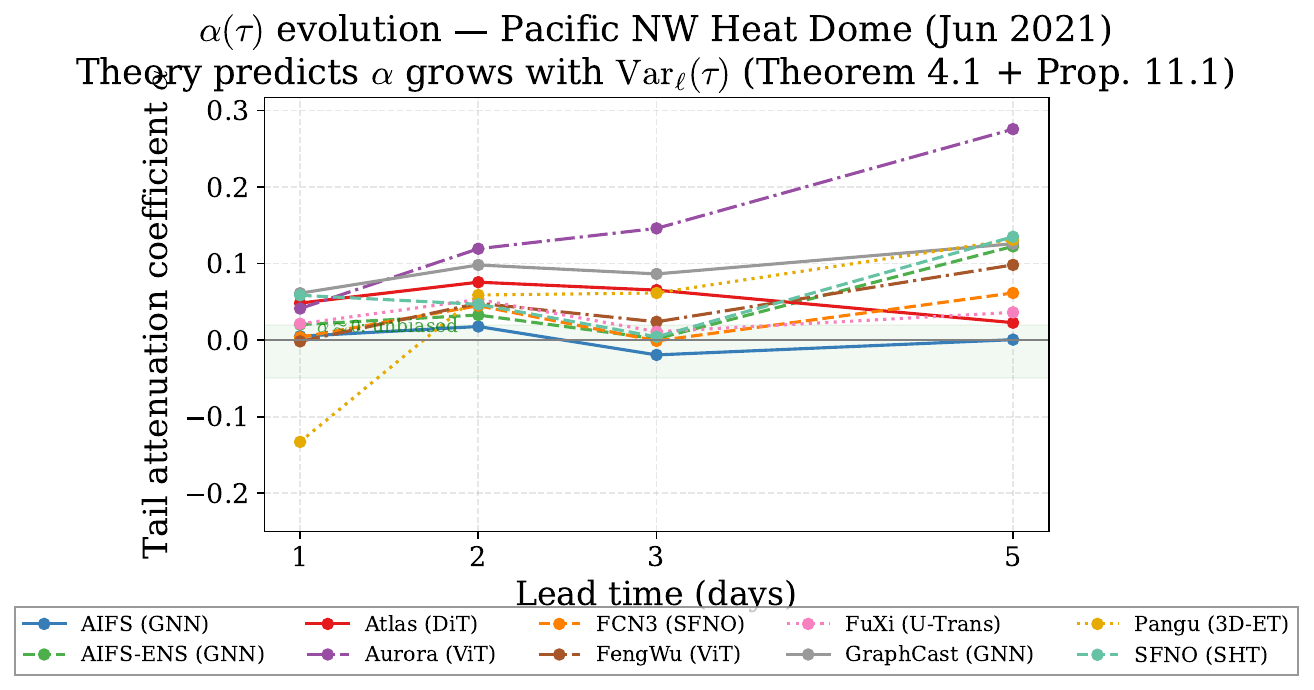}
\caption{[Supplementary] $\alpha$ coefficient evolution with lead time for the PNW heatwave, with $R^2$ values quantifying the goodness of fit of the linear bias--exceedance relationship.}
\label{fig:alpha_evolution_pnw}
\end{figure*}

\begin{figure*}[p]
\centering
\includegraphics[width=\textwidth]{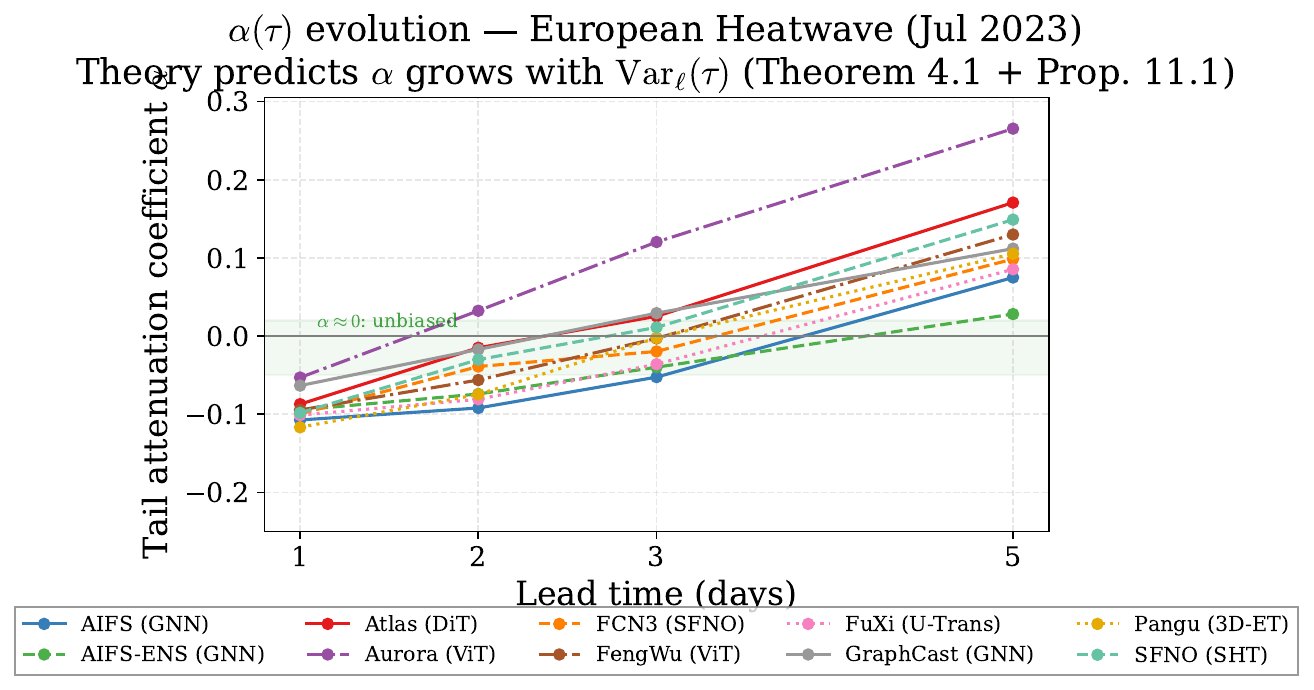}
\caption{[Supplementary] $\alpha$ coefficient evolution for the 2023 European heatwave.}
\label{fig:alpha_evolution_european}
\end{figure*}

\begin{figure*}[p]
\centering
\includegraphics[width=\textwidth]{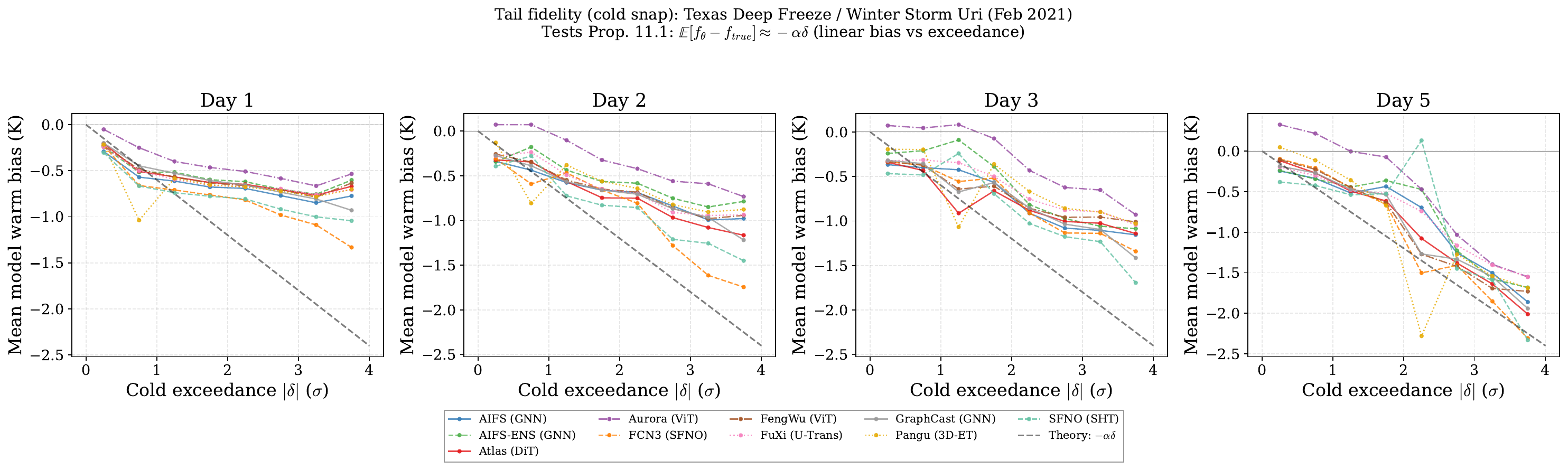}
\caption{[Supplementary] Tail fidelity for the February 2021 Texas freeze.
Cold extreme attenuation coefficients $\alpha \approx 0.34$--$0.56$ at day~5 are substantially larger than for heat events, indicating stronger OOD bias for cold extremes.}
\label{fig:tail_texas}
\end{figure*}

\begin{figure*}[p]
\centering
\includegraphics[width=\textwidth]{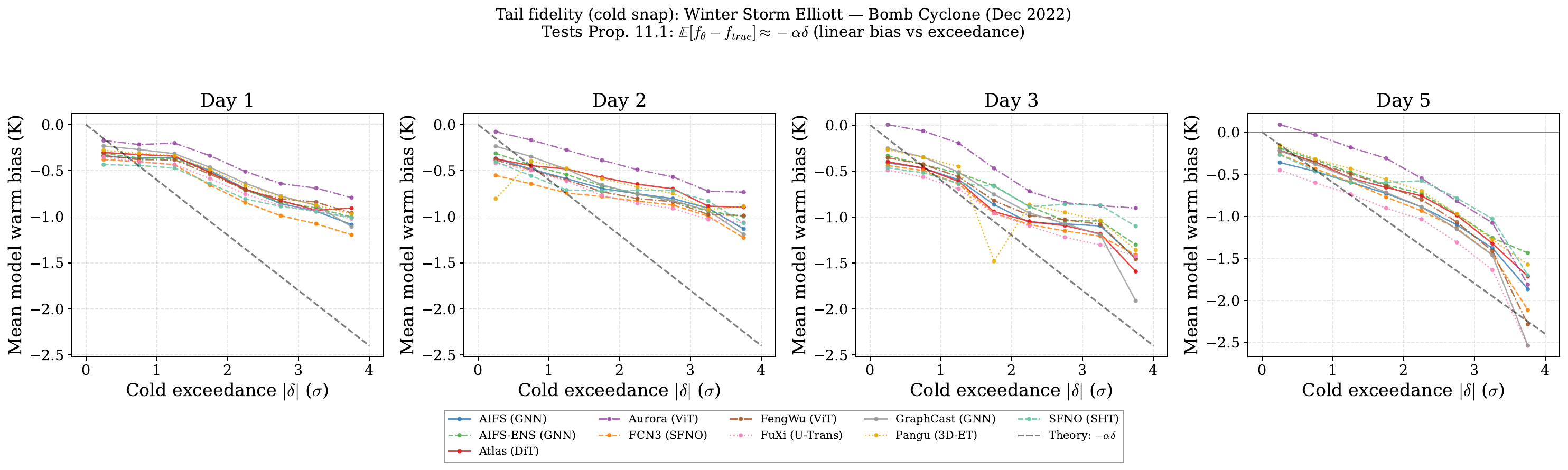}
\caption{[Supplementary] Tail fidelity for Winter Storm Elliott (December 2022).
Similar cold-extreme attenuation pattern to the Texas freeze, with $\alpha \approx 0.34$--$0.56$ at day~5.}
\label{fig:tail_elliott}
\end{figure*}

\begin{figure*}[p]
\centering
\includegraphics[width=\textwidth]{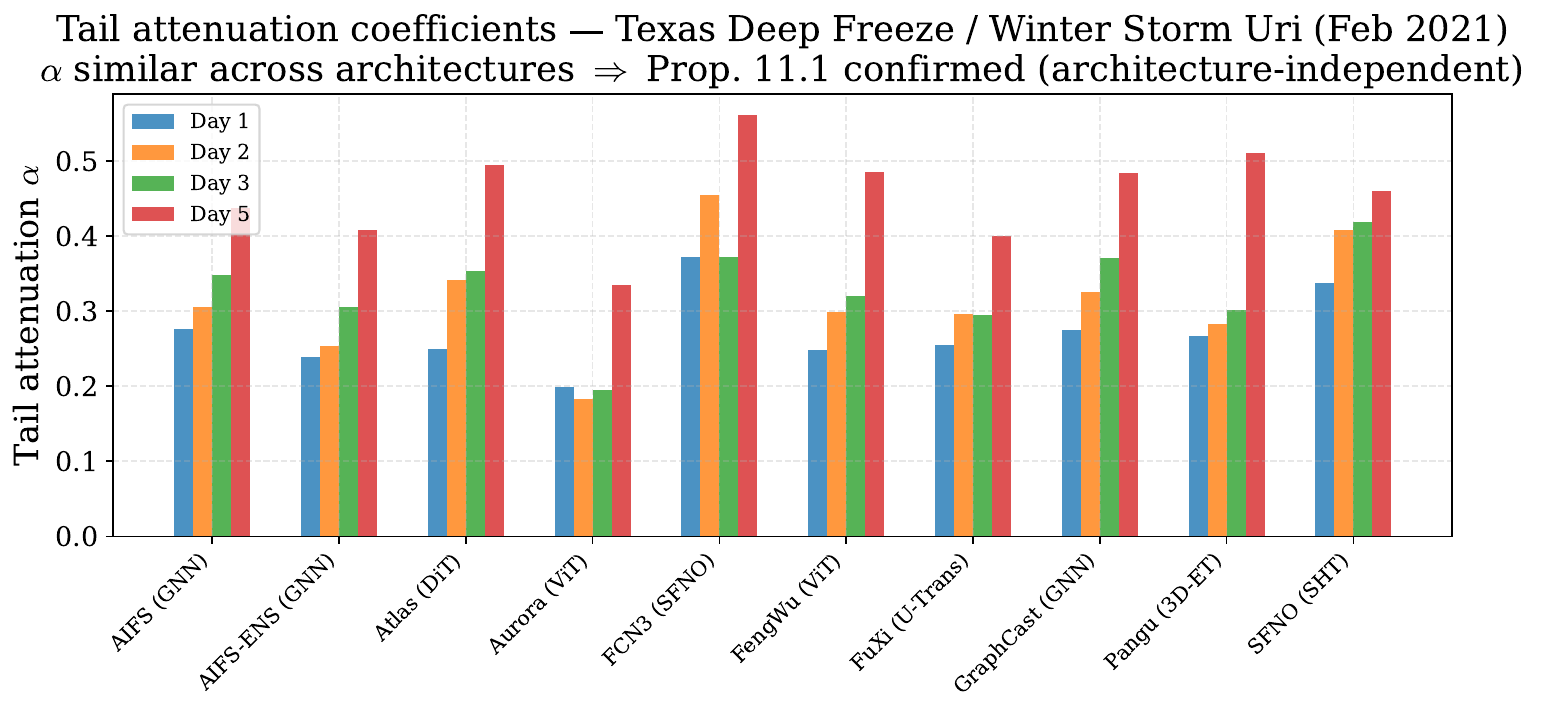}
\caption{[Supplementary] Tail attenuation coefficients $\alpha$ for the February 2021 Texas freeze.}
\label{fig:alpha_texas}
\end{figure*}

\begin{figure*}[p]
\centering
\includegraphics[width=\textwidth]{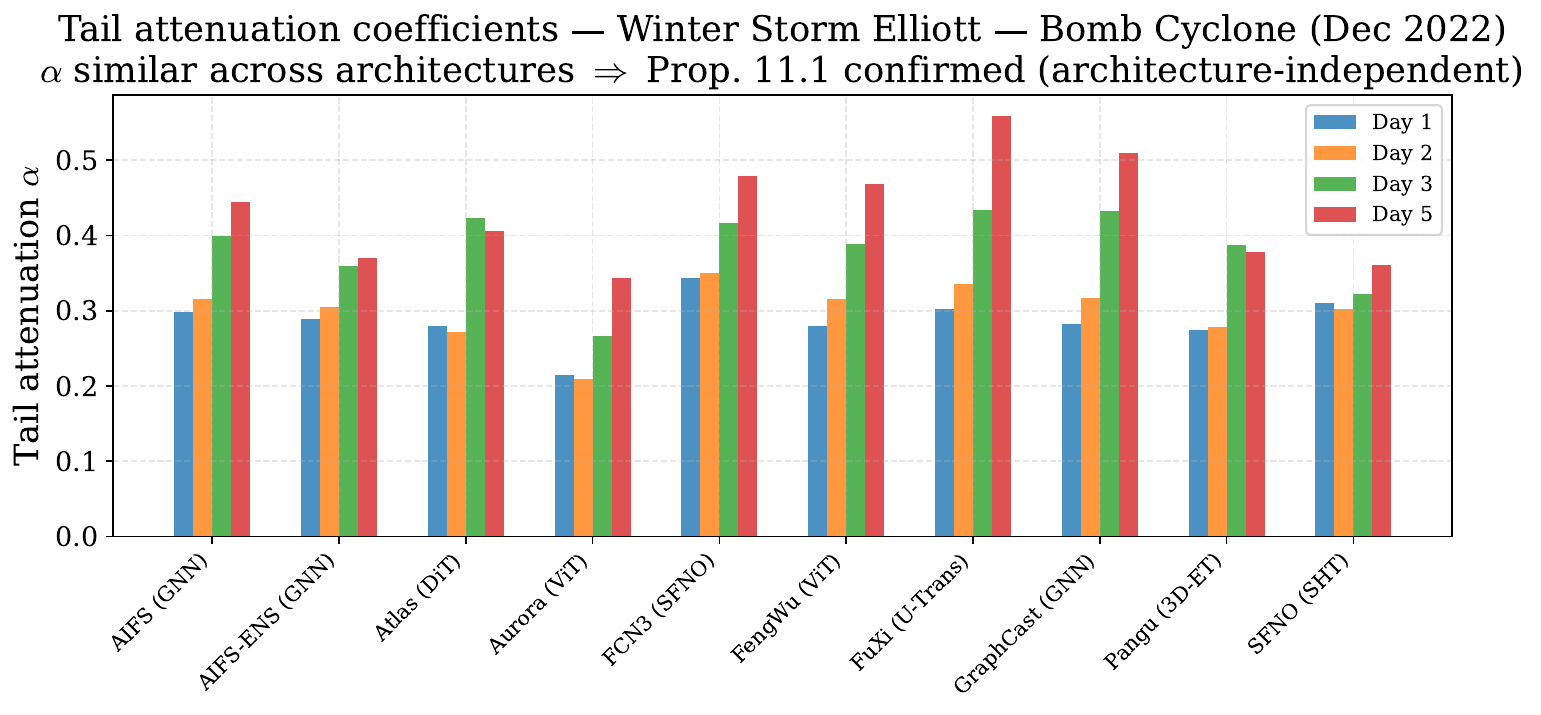}
\caption{[Supplementary] Tail attenuation coefficients $\alpha$ for Winter Storm Elliott (December 2022).}
\label{fig:alpha_elliott}
\end{figure*}

\begin{figure*}[p]
\centering
\includegraphics[width=\textwidth]{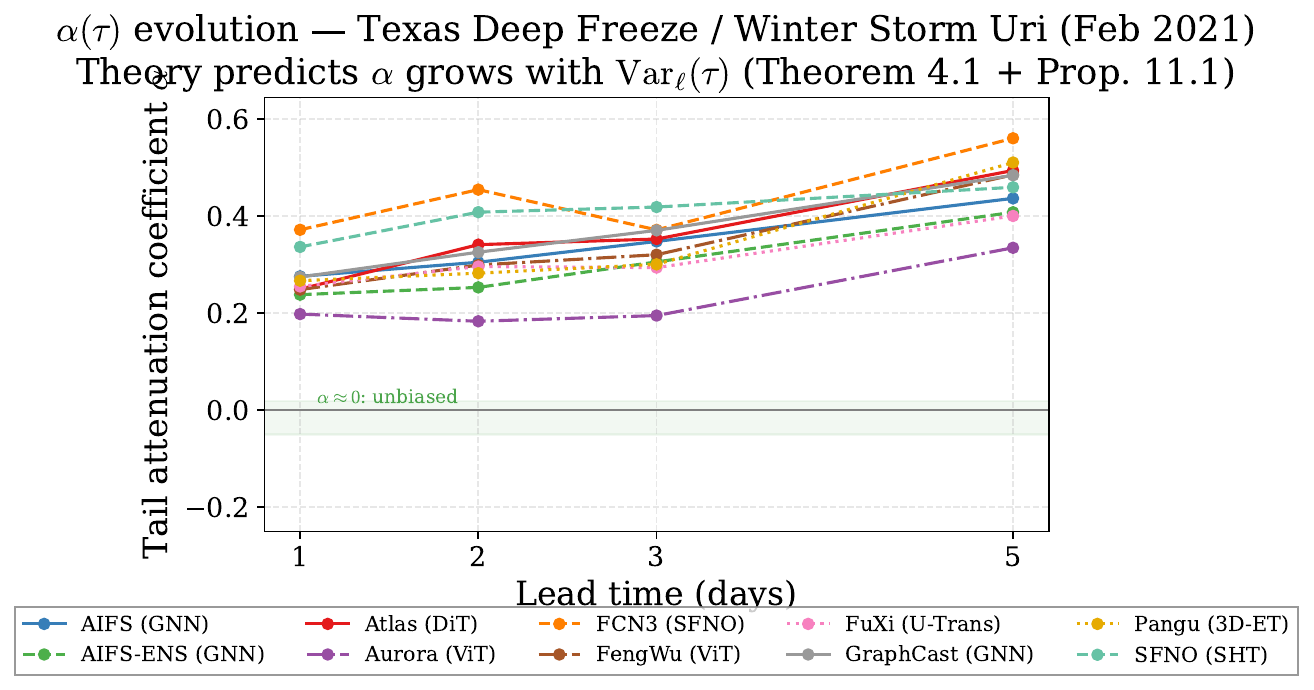}
\caption{[Supplementary] $\alpha$ coefficient evolution with lead time for the February 2021 Texas freeze.}
\label{fig:alpha_evolution_texas}
\end{figure*}

\begin{figure*}[p]
\centering
\includegraphics[width=\textwidth]{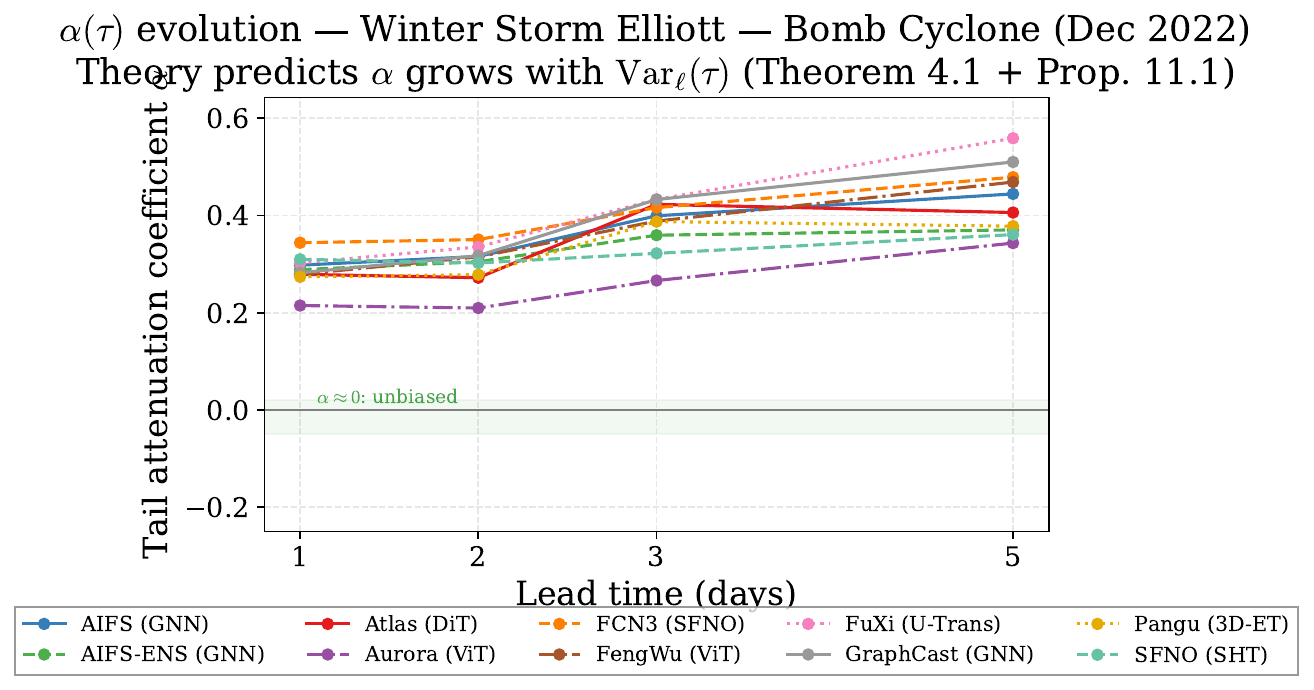}
\caption{[Supplementary] $\alpha$ coefficient evolution with lead time for Winter Storm Elliott (December 2022).}
\label{fig:alpha_evolution_elliott}
\end{figure*}

\begin{figure*}[p]
\centering
\includegraphics[width=0.75\textwidth]{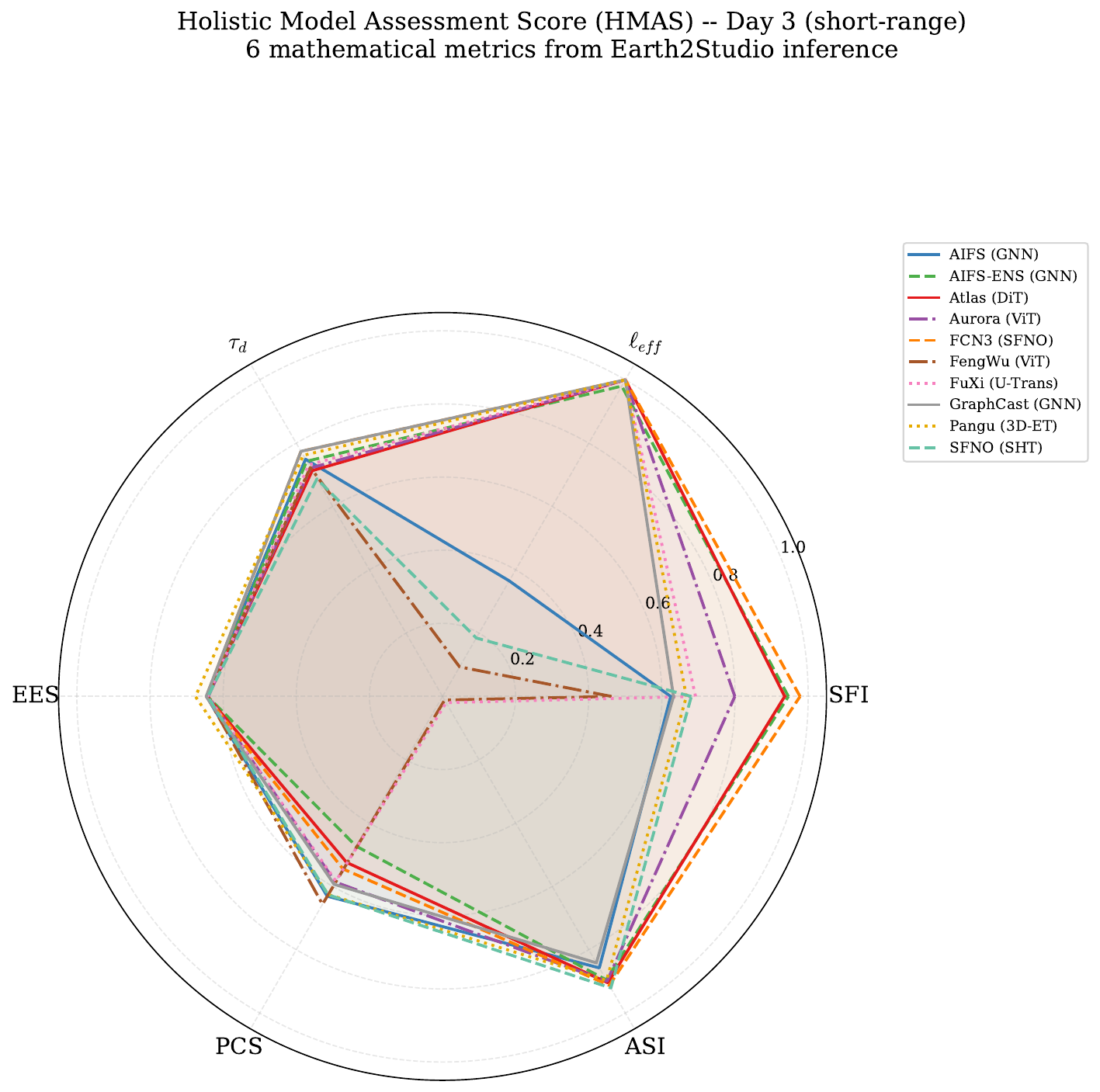}
\caption{[Supplementary] HMAS radar chart at day~3 (short-range).}
\label{fig:radar_day3}
\end{figure*}

\begin{figure*}[p]
\centering
\includegraphics[width=0.75\textwidth]{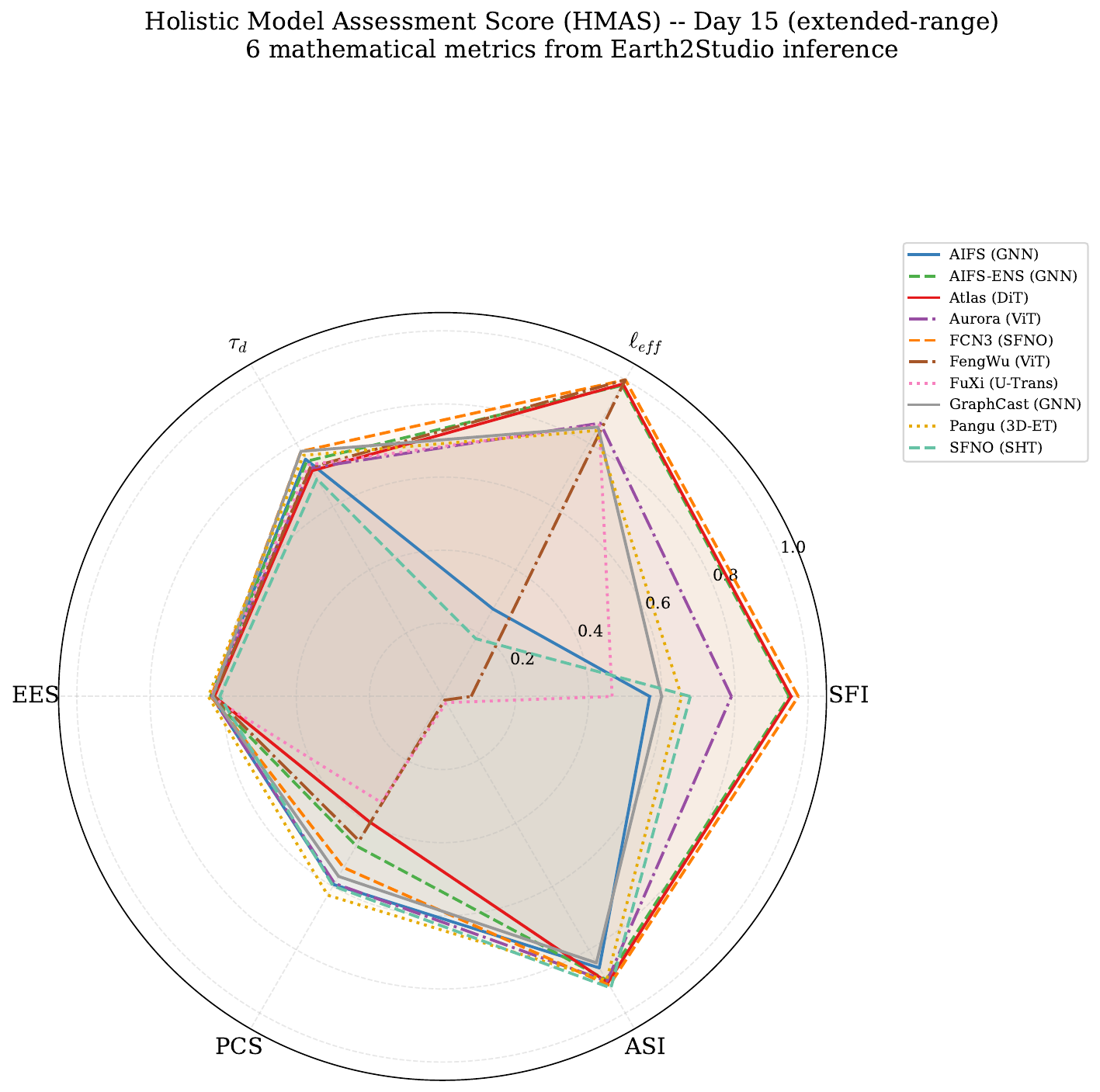}
\caption{[Supplementary] HMAS radar chart at day~15 (extended-range).}
\label{fig:radar_day15}
\end{figure*}

\begin{figure*}[p]
\centering
\includegraphics[width=\textwidth]{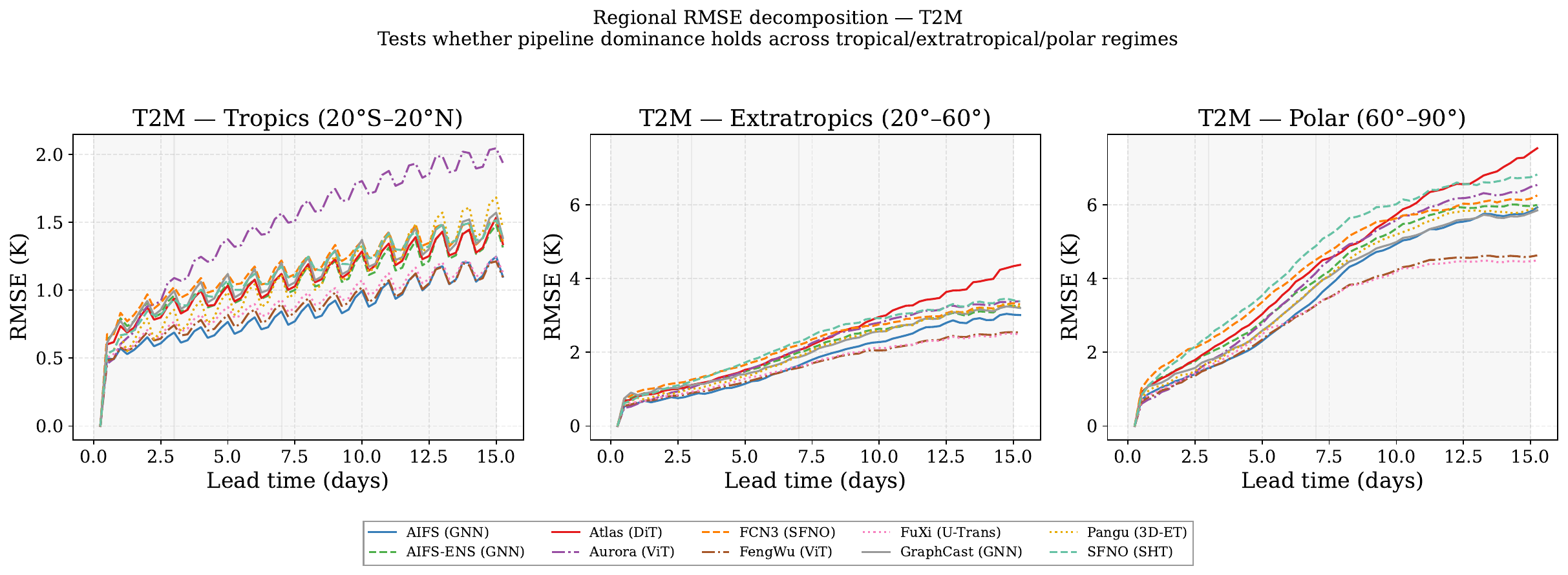}
\caption{[Supplementary] Regional RMSE decomposition for T2M by latitude band.}
\label{fig:regional_t2m}
\end{figure*}

\begin{figure*}[p]
\centering
\includegraphics[width=\textwidth]{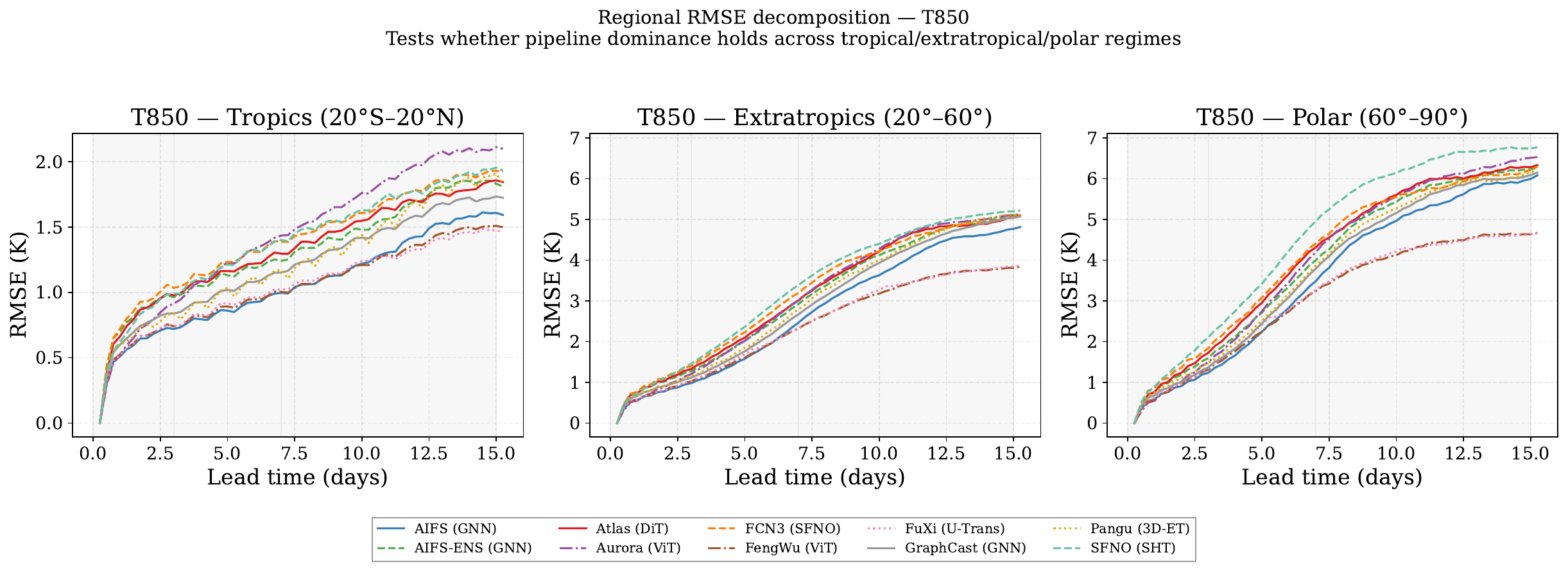}
\caption{[Supplementary] Regional RMSE decomposition for T850 by latitude band.}
\label{fig:regional_t850}
\end{figure*}

\begin{figure*}[p]
\centering
\includegraphics[width=\textwidth]{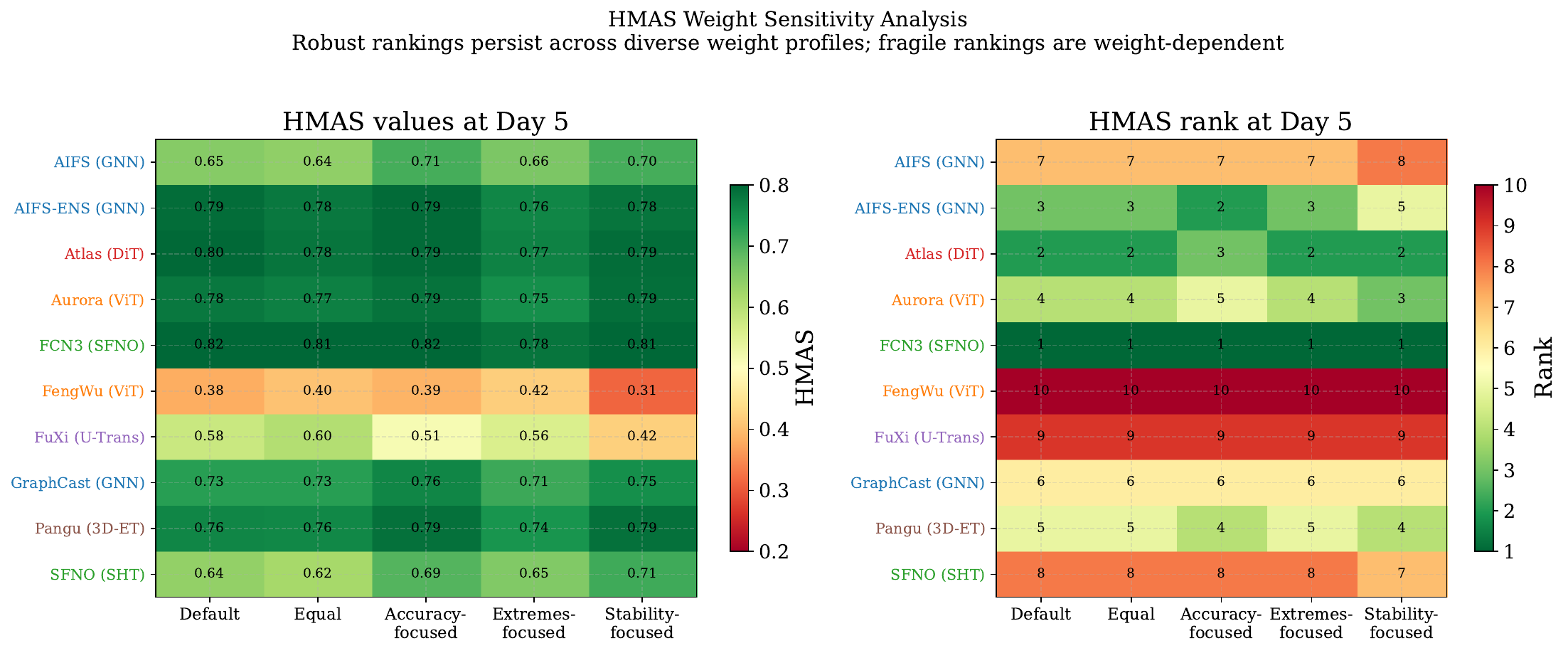}
\caption{[Supplementary] HMAS weight sensitivity analysis.
Grouped bar chart showing HMAS values under five weighting schemes (default, equal, accuracy-focused, extremes-focused, stability-focused) for all ten models.
Kendall's $W = 0.97$ confirms strong rank concordance, with the top and bottom rankings stable across all schemes.}
\label{fig:hmas_sensitivity}
\end{figure*}

\begin{figure*}[p]
\centering
\includegraphics[width=\textwidth]{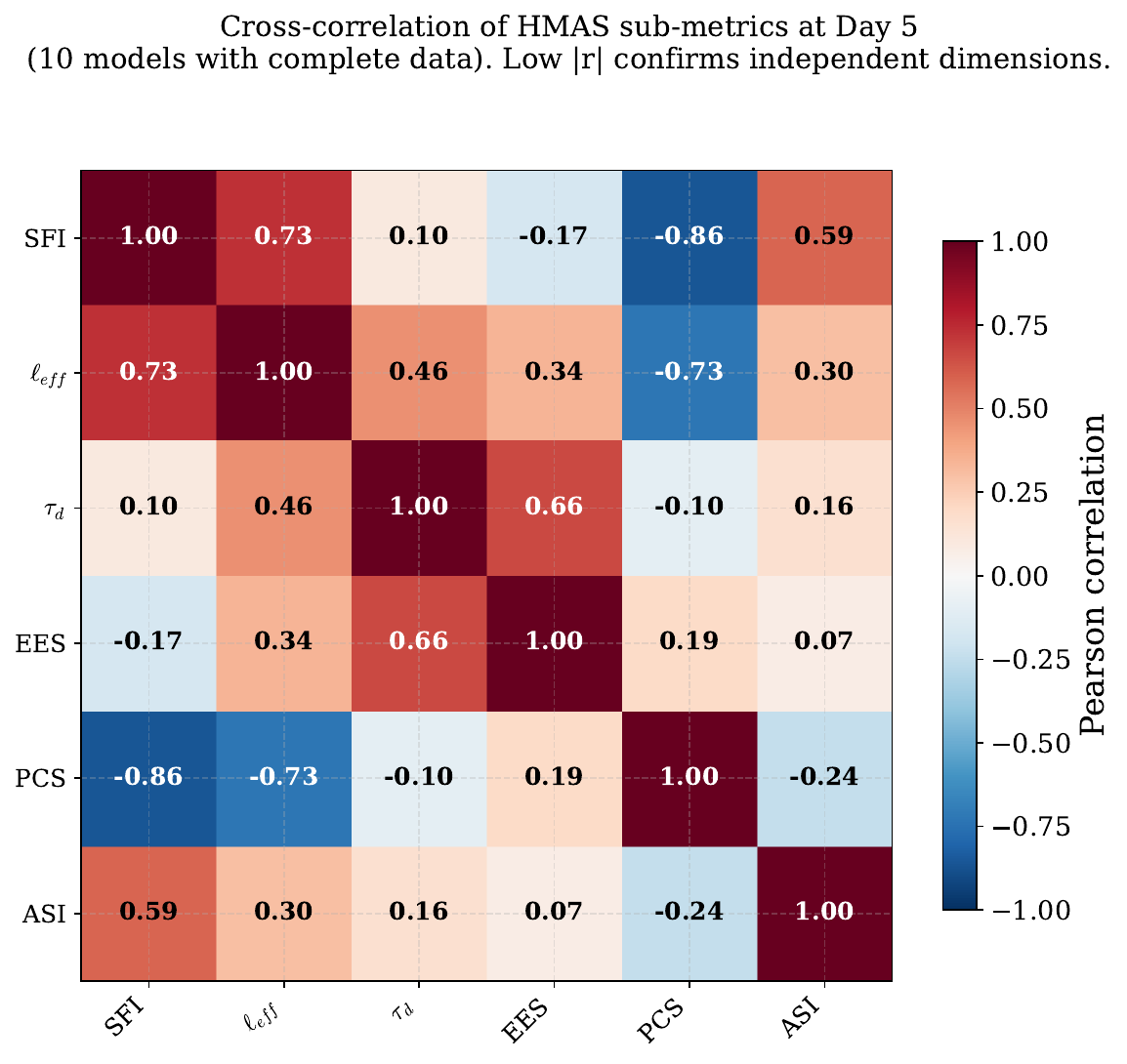}
\caption{[Supplementary] HMAS dimension cross-correlation matrix at day~5.
SFI and PCS are strongly anti-correlated ($\rho = -0.86$), confirming the spectral fidelity--physical consistency trade-off.
ASI is largely independent of accuracy and balance metrics.
Mean absolute off-diagonal correlation $\bar{|\rho|} = 0.38$.}
\label{fig:hmas_correlation}
\end{figure*}

\begin{figure*}[p]
\centering
\includegraphics[width=\textwidth]{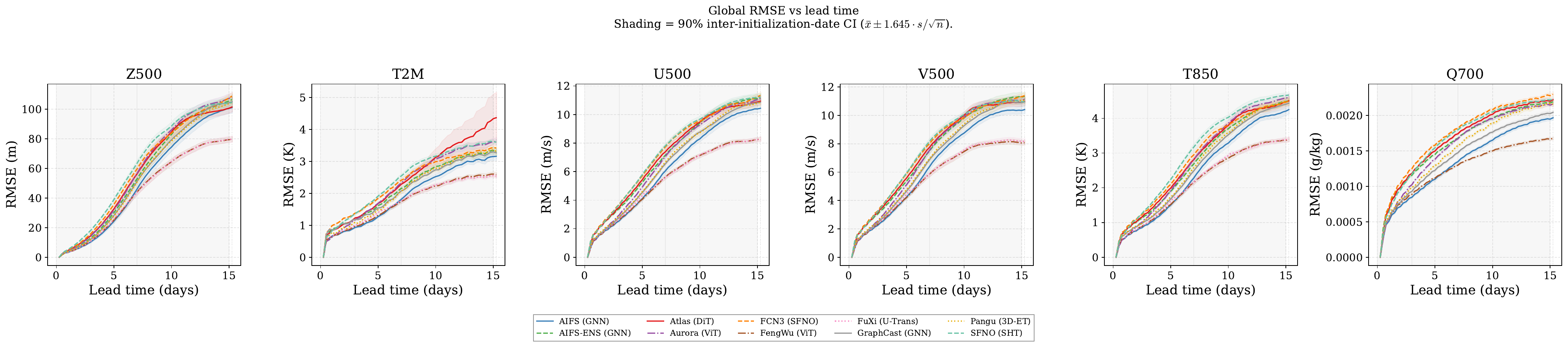}
\caption{[Supplementary] Global area-weighted RMSE vs.\ lead time for all six verification variables (Z500, T2M, U500, V500, T850, Q700), averaged over 30 initialization dates with inter-initialization-date 90\% CIs.}
\label{fig:rmse_all_vars}
\end{figure*}

\begin{figure*}[p]
\centering
\includegraphics[width=\textwidth]{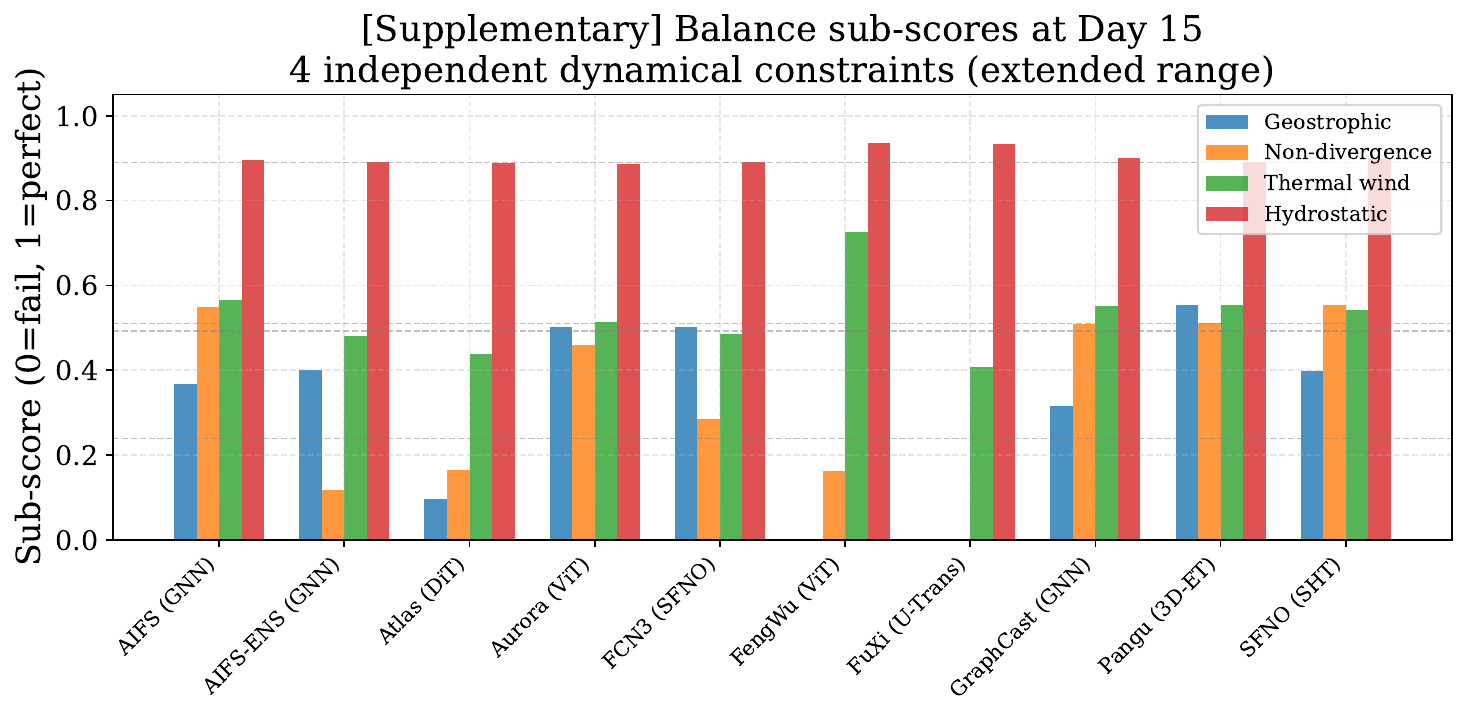}
\caption{[Supplementary] Balance sub-scores at day~15 (extended range) for all ten models.
Compared with day~5 (Fig.~\ref{fig:physical_consistency}, right panel), Atlas shows marked degradation in geostrophic balance and non-divergence, while Pangu maintains high sub-scores across all four balance components.}
\label{fig:pcs_day15_subscore}
\end{figure*}
\end{document}